\documentclass[journal]{IEEEtran}

\usepackage[utf8]{inputenc}
\usepackage{blindtext}
\usepackage{tabularx}

\usepackage{amsthm}

\usepackage{mathtools}
\usepackage{amssymb}
\usepackage{cite}
\usepackage[usenames,dvipsnames]{xcolor}
\usepackage{tikz}
\usepackage{pgfplots}
\usepackage{svg}
\usepackage{adjustbox}
\usepackage{comment}
\usepackage{standalone}
\usepackage{calc}
\usepackage[acronym]{glossaries}
\usepackage{ifthen}
\newboolean{proofs}
\usepackage[hyphens]{url}

\usepackage[
 colorlinks,
    linkcolor={blue!50!black},
    citecolor={blue!50!black},
    urlcolor={blue!80!black}]{hyperref}
\usepackage{bookmark}
\hypersetup{
  urlcolor=RoyalBlue, linkcolor=black, citecolor=black, %
}
\usepackage{multicol}
\usepackage{todonotes}
\presetkeys{todonotes}{inline}{}

\usepackage[capitalize]{cleveref}
\Crefname{equation}{Eq.}{Eqs.}
\Crefname{figure}{Fig.}{Figs.}
\Crefname{tabular}{Tab.}{Tabs.}
\Crefname{table}{Tab.}{Tabs.}
\Crefname{definition}{Def.}{Defs.}
\Crefname{section}{Sec.}{Sects.}
\Crefname{subsection}{Sec.}{Sects.}
\Crefname{theorem}{Thm.}{Thms.}
\Crefname{condition}{Cond.}{Conds.}

\usepackage{graphicx}
\usepackage{booktabs}
\usepackage{threeparttable}
\usepackage{multirow}
\usepackage[ruled,vlined]{algorithm2e}
\usepackage{tikz, pgfplots, pgfplotstable}
\usetikzlibrary{
    positioning,
  	cd,
  	shapes,
  	math,
  	decorations.markings,
  	decorations.pathmorphing,
	decorations.pathreplacing,
  	arrows.meta,
  	calc,
  	fit,
  	quotes,
   intersections,
   hobby,
   patterns,
   decorations.pathmorphing, decorations.markings, shadows,shapes, cd,arrows.meta, fit,quotes
  }
\tikzcdset{arrow style=tikz, diagrams={>=To}}

\tikzset{
   tick/.style={postaction={
      decorate,
      decoration={markings, mark=at position 0.5 with {\draw[-] (0,.4ex) -- (0,-.4ex);}}}
   }
}
\pgfplotsset{compat=1.15}

\tikzstyle{block} = [draw, rectangle, minimum height=2em, minimum width=3em,thick]

\tikzstyle{blockdot} = [block, dotted,rounded corners=4, inner sep=-2pt]

\tikzstyle{blockfill} = [block,rounded corners=4, inner sep=-2pt,fill=blue!5!white]

\tikzstyle{every node}=[font=\footnotesize]

\makeatletter
\let\MYcaption\@makecaption
\makeatother
\usepackage[font=footnotesize]{subcaption}
\makeatletter
\let\@makecaption\MYcaption
\makeatother

\theoremstyle{definition}
\newtheorem{theorem}{Theorem}

\newtheorem{lemma}[theorem]{Lemma}

\newtheorem{remark}[theorem]{Remark}
\newtheorem{definition}[theorem]{Definition}

\makeatletter
\def\@opargbegintheorem#1#2#3{\trivlist
   \item[]{\bfseries #1\ #2\ (#3)} \itshape}
\makeatother

\DeclareFontFamily{U}{mathx}{\hyphenchar\font45}
\DeclareFontShape{U}{mathx}{m}{n}{
      <5> <6> <7> <8> <9> <10>
      <10.95> <12> <14.4> <17.28> <20.74> <24.88>
      mathx10
      }{}
\DeclareSymbolFont{mathx}{U}{mathx}{m}{n}
\DeclareFontSubstitution{U}{mathx}{m}{n}
\DeclareMathAccent{\widecheck}{0}{mathx}{"71}

\newacronym{abk:av}{AV}{autonomous vehicle}
\newcommand{\agent}{\mathcal{A}}

\newcommand{\agentset}{\mathbb{A}}

\newcommand{\appearance}{\mathrm{appear}}
\newcommand{\appearanceset}{\mathrm{AP}}

\newcommand{\body}{\mathcal{B}}

\newcommand{\confspace}{\mathcal{Q}}
\newcommand{\confspacen}[1]{\mathcal{Q}_{#1}}

\newcommand{\config}[1]{q_{#1}}

\newcommand{\confspacenagent}[1]{Q_{0#1}}
\newcommand{\configagent}[1]{q_{0#1}}
\newcommand{\class}[1]{\mathcal{C}_{#1}}
\newcommand{\classinstance}[1]{C_{#1}}
\newcommand{\classinstanceset}[1]{\mathbb{C}_{#1}}
\newcommand{\controlinputs}[1]{\mathcal{U}_{#1}}
\newcommand{\controlinput}[1]{u_{#1}}

\newcommand{\sethree}{\mathrm{SE(3)}}
\newcommand{\setwo}{\mathrm{SE(2)}}

\newcommand{\sensreq}{\mathrm{PR}}

\newacronym{av}{AV}{Autonomous Vehicle}
\newacronym{ad}{AD}{Autonomous Driving}
\newacronym{algname}{CODEI}{Co-design of Embodied Intelligence}
\newcommand{\algname}{CODEI}

\definecolor{darkgreen}{rgb}{0.0, 0.7, 0.0}
\definecolor{darkred}{rgb}{0.7, 0.0, 0.0}

\newcommand{\dynamics}[1]{\mathrm{dyn}_{#1}}

\newcommand{\environment}{\mathrm{env}}
\newcommand{\environmentset}{\mathbb{E}}
\newcommand{\emtpyfunc}{\mathrm{emp}}
\newcommand{\edited}[1]{{\color{black}#1}}

\newacronym{fnr}{FNR}{False Negative Rate}
\newacronym{fpr}{FPR}{False Positive Rate}
\newacronym{fn}{FN}{False Negative}
\newacronym{fp}{FP}{False Positive}
\newacronym{tp}{TP}{True Positive}
\newacronym{fov}{FoV}{Field of View}

\newacronym{ilp}{ILP}{Integer Linear Programming}

\newcommand{\kclass}{\mathrm{K}_{\text{class}}}

\newcommand{\lvsp}{\mathrm{L}_{\text{mpp}}}

\newcommand{\mounts}{\mathrm{mp}}
\newcommand{\mountsset}{\mathrm{MP}}
\newcommand{\mountso}{\mathrm{mo}}
\newcommand{\mountsoset}{\mathrm{MO}}

\newcommand{\multipolygon}{\mu}
\newcommand{\menv}{\mathrm{M}_{\text{env}}}
\newcommand{\mountedpercpset}{\mathrm{MPP}}

\newcommand{\op}{^{\mathrm{op}}}

\newcommand{\powerset}[1]{\mathtt{POW}(#1)}

\def\prov{\mathsf{prov}}
\newcommand{\prior}[1]{\mathcal{P}_{#1}}

\newcommand{\percpset}{\mathrm{PP}}
\newcommand{\pcp}{\mathrm{pcp}}
\newcommand{\priorcheck}{\mathrm{priorcheck}}
\newcommand{\ppp}{\mathrm{ppp}}

\newcommand{\query}{\psi}
\newcommand{\queryspace}{\Psi}

\newcommand{\collision}{\mathrm{collision}}
\newcommand{\compress}{\mathrm{compress}}

\newcommand{\R}[1]{{\color{dpred}#1}}

\def\req{\mathsf{req}}
\newcommand{\reals}{\mathbb{R}}
\newcommand{\naturaln}{\mathbb{N}}
\newcommand{\robot}{\mathcal{R}}

\newcommand{\state}[1]{x_{#1}}
\newcommand{\statespace}{\mathcal{X}}
\newcommand{\threedshape}{\mathrm{SH}}

\newcommand{\shape}[1]{\mathrm{sh}_{#1}}

\newcommand{\scenario}{\mathcal{S}}
\newcommand{\scenarioinstance}{S}
\newcommand{\sensorp}{\mathrm{pp}}
\newcommand{\sensorpset}{\mathrm{PP}}

\ifPDFTeX
\newcommand{\tickar}{\begin{tikzcd}[baseline=-0.5ex,cramped,sep=small,ampersand replacement=\&]{}\ar[r,tick]\&{}\end{tikzcd}}
\else
\newcommand{\tickar}{\nrightarrow}
\fi

\newcommand{\tup}[1]{\langle#1\rangle}
\newcommand{\task}{\mathcal{T}}
\newcommand{\taskspace}{\mathbb{T}}
\newcommand{\taskqueries}{\mathrm{tq}}

\newcommand{\trajectorymap}[1]{\overline{\mathrm{u}}_{#1}}
\newcommand{\trajectorymapspace}[1]{\overline{\mathrm{U}}_{#1}}

\newcommand{\virtsensorp}{\mathrm{mpp}}

\newcommand{\vspccoverage}{\mathrm{mppcc}}
\newcommand{\vspccoveragecat}{\mathrm{MPPC}}

\newcommand{\workspace}{\mathcal{W}}

\newcommand{\xpos}[1]{x_{#1}}

\newcommand{\ypos}[1]{y_{#1}}

\usepackage[acronym]{glossaries}
\newacronym{abk:dp}{DP}{design problem}
\newacronym{abk:dpi}{DPI}{design problem with implementation}
\newacronym{abk:mdpi}{MDPI}{monotone design problem with implementation}
\newacronym{abk:poset}{poset}{partially ordered set}
\definecolor{dpred}{rgb}{0.7, 0.0, 0.0}
\newcommand{\setOfResources}[1]{\R{\mathcal{R}_{#1}}}

\definecolor{dpgreen}{rgb}{0.0, 0.5, 0.0}
\newcommand{\F}[1]{{\color{dpgreen}#1}}
\newcommand{\setOfFunctionalitiesOp}[1]{\F{\mathcal{F}_{#1}}^{\mathrm{op}}}
\newcommand{\setOfFunctionalities}[1]{\F{\mathcal{F}_{#1}}}

\newcommand{\setOfImplementations}[1]{\mathcal{I}_{#1}}
\newacronym{abk:cdp}{CDP}{co-design problem}
\newacronym{abk:cdpi}{CDPI}{co-design problem with implementation}
\newcommand{\power}[1]{\mathcal{P}(#1)}

\usetikzlibrary{positioning,intersections, hobby, patterns, calc, decorations.pathmorphing, decorations.markings, shadows,shapes, cd,arrows.meta, fit,quotes}

\tikzset{
   tick/.style={postaction={
      decorate,
      decoration={markings, mark=at position 0.5 with {\draw[-] (0,.4ex) -- (0,-.4ex);}}}
   }
}
\tikzstyle{block} = [draw, rectangle, minimum height=2em, minimum width=3em,font=\bfseries,rounded corners,thick]
\tikzstyle{block} = [draw, rectangle, minimum height=2em, minimum width=3em]
\tikzstyle{block1} = [draw, rectangle, minimum height=1.5em, minimum width=2.5em]
\tikzstyle{blockDyn} = [draw, rectangle, minimum height=2.5em, minimum width=3.5em, align=center, inner sep=10pt, thick, fill=white, copy shadow={draw=black,fill=black,opacity=1,shadow xshift=0.5ex,shadow yshift=-0.5ex}]
\tikzstyle{blockAlg} = [draw, rectangle, minimum height=1.5em, minimum width=2.5em, align=center, inner sep=10pt, thick]
\tikzstyle{sum} = [draw,circle]

\tikzstyle{nodePre} = [circle, draw,inner sep=1pt,node contents={$\preceq$},thick]
\tikzstyle{nodePreEmpty} = [circle, draw,inner sep=1pt,thick]
\tikzstyle{nodePos} = [circle, draw,inner sep=1pt,node contents={$\posceq$},thick]
\tikzstyle{nodeProd} = [rectangle, draw,inner sep=4pt,node contents={$\times$},rounded corners,thick]
\tikzstyle{nodeSum} = [rectangle, draw,inner sep=4pt,node contents={$\mathbf{+}$},rounded corners,thick]

\definecolor{red}{rgb}{0.75, 0.0, 0.0}

\tikzset{fcname/.store in =\fcname, fcname={}}
\tikzset{funame/.store in =\funame, funame={}}
\tikzset{rcname/.store in =\rcname, rcname={}}
\tikzset{runame/.store in =\runame, runame={}}
\tikzset{whereres/.store in =\whereres, whereres=0.5}
\tikzset{wherefun/.store in =\wherefun, wherefun=0.5}
\tikzset{relres/.store in =\relres, relres={above}}
\tikzset{relfun/.store in =\relfun, relfun={above}}
\tikzset{posres/.store in =\posres, posres=1}
\tikzset{posfun/.store in =\posfun, posfun=1}
\tikzset{loos/.store in =\loos, loos=2}
\tikzset{feedback/.store in =\feedback, feedback=0}
\tikzset{
   DP/.style={%
      label/.style={
         font=\everymath\expandafter{\the\everymath\scriptstyle},
         inner sep=5pt,
         node distance=2pt and -2pt},
      semithick,
      node distance=1 and 1,
      rconn/.style={color=white,opacity=0.0,postaction={decorate}, shorten <=3.2pt, shorten >= 0.8,
      decoration={markings, 
      mark= at position 0 with {
               \coordinate (a);
      },
      mark=at position .5 with
      {
              \ifthenelse{\equal{\feedback}{1}}{\def\angleOut{90}\def\angleIn{90}}{\def\angleOut{0}\def\angleIn{180}}    
              \coordinate (b);
              \draw[dashed,dpred,opacity=1.0] (a) to[out=\angleOut,in=\angleIn,looseness=\loos] 
              node[pos=\posres,\relres=\whereres mm,dpred,opacity=1,fill=white,inner sep=1pt,outer sep=1pt]{\footnotesize{\rcname}} (b);
      },
      mark= at position 1 with 
      {
             \ifthenelse{\equal{\feedback}{1}}{\def\angleOut{0}\def\angleIn{0}}{\def\angleOut{180}\def\angleIn{0}} 
              \ifthenelse{\equal{\feedback}{1}}{\def\symbol{\succeq}}{\def\symbol{\preceq}} 
              \coordinate (c);
              \draw[dpgreen,opacity=1.0] (c) to[out=\angleOut,in=\angleIn,looseness=\loos]
              node[pos=\posfun,\relfun=\wherefun mm,dpgreen,opacity=1,fill=white,inner sep=1pt,outer sep=1pt]{\footnotesize{\fcname}} (b){}; %
              \node[draw,circle,inner sep=0.5pt,color=black,fill=white,opacity=1.0] at (b) (nodepreceq) {$\symbol$}; 
      }
      }},
      runconn/.style={color=dpred,dashed,postaction={decorate},
      decoration={markings,
      mark= at position 1 with {
              \coordinate (a);
              \draw[dpred,opacity=1.0,dashed] ($(a)+(0.05,0)$) --++ (0.5,0) node[\relres,pos=\posres]{\footnotesize{\runame}};}
      }
      },
      funconn/.style={color=white,postaction={decorate},
      decoration={markings,
      mark= at position 0 with {
      \coordinate (a);
      \draw[dpgreen] ($(a)+(-0.05,0)$) -- ($(a)+(-0.5,0)$) node[\relfun, pos=\posfun]{\footnotesize{\funame}};}
      }
      },
      execute at begin picture={\tikzset{
         x=\dpx, y=\dpy,
         every fit/.style={inner xsep=\dpx, inner ysep=\dpy}}}
      },
   dpx/.store in=\dpx,
   dpx = 1.5cm,
   dpy/.store in=\dpy,
   dpy = 1.5ex,
   dp port sep/.store in=\dpportsep,
   dp port sep=2,
   dp port length/.store in=\dpportlen,
   dp port length=4pt,
   dp min width/.store in=\dpminwidth,
   dp min width=0.5cm,
   dp rounded corners/.store in=\dpcorners,
   dp rounded corners=2pt,
   dp small/.style={dp port sep=1, dp port length=2.5pt, dpx=.4cm, dp min width=.4cm, dpy=.7ex},
   dp/.code 2 args={%
      \pgfmathsetlengthmacro{\dpheight}{\dpportsep * (max(#1,#2)) * \dpy}
      \pgfkeysalso{draw,%
        minimum width=\dpminwidth,%
        minimum height=\dpheight,%
        font=\bfseries,
        outer sep=0pt,%
        inner sep=5pt,%
        rounded corners=\dpcorners,
        thick,
        prefix after command={\pgfextra{\let\fixname\tikzlastnode}},
        append after command={\pgfextra{\draw
            \ifnum #1=0{} \else foreach \i in {1,...,#1} { 
            ($(\fixname.north west)!{\i/(#1+1)}!(\fixname.south west)$) +(0,0) node[solid,left,circle,color=dpgreen,draw,fill=dpgreen,scale=0.3] {} coordinate (\fixname_fun\i) -- +(0,0) coordinate (\fixname_fun\i')}\fi %
            \ifnum #2=0{} \else foreach \i in {1,...,#2} {
            ($(\fixname.north east)!{\i/(#2+1)}!(\fixname.south east)$) +(0,0) coordinate (\fixname_res\i') -- +(0,0) node[solid,right,circle,color=dpred,draw,fill=dpred,scale=0.3] {} coordinate (\fixname_res\i)}\fi;
         }}}
         },
      dp name/.style={append after command={\pgfextra{\node[label=center,inner sep=2pt,fill=white] at (\fixname) {\textbf{#1}};}}}
   }
\usepackage{ifxetex}
\ifxetex
\usepackage[tracking=false,kerning=false,spacing=false]{microtype}
\usepackage[protrusion=true,tracking=false,kerning=false,spacing=false]{microtype}
\else
\usepackage[tracking=false,kerning=true,spacing=true]{microtype}

\fi
\usepackage{enumitem}

\usepackage{units}
\usepackage{hyperref}
\usepackage{url}
\usepackage{array} %
\newcolumntype{P}[1]{>{\raggedright\arraybackslash}p{#1}}

\begin{document}
\setboolean{proofs}{true}
\bstctlcite{IEEEexample:BSTcontrol}
\title{\LARGE \bf
	\edited{CODEI: Resource-Efficient Task-Driven Co-Design of Perception and Decision Making for Mobile Robots Applied to Autonomous Vehicles}
}
\author{Dejan Milojevic$^{1,2}$, Gioele Zardini$^{3}$, Miriam Elser$^{2}$, Andrea Censi$^{1}$, Emilio Frazzoli$^{1}$
\thanks{
$^{1}$Institute for Dynamic Systems and Control, ETH Z\"urich, Zurich, Switzerland {\tt\small \{dejanmi, acensi, efrazzoli\}@ethz.ch}}
\thanks{
$^{2}$Chemical Energy Carriers and Vehicle Systems Laboratory, Empa - Swiss Federal Laboratories for Materials Science and Technology, D\"ubendorf, Switzerland {\tt\small miriam.elser@empa.ch}}
\thanks{
$^{3}$Laboratory for Information and Decision Systems, Massachusetts Institute of Technology, Cambridge, MA, USA {\tt\small gzardini@mit.edu}}
}

\maketitle
\begin{abstract}
This paper discusses the integration challenges and strategies for designing mobile robots, by focusing on the task-driven, optimal selection of hardware and software to balance safety, efficiency, and minimal usage of resources such as costs, energy, computational requirements, and weight.
\edited{We emphasize the interplay between perception and motion planning in decision-making by introducing the concept of occupancy queries to quantify the perception requirements for sampling-based motion planners. Sensor and algorithm performance are evaluated using~\gls{fnr} and~\gls{fpr} across various factors such as geometric relationships, object properties, sensor resolution, and environmental conditions. By integrating perception requirements with perception performance, an~\gls{ilp} approach is proposed for efficient sensor and algorithm selection and placement.}
This forms the basis for a co-design optimization that includes the robot body, motion planner, perception pipeline, and computing unit. \edited{We refer to this framework for solving the co-design problem of mobile robots as~\algname, short for Co-design of Embodied Intelligence.}
A case study on developing an~\gls{av} for urban scenarios provides actionable information for designers, and shows that complex tasks escalate resource demands, with task performance affecting choices of the autonomy stack.
The study demonstrates that resource prioritization influences sensor choice: cameras are preferred for cost-effective and lightweight designs, while lidar sensors are chosen for better energy and computational efficiency.

\end{abstract}
\IEEEpeerreviewmaketitle
\section{Introduction}
\IEEEPARstart{E}~mbodied intelligent systems hold great promise for addressing critical societal challenges and enhancing our daily lives.
Whether revolutionizing mobility via autonomous driving, or supply-chain via automated logistics, this technology will impact the world we live in.
However, realizing the full potential of these advances depends on the efficient design and safe operation of such systems.
The complexity of developing embodied intelligence lies in selecting the optimal mix of interdependent hardware and software components.
The final design must ensure safety and efficient task performance while minimizing the resources required for design and operation, such as cost, power consumption, computation, and weight.

In the context of robot perception this involves the choice and placement of sensors and the selection of algorithms which process the sensor measurements. 
Clearly, hardware and software choices are interdependent and influence each other. 
Moreover, they are interconnected with other systems such as the computing units, actuators, or decision making. 
Indeed, a controller relies on the reference created by a motion planner, which is based on state estimates from an estimator, which in turn depends on sensor data and power supply. 
In addition, the integration of perception software such as object detection algorithms introduces uncertainties in the algorithm output that must be carefully considered in the design. 

To tackle these intricate issues, a comprehensive framework which applies abstract reasoning across different areas, and balances functional requirements with resource constraints and trade-offs is needed.
This work outlines our method for addressing the complex task of robot co-design by tackling such challenges.

\begin{figure}[t]
    \centering
    \includegraphics[width=0.49\textwidth]{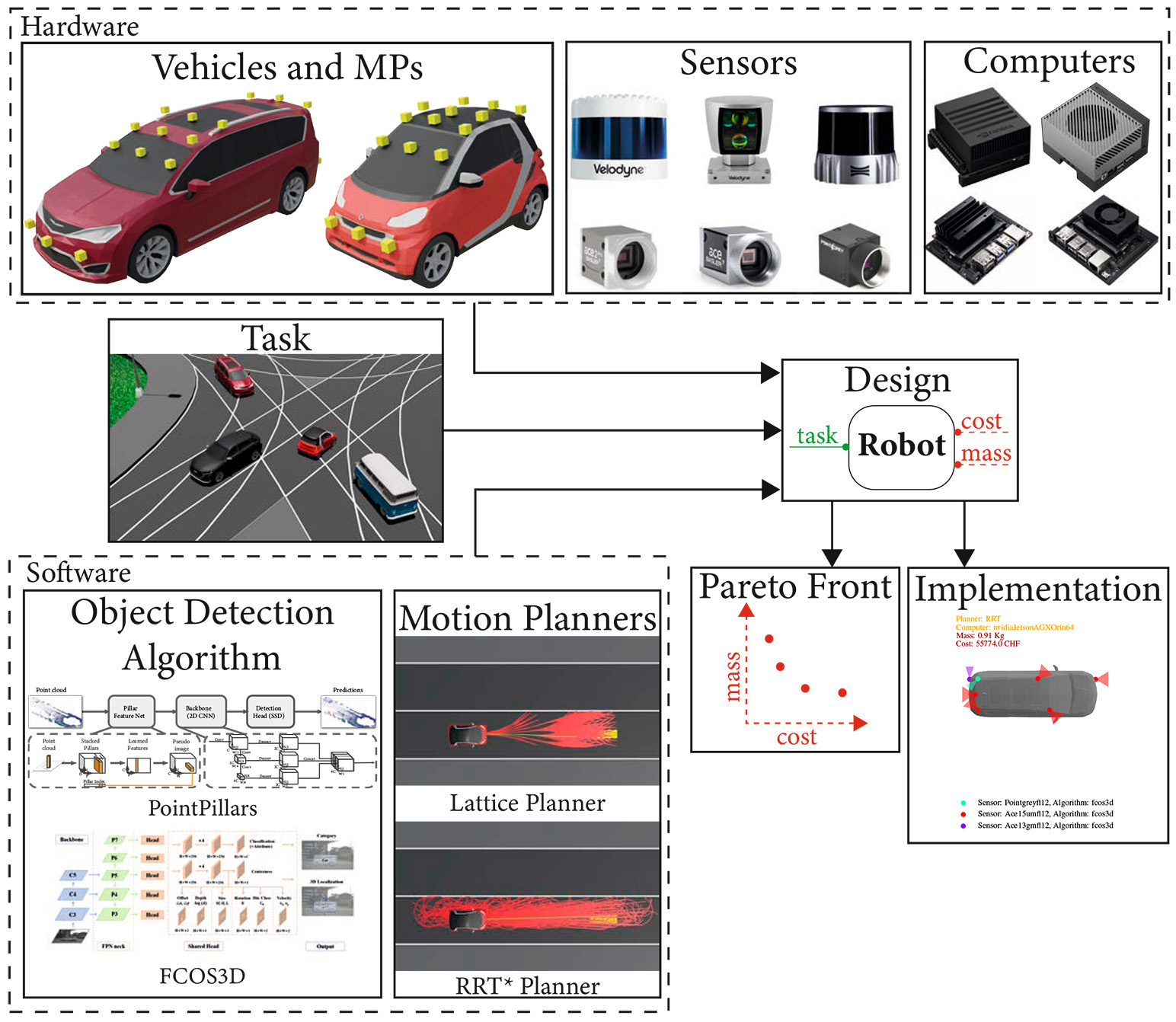}
    \caption{Graphical illustration of the informal problem definition for designing an \gls{av} for urban driving tasks, based on a catalog of hardware and software components with an emphasis on minimizing resources.}
    \label{fig:introproblem}
\end{figure}

\noindent\paragraph*{Informal Problem Definition}
As illustrated in~\cref{fig:introproblem}, the problem features the definition of catalogs with both hardware and software components necessary for the robot design. These include:
\begin{itemize}
    \item \emph{Robot Bodies}: a selection of mobile robot chassis, each with its shape and actuators.
    \item \emph{Sensor Mounting Configurations}: options for mounting sensors on each robot body.
    \item \emph{Perception Pipeline}: combinations of sensors and their corresponding perception algorithms for processing data.
    \item \emph{Decision-Making Algorithms}:  software for determining the robot's motion and actions to complete the task.
    \item \emph{Computing Units}: catalog of computing resources to support the software's operational needs.
\end{itemize}

\noindent \paragraph*{Contribution} The contributions can be summarized as follows.
First, we explore the interconnections between perception pipelines and sampling-based motion planners via the concept of occupancy queries.
Second, we show how to formulate and solve the sensor selection and placement problems for a robot, via set cover problems.
Third, we develop a robot co-design framework leveraging a monotone theory of co-design optimization, promoting the robot task as a functionality, and minimizing resource consumption in terms of monetary costs, power and computational needs, and mass.
Finally, we illustrate the above contributions through a suite of case studies on \glspl{av} design.

\paragraph*{Organization of the paper}
\cref{sec:litrev} reviews the related work, and contextualizes the efforts proposed in this paper. \edited{\cref{sec:approach} outlines the employed methodology and underlying assumptions.} \cref{sec:model} presents in depth our system models, including the robotic platform, tasks, decision-making, perception performance, and requirements. We then present the system co-design optimization problem, and its solution in \cref{sec:codesign}, and showcase various case studies in \cref{sec:results-gen}. Finally, we conclude and provide an outlook for future research in \cref{sec:conclusion}.

\section{Related Work}
\label{sec:litrev}
The challenges of design automation for embodied intelligence are highlighted in several studies~\cite{seshia2016design, zhu2018codesign, lee2008cyber, DBLP:journals/corr/abs-1806-05157}.
Such challenges primarily revolve around developing a framework which can accommodate the complex nature of cyber-physical systems, encompassing both software and hardware within dynamic environments~\cite{seshia2016design,zhu2018codesign}.
Additionally, there is a significant need for algorithms that can efficiently navigate the heterogeneous design landscapes available~\cite{zhu2018codesign}.
The work presented in~\cite{DBLP:journals/corr/abs-1806-05157} highlights the difficulty of integrating diverse components into robotic systems, and determining the specific information needs for a robot to fulfill a given task.
In the following, we review the literature in this field, mainly focusing on sensor selection and its relevance to robotics, design space exploration, comparative analysis of methods and trade-offs, benchmarks, and co-design frameworks.

The challenge of selecting and positioning sensors within a system is complex, and often lacks a closed-form solution.
For instance,~\cite{Joshi2009b} leveraged convex optimization techniques with the objective of reducing the estimation error of certain parameters.
In~\cite{gupta2006stochastic}, the authors introduce a stochastic algorithm designed to optimize sensor scheduling and improve coverage.
A greedy method for sensor selection aimed at state estimation in linear dynamical systems, utilizing Kalman filtering, is described in~\cite{shamaiah2010greedy}.
Furthermore,~\cite{golovin2010online} proposed a novel distributed online greedy algorithm selecting sensors based on the real-time feedback of their utility, targeting the maximization of information richness and energy efficiency.
One of the earliest approaches to sensor selection in robotics was presented in~\cite{hovland1997dynamic}, which introduced a real-time method using stochastic dynamic programming tailored to robotic systems.
Erdmann proposed a sensor selection strategy deeply integrated with a robot's task and planning requirements, based on the hypothetical premise of an ideal sensor fulfilling all informational needs for plan formulation~\cite{erdmann1995understanding}.
Geometric considerations in sensor selection are explored in~\cite{giraud1995sensor}, which employs Gaussian models to approximate uncertainties in the sensor-environment interaction.
Work proposed in~\cite{tzoumas2020lqg} addresses an LQG control co-design problem, simultaneously developing control and sensing strategies with resource limitations.
Furthermore,~\cite{collin2020multiobjective, collin2020autonomous} present a sensor selection framework specifically designed for localization and mapping, and~\cite{dey2020vespa, dey2023machine} focus on placement, orientation, and architecture designs in the context of \glspl{av}.
Additionally,~\cite{dey2023machine} presents a machine learning based framework for generating perception architecture designs for \glspl{av}, simultaneously optimizing sensor positions and orientations, detection algorithms, and fusion algorithms for a given target vehicle.
Finally, a learning algorithm for sensor placement in the context of soft robotics is presented in~\cite{spielberg2021co}.
Despite the presented advancements, current literature does not fully address the integrated selection and placement of sensor hardware in conjunction with the choice of perception algorithms.
There is a notable gap in discussions on optimizing sensor selection for object detection tasks, with a particular focus on contemporary deep learning techniques.
Furthermore, the critical exploration of sensing requirements for bridging decision-making and perception is underrepresented.
The analyzed studies also overlook the necessity for methodologies that unlock seamless integration with other design considerations, such as computer and actuator selection, and tend to neglect the impact of sensor uncertainties.

Design space exploration has gathered substantial interest in robotics research, with significant contributions aiming to delineate the boundaries of sensor and actuator requirements for effective robotic planning.
In this context,~\cite{ghasemlou2018delineating, saberifar2019toward} examined the minimal necessary sensors or actuators by assessing the consequences of their degradation on robotic planning capabilities, and~\cite{zhang2020abstractions} proposed an innovative method to identify sensors sufficient for resolving planning problems, employing an upper cover concept to condense sensor data and expedite the exploration process.
Nardi introduced a practical approach for navigating design spaces within multi-objective optimization frameworks, specifically applied to hardware design challenges~\cite{nardi2019practical}.
Furthermore, comparative analyses of robotic components have been advanced through the works of O'Kane, Lavalle, and Censi~\cite{o2008comparing, lavalle2012sensing, censi2015power}, which explore methodologies for assessing sensor performance and establishing criteria for comparisons.
The notion of sensor dominance and the subsequent development of a sensor lattice~\cite{o2008comparing, lavalle2012sensing} provide a structure means to rank sensors according to their task efficacy.
Additionally,~\cite{censi2015power} conducts a power-performance analysis, comparing different sensor families for specific tasks. 
In a similar context,~\cite{nardi2015introducing} introduces a benchmark for evaluating SLAM algorithms in robotics, utilizing metrics such as execution time and energy consumption.
Trade-off analysis in design choices is examined in contributions such as~\cite{lahijanian2018resource, seok2014design, saberifar2022charting}.
In \cite{lahijanian2018resource}, a methodology is introduced for exploring trade-offs between performance and resource utilization in the design of mobile robots. 
On the other hand,~\cite{seok2014design} outlines design principles aimed at enhancing energy efficiency in legged robots.
The trade-off between design complexity of a robot and plan execution are explored in~\cite{saberifar2022charting}.

From the point of view of holistic co-design frameworks, significant advancements have been made in robot design methodologies encompassing both software and hardware elements, facilitated by high-level behavioral specifications.
Mehta introduced a novel approach utilizing linear temporal logic to transform high-level design specifications into tangible selections of robot components from an extensive library, bridging the gap between abstract design requirements and practical component choices, streamlining the design process~\cite{mehta2018robot}.
Furthermore,~\cite{ha2018computational} develops a heuristic algorithm specifically targeted at the creation of robotic devices tailored to follow predefined motion trajectories accurately.
The algorithm navigates through the vast array of possible configurations of modular components to pinpoint the ones which best match the desired trajectories. 
In a similar vein,~\cite{shell2021design} explores the optimization of robotic design by carefully selecting actuation and sensing hardware to minimize design costs while ensuring the robot's ability to execute plans and accomplishing tasks.

The methods previously discussed do not focus on fully automating the design process for an entire robotic system. 
They overlook several critical co-design challenges that must be addressed to achieve a comprehensive and automated design process, as identified in~\cite{zardini2023co, seshia2016design, zhu2018codesign, lee2008cyber, DBLP:journals/corr/abs-1806-05157}, such as a) formalizing heterogeneous components across varying levels of abstraction, b) composition heterogeneous components to allow co-design across the entire system, c) facilitating collaboration among different systems as well as their domain experts, d) ensuring computational tractability, which allows quantitative design solutions, e) accommodating continuous systems that evolve over time, and f) maintaining intellectual tractability for simple usage and understanding.

Our research is based on the monotone theory of co-design~\cite{Censi2015a,censi2024} and builds on our series of previous works~\cite{Zardini2021d, zardini2021co, zardini2022task, zardini2022co}, where we studied the co-design of autonomy in the context of \glspl{av} and mobility.
In the current work, we advance our methodology by modeling each component separately and fostering compositional interconnections, particularly between the perception and the decision-making processes of a robot.
\edited{
\section{Approach}
\label{sec:approach}
To address the robot co-design problem, we must understand the information required from the environment to fulfill the robot's task. For instance, when designing an \gls{av} in an urban setting with a maximum speed of \unitfrac[30]{km}{h}, is it necessary to detect a pedestrian at distances of \unit[1]{km}, \unit[50]{m}, or \unit[10]{m} from the ego vehicle? Furthermore, we need to understand which information can be provided by the perception pipelines by understanding their perception performance. Once we know the \emph{perception requirements} and \emph{perception performance}, we present a method for selecting sensors, perception algorithms, mounting positions, and orientations, all from a predefined catalog, which cover the perception requirements of an agent, while minimizing certain resources such as sensor weight, power consumption or price. This forms the inner optimization of our~\algname~algorithm, demonstrated in \cref{sec:sspp,,sec:solve:ssp}. The outer optimization of \algname, explained in \cref{sec:codesigntheory,,sec:codesign:task:pr,,sec:sspp}, performs a holistic monotone co-design optimization over the entire robot's hardware and software components. Additionally, our framework allows the consideration of prior knowledge about probable object configurations in the environment and object dynamics to refine the robot design. To solve the inner and outer optimization, we leverage the following fundamental and computational assumptions.
}

\edited{
\noindent\paragraph*{Fundamental Assumptions} To determine information requirements and address the robot co-design problem, we assume that the robot's software architecture is factorized into perception, state estimation, planning, and control modules. Specifically, the perception module provides necessary information to the motion planner, which computes trajectories that guide the robot toward its goal. Note that our methodology is adaptable to different software architectures and motion planners, as long as one can acquire the information needed by the robot to complete the task, which must be provided by the perception pipeline. We illustrate how such information can be obtained using sampling-based motion planners that generate \emph{occupancy queries} to infer the agent's state, enabling us to define perception requirements. Finally, we assume that object detections from the perception layer are binary: objects are either detected or not, based on outputs from the perception pipeline. To model this, we use a probabilistic representation of the perception pipeline’s performance, defined by \gls{fpr} and \gls{fnr}, which depend on various factors. This binary relation assumes an object is detectable under certain configurations and environmental conditions if \gls{fpr} and \gls{fnr} are below a predefined threshold. We focus on object detection as the perception task. The approach could be extended to additional perception tasks such as localization.
}

\edited{
\noindent\paragraph*{Computation Assumptions} The configuration space is assumed to be planar, i.e., in $\setwo$. The search space for the co-design problem is huge, which restricts certain components, such as mounting position and orientation, to finite options, while others, such as perception requirements, remain continuous. Perception requirements for a specific agent and task are determined through simulation. The \gls{fpr} and \gls{fnr} values are estimated by benchmarking real sensor data and existing perception algorithms. We determine the mounted perception pipeline coverage on the vehicle body via ray casting in a 3D simulation, accounting for potential self-occlusion by the vehicle body.
}

\section{System Modeling}
\label{sec:model}
\edited{In this section, we define all necessary components for solving the co-design problem with~\algname. We begin in~\cref{subsec:robot}, which outlines the robot, including the hardware and software components available for design. Next, in~\cref{sec:task}, we define the robot's task to optimize its design accordingly. In~\cref{subsec:agent}, we model the agent as a sampling-based motion planner that generates occupancy queries to infer its state. This modeling forms the basis for defining perception requirements in~\cref{sec:requirements}, which must be covered by the selected perception pipelines. Finally, in~\cref{sec:sensing}, we describe how we model perception pipeline performance to understand how it meets the defined perception requirements.}
\subsection{Modeling the robotic platform}\label{subsec:robot}
We consider a mobile robot~$\robot$,  defined by its physical body~$\body$ (which includes considerations of shape, actuators, and hardware configurations) with configuration space~\edited{$\confspacenagent{}^{\workspace}$, where the superscript~$\workspace$ indicates the global coordinate frame. The robot’s software, responsible for decision-making and control, is referred to as the agent~$\agent$.}
\noindent \paragraph*{Agent}

We assume that the agent~$\agent$ consists of a modular software architecture, comprising perception, state estimation, motion planning, and control~\cite{paden2016motionplanningsurvey}. In particular, we want to choose the planner and the perception system for the agent.

\noindent \paragraph*{Body}
The robot body~$\body$ encompasses hardware components, including its 3D shape and actuators. We define the robot's body as follows.

\edited{\begin{definition}[Body]
    A robot body~$\body$ is defined by a tuple including the physical 3D shape of the robot~$\threedshape \subset \reals^3$, the configuration space~$\confspacenagent{}^{\workspace}$, the control space, the dynamics~$\dynamics{}$, the state space, and all additional hardware components and robot's body appearance, such as actuators, batteries, color, material, etc..
    \end{definition}}
\edited{\begin{remark} The dynamics function is expressed as~$\dot{\state{}}_t := \dynamics{}(\state{t}, \controlinput{t})$, where~$u_t$ denotes the control input and~$x_t$ the state at time~$t \in \reals_{\geq 0}$. The state~$\state{} \in \statespace$, where the state space is defined as~$\statespace := \confspacen{}^{\workspace} \times \mathcal{H}$, with~$\mathcal{H}$ representing a hidden space. It is important to note that, without loss of generality, the dynamics may be stochastic.\end{remark}}

The examination of the robot's structural framework~$\body$ involves assessing its mounting positions~$\mounts$ (with~$\mounts\in \mountsset$ and~$\mountsset \subset \threedshape$), as well as the selection of sensors. The sensor hardware with the related perception algorithm is referred to as a ``perception pipeline''~$\sensorp$.
In particular, our analysis focuses on 3D object detection to demonstrate the perception pipeline's ability to detect objects in the environment.
The collection of all perception pipelines is denoted by~$\sensorpset$.
Furthermore, we evaluate sensor mounting orientations~$\mountso\in \mountsoset$, characterized by sensor yaw and pitch angles, such that~$\mountsoset\subseteq \reals^{2}$.
These aspects together form the specification of the robot's body.

\begin{definition}[Robot]
A robot~$\robot$ is a tuple consisting of an agent~$\agent$ and body~$\body$:~$\robot \coloneqq \tup{ \agent, \body}$.
\end{definition}

\subsection{Modeling a task}\label{sec:task}
Consider a robot~$\robot$, operating within the workspace~$\workspace \subset \reals^3$. The robot starts its mission from an initial configuration denoted by~\edited{$\configagent{, \mathrm{start}}^{\workspace} \in \confspacenagent{}^{\workspace}$} and seeks to reach a goal area.\footnote{We consider the goal in~$\reals^2$, but in general the goal can manifest in various forms, including a terminal configuration~$\configagent{, \mathrm{end}}^{\workspace}$, a volume in~$\reals^3$ to be reached, following another object, or the ability to move for a specified duration.}{} The environment may include both dynamic and static objects. Dynamic objects encompass moving entities such as robots, vehicles, and humans. On the other hand, static objects consist of stationary elements such as trees or buildings.

\edited{
\begin{definition}[Object class]
An \emph{object class}~$\class{}$ is a tuple which contains the configuration space~$\confspace^{\workspace}$, the control space, and the dynamics~$\dynamics{}$ of the class. The final element in the tuple is the appearance distribution of a class, where the appearance of a class is represented by a tuple comprising elements such as shape, color, material, etc., denoted as~$\appearance$. 
The set of all possible appearances is represented by~$\appearanceset$.
\end{definition}
}
\edited{\begin{remark}
    As aforementioned, the dynamics~$\dynamics{}$ may be stochastic. Consequently, with limited prior knowledge of how objects can move in the environment, the robot design becomes more conservative, resulting in higher resource costs, since it must account for objects potentially moving in any direction and at any velocity.
\end{remark}}
An instance of a class~$\class{i}$ is defined as a tuple~$\classinstance{i}=\tup{\confspace_{i}^{\workspace}, \controlinputs{i}, \dynamics{i}, \appearance_{i}}$, 
where a particular appearance~$\appearance_{i}$ is drawn from the appearance distribution.

\edited{The function~$\shape{i} \colon \powerset{\confspace_{i}^{\workspace}} \to \powerset{\reals^2}$ maps a class or robot configuration into the footprint 
projecting the 3D shape onto the ground plane, where $\powerset{}$ indicates the power set}.

\begin{remark}
It is crucial to differentiate between the robot's 3D shape,~$\threedshape \in \reals^3$, which includes its elevation, 
and the robot's footprint,~\edited{$\shape{0}(\configagent{}^{\workspace}) \in \reals^2$} for a given configuration~\edited{$\configagent{}^{\workspace} \in \confspacenagent{}^{\workspace}$}. The footprint is essentially a projection of the robot's shape onto the ground plane. 
This distinction becomes particularly relevant in later discussions, as outlined in~\cref{sec:codesign:task:pr}.
\end{remark}

In addition, the operational environment encompasses various weather and light conditions. 
Such conditions are collectively referred to as \emph{environmental conditions}, denoted as~$\environment$. 
For simplicity, we use the term environmental conditions to encapsulate a range of possibilities, which include discrete values such as day and night time or rain and sunny conditions. 
Without loss of generality, this can also refer to continuous values such as rain density or time of day. 
The entire set of possible environmental conditions is denoted as~$\environmentset$.

\edited{
\begin{definition}[Scenario]
A scenario~$\scenario$ is defined by the workspace and the distributions governing the robot’s initial configuration, goal area, and environmental conditions. The scenario includes~$N$ object classes, each with an associated object class distribution following a Poisson distribution, specifying the expected number of objects per class. Additionally, the prior configurations~$\prior{}$ of the classes are defined such that~$\prior{}\subseteq \pi_3(\class{})=\confspacen{}^{\workspace}$ for a given class. This prior outlines the allowed configurations for objects of that specific class. 
\end{definition}
}
\edited{
\begin{remark}
    The prior~$\prior{}$ can be used by the agent during online planning, though agents that do not rely on it are also feasible, as our focus is on the overall design rather than specific agent implementation. In our~\algname~framework, the prior constrains perception requirements, similar as class dynamics, to support a resource-efficient robot design. Without prior knowledge,~$\prior{}=\confspacen{}^{\workspace}$, meaning object classes could appear anywhere in the environment, requiring a comprehensive perception system capable of detecting objects from any direction.
\end{remark}
}

\edited{
A scenario instance~$\scenarioinstance$ represents a concrete realization of a scenario 
$\scenario$,  where the initial configuration, goal, and environment are drawn from the respective distributions. Moreover,
$M$ number of object class instances are drawn from their corresponding Poisson distributions.
In this work, we define the task as a set of scenario instances. 
In principle, however, a task could also be defined as a distribution of scenarios, where a set of scenarios can be sampled.
}
\begin{definition}[Task]
A \emph{task}~$\task$ is a set of scenario instances.
\end{definition}

\subsection{Modeling an agent}\label{subsec:agent}
In a common agent's architecture, including perception, state estimation, motion planning, and control, the dependency of motion planning on perception data underscores the importance of defining the precise ``information'' necessary for trajectory planning. Identifying the ``minimum'' required sensors and perception algorithms for a robot, given a particular motion planner, necessitates this specificity. Motion planning algorithms typically need a notion of the obstacle free configuration space to compute a reference trajectory. Combinatorial motion planning~\cite{lavalle2006, paden2016motionplanningsurvey, chazelle1985approximation, takahashi1989motion, backer2007finding} and  optimization-based motion planning~\cite{paden2016motionplanningsurvey, claussmann2019review, falcone2007predictive, falcone2007linear, kim2014model, raffo2009predictive, yoon2009model, liniger2015optimization} depend on mathematical models for the free configuration space, represented through geometric shapes or optimization constraints. The task of pinpointing the critical information necessary for calculating a reference trajectory is notably challenging in these frameworks, mainly because they require knowledge of the entire state space including all obstacles. In contrast, sampling-based planners~\cite{lavalle2006,paden2016motionplanningsurvey,claussmann2019review} offer a different strategy, sidestepping the need for precise internal representations of obstacles. Such planners generate a state hypothesis by posing a series of questions, such as ``\textit{Will there be a collision if I occupy a certain configuration at a certain time?}''. 
These questions are referred to as \emph{occupancy queries} or just \emph{queries} and are represented as elements of the configuration space~\edited{$\confspacenagent{}^{\robot}$} at a certain time~$t$ with a certain environment~$\environment$. \edited{With the superscript in~\edited{$\confspacenagent{}^{\robot}$} we indicate the ego coordinate frame}. Sampling-based planners thus enable a reverse flow of information within the outlined agent architecture, indicating a progression of data from the motion planning phase back to the perception system.
For the sake of simplicity, the term agent throughout the remainder of this paper denotes a sampling-based motion planner.

\begin{definition}[Query]
A \emph{query} is defined as~$\query \in \queryspace$, where~$\queryspace$ is the product space of the configuration space~\edited{$\confspacenagent{}^{\robot}$}, the time in~$\reals^{+}$ and the environment in~$\environmentset$:~$\queryspace \coloneqq \edited{\confspacenagent{}^{\robot}} \times \reals^{+} \times \environmentset.$
\end{definition}

\begin{remark}
\edited{Different motion planners produce different distributions of queries. Planners such as RRT* converge to an optimal solution. However, during the search for the optimal solution, a large number of random configurations are sampled, which can potentially be unbounded. Lattice planners, on the other hand, are not optimal, but by simply relying on a fixed discretization of the search space with particular motion primitives, less information is required from the sensors compared to RRT*. This trade-off between planner optimality and information requirements is illustrated in \cref{fig:planners}. In \cref{fig:rrtstar}, we show an example of an \gls{av} using an RRT* planner, while in \cref{fig:astar}, the \gls{av} is paired with a lattice planner, employing motion primitives and A* search.}
\end{remark}

\begin{figure}[tb]
    \centering
    \begin{subfigure}[b]{0.5\textwidth}
        \centering
        \includegraphics[width=\textwidth, trim=0cm 6cm 0cm 18.5cm, clip]{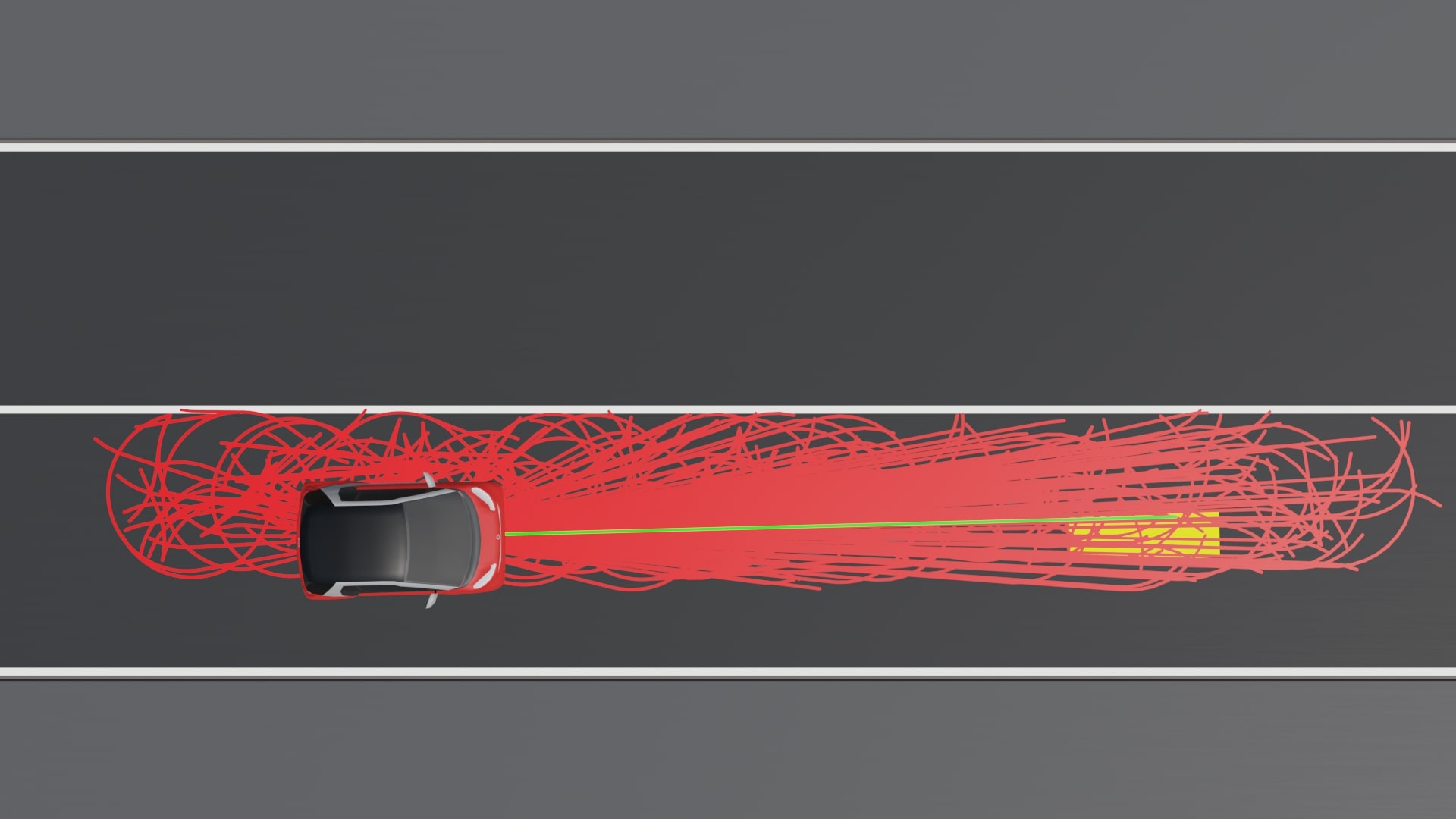}
        \caption{Example of an RRT* planner.}
        \label{fig:rrtstar}
    \end{subfigure}
    \begin{subfigure}[b]{0.5\textwidth}
        \centering
        \includegraphics[width=\textwidth, trim=0cm 6cm 0cm 18.5cm, clip]{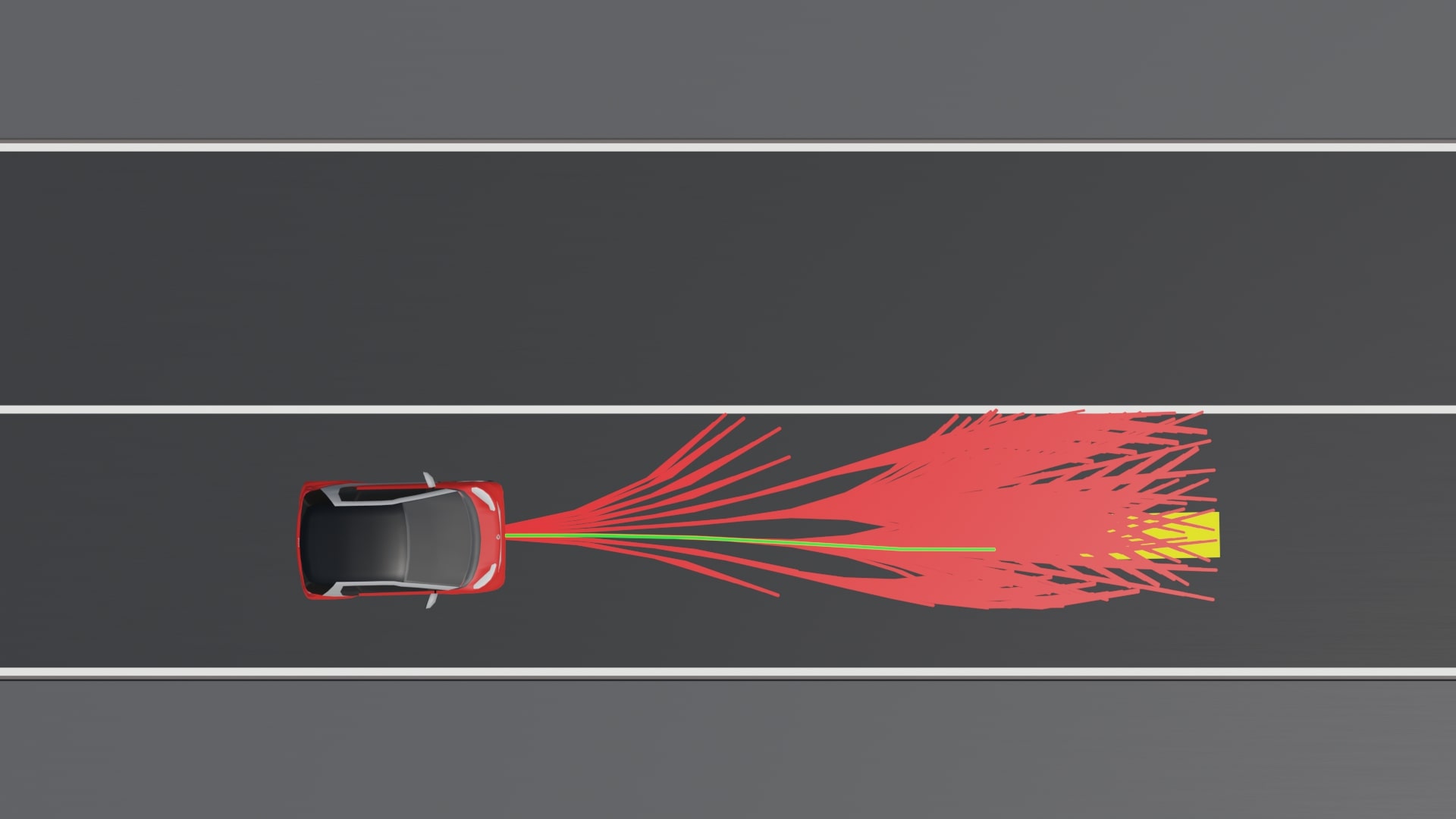}
        \caption{Example of a lattice planner, paired with motion primitives and A* search.}
        \label{fig:astar}
    \end{subfigure}
    \caption{An illustration of an \gls{av} navigating towards the yellow target area. The figure showcases two motion planners: an RRT*-based planner and a lattice planner. 
    The red lines represent the tree of paths generated by each planner, while the green line indicates the solution path identified by the planner.}
    \label{fig:planners}
\end{figure}

Given an agent~$\agent$ and a task~$\task$, the goal is to obtain a set of configurations which are generated by the agent's state inference process, motivated by the concept of deterministic sampling-based motion planning in~\cite{janson2018deterministic}. Technically, for an agent~$\agent$ and a task~$\task$, the set of queries which are generated by the agent's state inference process in all scenario instances of the task is denoted by~$\taskqueries{}(\agent,\task)\subseteq \queryspace$.

\begin{definition}[Task Queries]
Given a task~$\task$, the \emph{task queries} generated by an agent~$\agent$ are the union over all queries of all the scenario instances in the task:
\begin{equation}
    \taskqueries \colon \powerset{\taskspace} \times \agentset \to \powerset{\queryspace},
\end{equation}
such that~$\taskqueries(\agent, \task) \subseteq \queryspace$.
\end{definition}

\subsection{\edited{Modeling perception requirements}}\label{sec:requirements}
\edited{To establish an interface between agent and perception pipeline, the task queries are converted into \emph{class configurations} which need to be detected by the perception pipelines. Such class configurations are referred to as \emph{perception requirements}.} 

The transition from queries to class configuration involves determining which class configurations may collide with the robot at a specific query. 
At a more abstract level, the objective is to identify all class configurations for which the perception pipelines must indicate a collision, when posed with the query. It is important to emphasize that we are looking at agents which can ask for some occupancy queries~\edited{$\query=\tup{\configagent{}^{\robot}, \tau, \environment}$} in the future. It is not just a simple matter of checking which class configurations could collide with the robot at a certain configuration~\edited{$\configagent{}^{\robot}$}. All class configurations at time 0 that would lead to a collision at time~\edited{$\tau$} are needed, given the dynamics of the classes and the class prior~$\prior{i}$ provided by the scenario. \edited{By ``time'' we refer to the planning time, starting at 0, rather than the current time in the scenario.} Therefore, the objective is to derive the set of configurations for all classes within the scenario at time 0, where there exist control inputs that could lead the class to a collision with the robot with configuration\edited{~$\configagent{}^{\robot}$} at time~\edited{$\tau$}. Such class configurations are obtained by sampling the dynamics and going backwards in time. In essence, all configurations generated through the sampling are the ones that the perception layer needs to detect. \edited{An illustration of this process is shown in~\cref{fig:sensreq}.}

\begin{figure}[tb]
    \centering
    \includegraphics[width=0.5\textwidth]{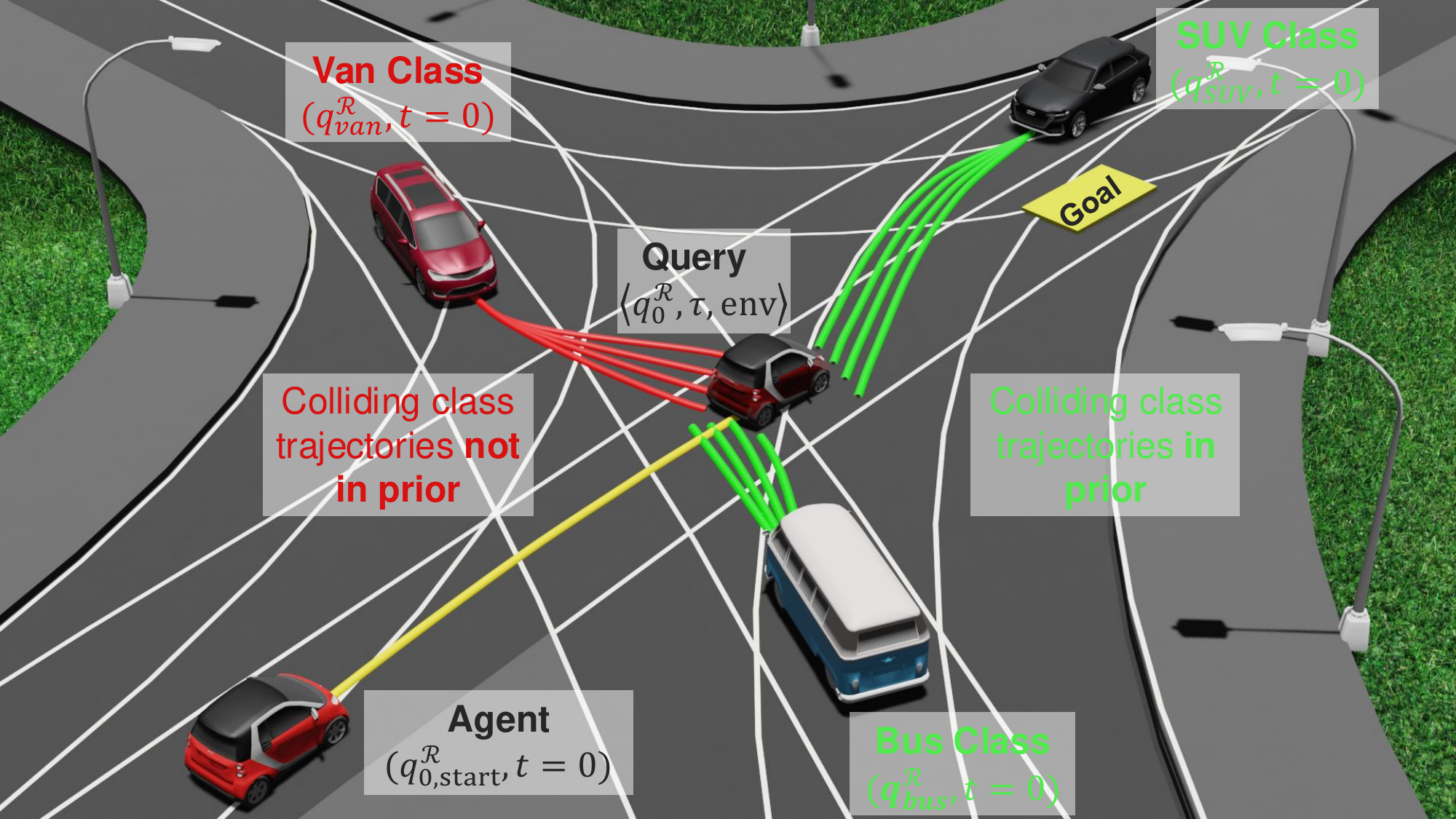}
    \caption{This figure shows class configurations at time 0 leading to potential collisions with a robot at a specific query~\edited{$\query=\tup{\configagent{}^{\robot}, \tau, \environment}$}. The robot is depicted as a small red \gls{av} on the left and the robot's future configuration~\edited{$\configagent{}^{\robot}$} from the query is the transparent \gls{av}  in the intersection's center. Surrounding cars represent classes with trajectories that lead to a collision with the \gls{av} at time~$\tau$ with configuration~\edited{$\configagent{}^{\robot}$}. Green lines show feasible trajectories based on prior knowledge, and a red line shows an infeasible trajectory that violates the prior. The perception requirements in this example are the depicted car \edited{configurations~$\config{\text{suv}}^{\robot}$ and~$\config{\text{bus}}^{\robot}$ with green trajectories}.}
    \label{fig:sensreq}
\end{figure}

The following definitions are used to define the perception requirements of a certain task for a given agent.

\begin{definition}[Collision]\label{def:collision}
Collision is a mapping that generates all possible class configurations in~\edited{$\confspacen{i}^{\robot}$} that are in collision with the robot at a certain configuration~\edited{$\configagent{}^{\robot}$} using their \edited{footprints~$\shape{0}(\configagent{}^{\robot})$ and~$\shape{i}(\config{i}^{\robot})$}.
\edited{
\begin{equation}
\label{eq:collision}
\begin{aligned}
    \collision \colon \confspacenagent{}^{\robot} \times \classinstanceset{} \to \powerset{\confspacen{}^{\robot}},
\end{aligned}
\end{equation}
where~$\classinstanceset{}$ is the set of all class instances.}
\end{definition}

\edited{\begin{definition}[Perceptual Collision Prediction]\label{def:pcp}
    For each query~$\query=\tup{\configagent{}^{\robot},\tau, \environment}$ of an agent~$\agent$, there exist class trajectories~$\trajectorymap{i} \colon [0, \tau] \to \confspace_{i}^{\robot}$ that are the preimage of the class dynamics~$\dynamics{i}$. 
    These trajectories lead to a class configuration~$\trajectorymap{i} (\tau) \in \collision(\configagent{}^{\robot}, \classinstance{i})$ at time $\tau$, starting at time~$0$. This mapping is termed \emph{perceptual collision prediction}:

    \begin{equation}
        \label{eq:pcp}
        \pcp_i \colon \powerset{\queryspace} \to \powerset{\trajectorymapspace{i}(\reals^+,\confspacen{i}^{\robot})},
    \end{equation}
    where $\trajectorymapspace{i}(\reals^+,\confspacen{i}^{\robot})$ is the set of all trajectories $\trajectorymap{i}(t)$.
\end{definition}}
\edited{\begin{definition}[Prior Check]\label{def:priorcheck}
    The prior check is a function that evaluates a set of trajectories~$\trajectorymap{i} \colon [0, \tau] \to \confspacen{i}^{\workspace}$ and determines whether all configurations along each trajectory, from the initial time~$\trajectorymap{i}(0)$ to the final time~$\trajectorymap{i}(\tau)$, are contained within the prior~$\prior{i}$. For those trajectories satisfying this condition, the function returns the starting configuration~$\trajectorymap{i}(0)$.

    \begin{equation}
        \label{eq:priorcheck}
        \priorcheck \colon \powerset{\trajectorymapspace{i}(\reals^+,\confspacen{i}^{\workspace})} \times \powerset{\confspacen{i}^{\workspace}} \to \powerset{\confspacen{i}^{\workspace}}.
    \end{equation}
\end{definition}}

\begin{definition}[Task Perception Requirements]
The perception requirements for an agent~$\agent$ undertaking a task~$\task{}$ are defined as the mapping from task queries~$\taskqueries(\agent,\task)$ 
to all possible subsets of class configurations for each environment~$\environment$ within the task. 
\edited{This mapping is established by transforming queries into colliding class trajectories through~$\collision$ and perceptual collision prediction~$\pcp$. The resulting colliding class trajectories are transformed to the global frame using the agent’s global configuration when the corresponding queries were performed. Finally, feasible starting configurations from the colliding class trajectories are filtered using~$\priorcheck$ and transformed back to the ego frame}:
 \begin{equation}
\label{eq:tpr}
\sensreq \colon \agentset \times \powerset{\taskspace} \to \prod_{\environment{} \in \environmentset{}} \prod_{k\in \{1,\ldots \kclass\}}\powerset{\edited{\confspacen{k}^{\robot}}}, 
\end{equation}
where~$\kclass$ is the number of unique object class instances in the task and~$\environmentset$ is the set of all environments in the task.
\end{definition}

For a given object class instance~$\classinstance{i}$ and environment~$\environment$ within task perception requirement~$\sensreq(\agent,\task)$, we express this as~$\sensreq(\agent,\task, \classinstance{i}, \environment)$, 
indicating that~$\sensreq(\agent,\task, \classinstance{i}, \environment) \subseteq \confspacen{i}^{\robot}$.

\edited{
\begin{remark}\label{rem:percreq}
    As stated in the computation assumptions in~\cref{sec:approach}, queries are gathered through simulation, where each scenario instance, agent, and robot body is simulated, and the generated queries are stored. The transformation from queries to perception requirements is performed offline after simulations by randomly sampling colliding class configurations for each query. We then sample a finite set of random trajectories and transform them such that their end configurations coincide with the colliding configurations at the specified time in the query. Next, the generated colliding class trajectories are transformed to the world frame and the infeasible trajectories which are not in the prior are removed. The feasible trajectories are transformed back to the ego vehicle frame, with their starting configurations taken as the perception requirements. We simulate continuously, performing this post-processing after each simulation and taking the union with perception requirements from previous runs. We then apply our inner and outer optimization, allowing the solution to evolve continuously and converge over time.
\end{remark}
}

\subsection{\edited{Modeling perception performance}}\label{sec:sensing}
The next step is to evaluate the capabilities of a perception pipeline, including sensor hardware and perception software, to measure and provide the perception requirements of an agent. \edited{These detection capabilities, denoted as \emph{perception performance}, are represented in terms of \gls{fpr} and \gls{fnr}, representing the probability of generating a false detection when no object is present and the probability of missing an object when it is present, respectively. For each perception pipeline~$\sensorp_{j}$, class configuration in~$\confspacen{i}^{\sensorp_{j}}$, object class instance~$\classinstance{i}$, appearance~$\appearance_i$ and environment~$\environment$, the function~$\ppp$ maps the confidence interval of the \gls{fnr} and \gls{fpr}, respectively. The superscript in~$\confspacen{i}^{\sensorp_{j}}$ indicates the sensor coordinate frame}.
\edited{
\begin{equation}
\ppp \colon \confspacen{i}^{\sensorpset} \times \appearanceset_i \times \sensorpset \times \environmentset \to \powerset{I} \times \powerset{I}.
\end{equation}
}
The set~$I$ is the set of all intervals:~$[a,b] \subseteq \reals : 0 \leq a \leq b \leq 1$. The obtained interval~$[a,b]$ represents the confidence interval with a lower bound~$a$ and upper bound~$b$ of the perception pipeline's \gls{fnr} and \gls{fpr}. 
During our selection process, we use the upper bound to conduct a \emph{worst-case analysis}.
With the variables~$\appearance$,~$\sensorp$ and~$\environment$ we summarize other relevant parameters for representing the \gls{fnr} and \gls{fpr} as for instance the object size, object color, sensor resolution or weather condition. An illustration of the perception performance with two distinct perception pipelines is shown in~\cref{fig:sensor_perf}.

\edited{
\begin{remark}\label{rem:bench}
The selected variables are chosen for their impact~\cite{milojevicphd2024}, as supported by literature~\cite{hoss2022review}, our results~\cite{milojevicphd2024} and constrained annotation data~\cite{caesar2020nuscenes} of the sensor measurements, without claiming they encompass all influential factors in general. The implementation of the \gls{fnr} and \gls{fpr} in $\ppp$ is not the focus of this work. Briefly, we use real sensor data (e.g., from the nuScenes dataset~\cite{caesar2020nuscenes}) and pre-trained 3D object detection algorithms (e.g., from the MMDetection3D library~\cite{mmdet3d2020}) to perform inference on the test dataset. This process generates \gls{fn}, \gls{fp}, and \gls{tp} events, from which we extract features such as relative radial distance to the sensor, relative orientation, object size, relative velocity, or light conditions (e.g., night or day). Using this data, we perform binary classification with Gaussian process classification~\cite{milios2018dirichlet} on the extracted features. The \gls{fn} and \gls{tp} events model the \gls{fnr}, while \gls{fp} and \gls{tp} events model the \gls{fpr}. The entire data flow is illustrated in~\cref{fig:perceptionbenchalgorithm}.
\end{remark}
}
\begin{figure} 
    \centering
    \begin{subfigure}[b]{0.22\textwidth}
        \centering
        \includegraphics[width=\textwidth]{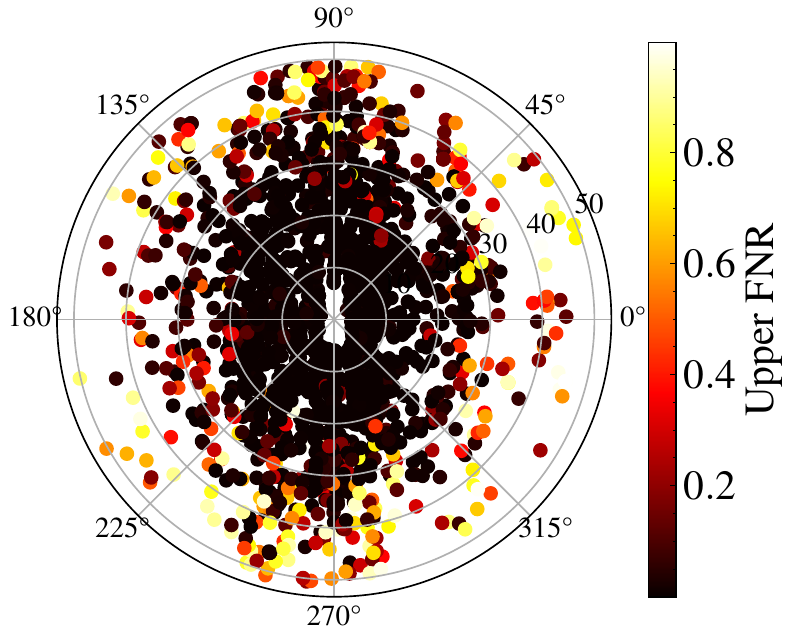}
    \end{subfigure}
    \begin{subfigure}[b]{0.22\textwidth}
        \centering
        \includegraphics[width=\textwidth]{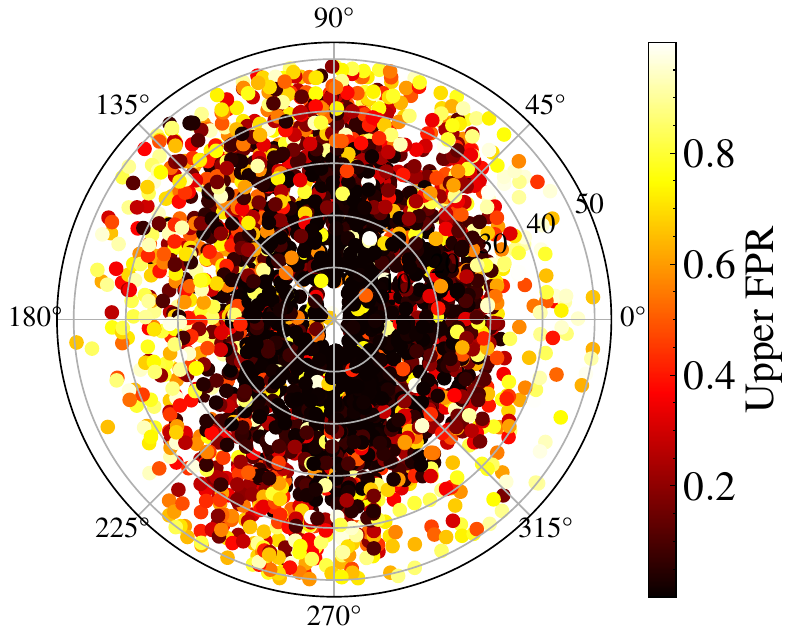}
    \end{subfigure}
    \begin{subfigure}[b]{0.22\textwidth}
        \centering
        \includegraphics[width=\textwidth]{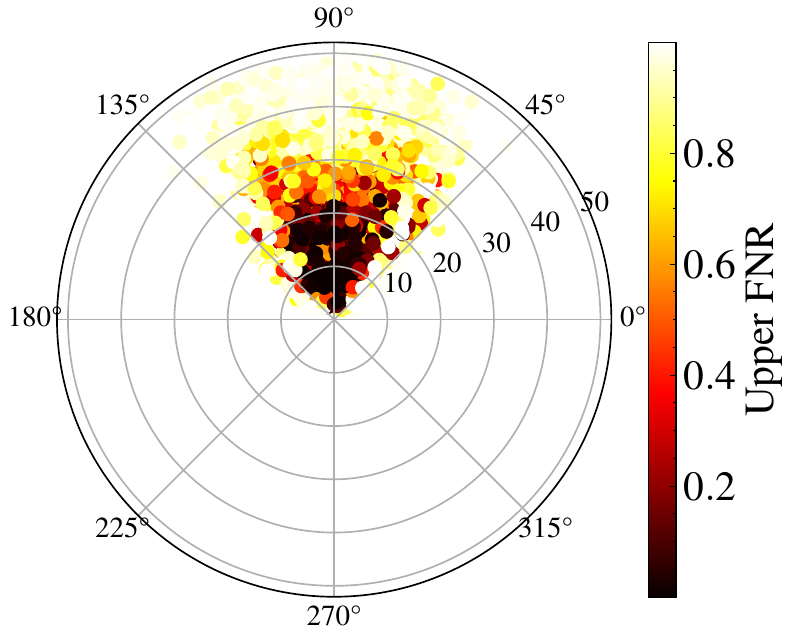}
    \end{subfigure}
    \begin{subfigure}[b]{0.22\textwidth}
        \centering
        \includegraphics[width=\textwidth]{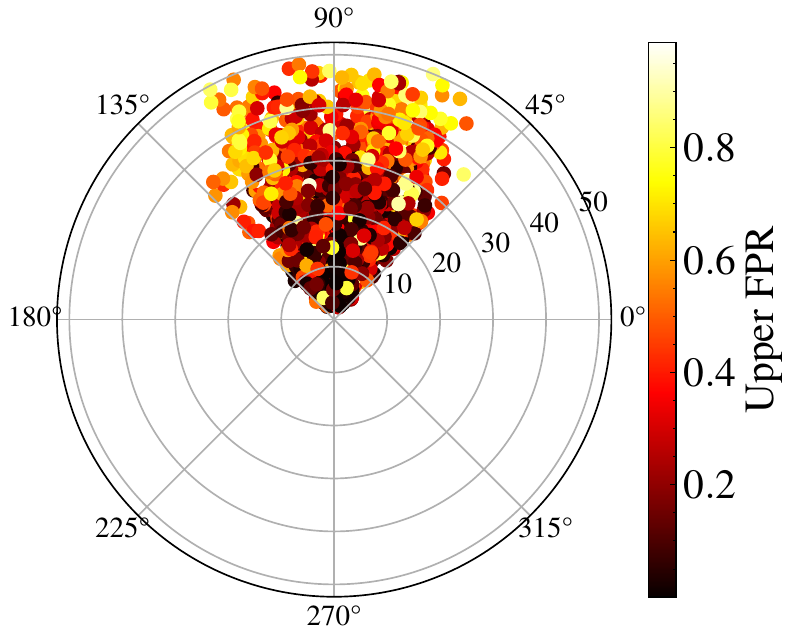}
    \end{subfigure}
    \caption{Comparison of the perception performance of two pipelines: Velodyne HDL32E lidar with PointPillars detection model (top plots)~\cite{Lang2020PointPillars:Clouds} and Basler acA1600-60gc camera with FCOS3D detection model (bottom plots)~\cite{wang2021fcos3d}. 
    Left plots show FNRs and right plots FPRs, highlighting the upper bounds of confidence intervals against radial distance~$r$ and relative orientation~$\theta$ between sensor and object class in polar coordinates. Data is from the nuScenes dataset~\cite{caesar2020nuscenes}, using models from the MMDetection3D~\cite{mmdet3d2020} library.}
    \label{fig:sensor_perf}
\end{figure}
\begin{figure}[tb]
    \centering
    \includegraphics[width=0.5\textwidth]{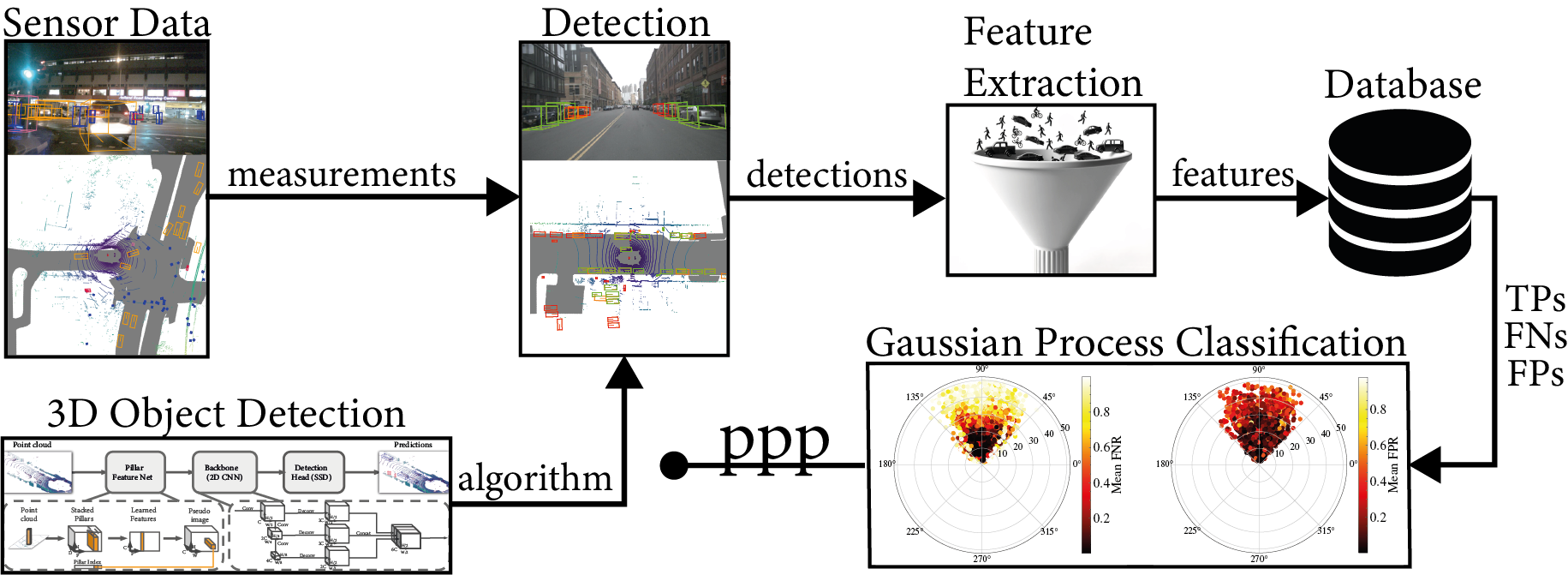}
    \caption{\edited{Graphical representation of the simplified data flow for the entire benchmarking process, 
    from sensor measurements to the calculation of \gls{fnr} and \gls{fpr} (sensor data taken from~\cite{caesar2020nuscenes}, 
    model architecture image from~\cite{Lang2020PointPillars:Clouds}).}}
    \label{fig:perceptionbenchalgorithm}
\end{figure}
\section{Solving the robot co-design problem}
\label{sec:codesign}
In this chapter, we establish an optimization framework for determining the optimal robot design tailored to a specific task, leveraging a monotone theory of co-design~\cite{Censi2015, censi2024}. The primary objective is the minimization of resource consumption, which includes power consumption, robot body mass, cost and computing resources. 
~\cref{sec:codesigntheory} introduces the basic principles of co-design.
Subsequently, in~\cref{sec:codesign:task:pr,,sec:sspp} we address task-oriented co-design of a complete mobile robot.~\cref{sec:sspp,,sec:solve:ssp} addresses the sensor selection and placement problem, which forms the \edited{inner} optimization, using the formulations introduced in \cref{sec:model}. 

\subsection{Background on a monotone theory of co-design}
\label{sec:codesigntheory}
The reader is assumed to be familiar with posets and basic concepts of order theory (a good source is~\cite{Davey2002b}).

\paragraph{Formulating co-design problems}
The atom of the theory is the notion of a \gls{abk:mdpi}, through which we will model different components of the autonomy stack.
\begin{definition}
Given \glspl{abk:poset}~$\setOfFunctionalities{},\setOfResources{}$,  (mnemonics for \F{functionalities} and \R{resources}), we define a \emph{\gls{abk:mdpi}} as a tuple~$\tup{\setOfImplementations{d},\prov, \req}$, where~$\setOfImplementations{d}$ is the set of implementations, and~$\prov$,~$\req$ are maps from~$\setOfImplementations{d}$ to~$\setOfFunctionalities{}$ and~$\setOfResources{}$, respectively:
\begin{equation*}
        \setOfFunctionalities{} \xleftarrow{\prov} \setOfImplementations{d} \xrightarrow{\req} \setOfResources{}.
\end{equation*}
We compactly denote the \gls{abk:mdpi} as~$d\colon \setOfFunctionalities{} \tickar \setOfResources{}$.
Furthermore, to each \gls{abk:mdpi} we associate a monotone map~$\bar{d}$, given by:
\begin{equation*}
    \begin{aligned}
        \bar{d}\colon \setOfFunctionalitiesOp{} \times \setOfResources{} &\to \tup{\power{\setOfImplementations{d}},\subseteq}\\
        \langle \F{f}^*,\R{r}\rangle &\mapsto \{i \in \setOfImplementations{d}\colon (\prov(i) \succeq_{\setOfFunctionalities{}}\F{f}) \wedge (\req(i)\preceq_{\setOfResources{}}\R{r})\},
    \end{aligned}
\end{equation*}
where~$(\cdot)\op$ reverses the order of a \gls{abk:poset}. 
The expression~$\bar{d}(\F{f}^*,\R{r})$ returns the set of implementations (design choices)~$S\subseteq \setOfImplementations{d}$ for which \F{functionalities}~$\F{f}$ are feasible with \R{resources}~$\R{r}$.
A \gls{abk:mdpi} is represented in diagrammatic form as a block with green wires on the left for functionalities, and dashed red ones on the right for resources, as visualized in~\cref{fig:codesignmodel}.

\begin{remark}[Monotonicity]
What does monotonicity mean in this context? 
Consider a \gls{abk:mdpi} for which~$\bar{d}(\F{f}^*,\R{r})=S$:
\begin{itemize}
    \item One has:~$\F{f'}\preceq_{\setOfFunctionalities{}} \F{f}\Rightarrow \bar{d}(\F{f'}^*,\R{r})=S'\supseteq S$.
    Intuitively, decreasing the provided functionalities will not increase the required resources;
    \item One has:~$\R{r'}\succeq_{\setOfResources{}} \R{r}\Rightarrow \bar{d}(\F{f}^*,\R{r'})=S''\supseteq S$.
    Intuitively, increasing the available resources cannot decrease the provided functionalities.
\end{itemize}
\end{remark}
\end{definition}

\begin{remark}[Populating the models]
The presented framework is very flexible.
In practice, one populates the \glspl{abk:mdpi} via analytic relations (e.g., cost functions), numerical analysis of closed-form relations (e.g., solving optimal control problems), and in a data-driven, on-demand fashion (e.g., via POMDPs, simulations, or by solving instances of optimization problems).
For detailed examples related to mobility and autonomy, please refer to~\cite{Zardini2021d, zardini2021co, zardini2022co, zardini2023co, zardini2022task, censi2024}.
\end{remark}

One can compose individual \glspl{abk:mdpi} in several ways to form a co-design problem (i.e., a multigraph of \glspl{abk:mdpi}, where nodes are \glspl{abk:mdpi}, and edges their interconnections), which is again a \gls{abk:mdpi} (i.e., closure).
This makes the presented framework practical to decompose a large problem into smaller ones, and to interconnect them\footnote{A detailed list of compositions is provided in~\cite{censi2024,zardini2023co}. Formally, their specification makes the category of design problems a traced monoidal category, with locally posetal structure.}{}.
Series composition happens when the functionality of a \gls{abk:mdpi} is required by another \gls{abk:mdpi} (e.g., information acquired by a sensor is processed by an estimator). 
The symbol~$\preceq$ is the posetal relation, representing a co-design constraint: the resource a problem requires cannot exceed the functionality another problem provides.
Parallel composition, instead, formalizes decoupled processes happening together.
Finally, loop composition describes feedback.

\paragraph{Solving co-design problems}
Given a \gls{abk:mdpi}, we essentially have two queries.
First, given some desired functionalities, find the optimal design solutions which minimize resources (FixFunMinRes). 
Alternatively, given some available resources, find the optimal design choices which maximize functionalities (FixResMaxFun).

\begin{definition}
\label{def:h_map}
Given a \gls{abk:mdpi}~$d$, one defines monotone maps
\begin{itemize}
    \item $h_d\colon \setOfFunctionalities{}\to \mathsf{A}\setOfResources{}$, mapping a functionality to the \emph{minimum} antichain of resources providing it;
    \item $h_d'\colon \setOfResources{}{}\to \mathsf{A}\setOfFunctionalities{}$, mapping a resource to the \emph{maximum} antichain of functionalities provided by it.
\end{itemize}
\end{definition}
Solving \glspl{abk:mdpi} requires finding such maps.
If such maps are Scott continuous, and posets are complete, one can rely on Kleene's fixed point theorem to design an algorithm solving both queries (and returning the related optimal design choices).

Interestingly, the resulting algorithm is guaranteed to converge to the set of optimal solutions, or to provide a certificate of infeasibility.
Furthermore, the complexity of solving such problems is only linear in the number of options available for each component (as opposed to combinatorial). 
For more details, refer to~\cite{censi2024, zardini2023co}.
\begin{figure*}[h!]
    \centering
    \includegraphics[width=0.9\textwidth]{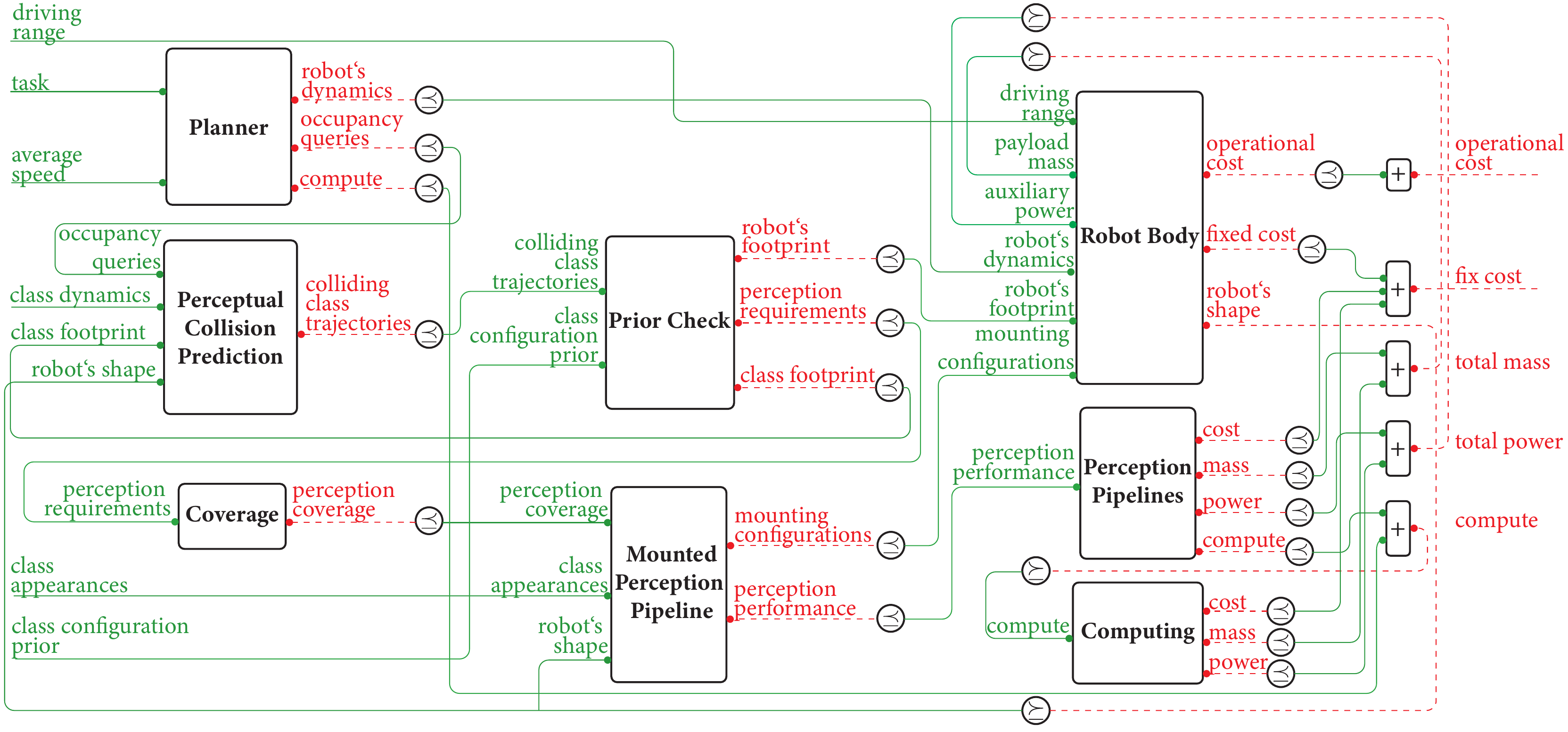}
    \caption{The co-design diagram for the design of a mobile robot tailored to accomplish a \F{task}, including a collection of scenario instances and class instances, aiming to achieve specified \F{average speed} and \F{driving range}. \edited{The class instances include \F{class dynamics}, \F{class appearances}, and \F{class configurations prior}.} The objective is to minimize the \R{fix cost} and \R{operational cost}.}
    \label{fig:codesignmodel}
\end{figure*}
\subsection{Modeling from task to perception requirements}\label{sec:codesign:task:pr}
First, we describe the \textbf{Planner} \gls{abk:mdpi} in~\cref{fig:mdpi_planner}, representing the choice of motion planner for the robot. It provides a set of scenario instances representing the \F{task} $\task$ as a functionality and 
the \F{average speed} in \unit[]{km/h} the planner navigates the \gls{av} across all scenario instances, indicating the task performance. The Planner \gls{abk:mdpi} requires \R{occupancy queries} $\queryspace{}$, \R{compute} and the robot's \R{dynamics} resources. The more scenario instances are required, the more queries are needed by the planner, as detailed in~\cref{lem:queries_monotone}. The \R{compute} resource encompasses computational capabilities, including CPU and GPU performance, quantified by operations per second and available memory. 
An increase in collision checks for occupancy queries leads to a higher demand for \R{compute} resources. \R{Robot's dynamics} are characterized by parameters such as minimum turning radius, maximum acceleration, and maximum deceleration. Higher acceleration and deceleration expand the range of possible queries, enabling faster achievement of goals in scenario instances. A smaller minimum turning radius increases the diversity of occupancy queries and the robot's capability to navigate through complex scenarios, 
such as tight passages that a large turning radius would not permit. Consequently, we utilize the opposite of a poset for minimum turning radius. Additionally, greater acceleration necessitate more computational resources to quickly process planning strategies.
Extending the \F{average speed} requires improved dynamics with quicker acceleration, or a more efficient planner, which increases the need for \R{compute} resources and \R{occupancy queries}.

\begin{lemma}
\label{lem:queries_monotone}
The task \R{occupancy queries} $\taskqueries$ is monotone in the \F{task}, as shown in~\cref{fig:mdpi_planner}.
\end{lemma}
\begin{proof}
Consider two tasks $\task_1\subseteq \task_2$. We have 
\begin{equation*}
\begin{aligned}
\taskqueries(\agent,\task_1)&\subseteq 
\big ( \taskqueries(\agent,\task_1) \cup \taskqueries(\agent,\task_2 \setminus \task_1) \big )\\
&=\taskqueries(\agent,\task_2).
\end{aligned}
\end{equation*}
\end{proof}

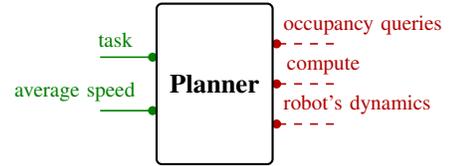
\begin{figure}[tb]
    \centering
    \scalebox{1.0}{\begin{tikzpicture}[DP, dp port sep=3pt]
    \node[dp={2}{3}] (ch) {Planner};
    \draw[runconn, runame={occupancy queries},relres=above, posres=1.5] (ch_res1){};
    \draw[runconn, runame={compute},relres=above, posres=0.8] (ch_res2){};
    \draw[runconn, runame={robot's dynamics},relres=above, posres=1.4] (ch_res3){};
    \draw[funconn, funame={task},relfun=left,relfun=above, posfun=0.7] (ch_fun1){};
    \draw[funconn, funame={average speed},relfun=left,relfun=above, posfun=1.5] (ch_fun2){};
\end{tikzpicture}}
    \caption{The planner \gls{abk:mdpi} which implements a motion planner for the robot to accomplish scenario instances of a \F{task} and thereby providing \F{average speed}, while requiring \R{occupancy queries} $\queryspace{}$, \R{compute} and \R{robot's dynamics}}
    \label{fig:mdpi_planner}
\end{figure}

The \textbf{Perceptual Collision Prediction} \gls{abk:mdpi}, visualized in~\cref{fig:mdpi_wtl}, describes the $\pcp{}$ function to determine all potential feasible \R{colliding class trajectories}~\edited{$\trajectorymapspace{}$} that could result in collisions with the robot at the \F{occupancy queries} from the planner. This guides the perception system to focus on critical areas based on the \F{occupancy queries} and the \F{class dynamics}. \edited{Consequently, the \F{occupancy queries} $\queryspace{}$, \F{class dynamics}, \F{class footprint} and \F{robot's shape} $\threedshape$ in $\reals^3$ serve as functionalities of this \gls{abk:mdpi}. In this context, we do not consider a poset structure for the footprint and shape based on area or volume but instead use set-wise inclusion on~$\powerset{\reals^2}$ and~$\powerset{\reals^3}$, respectively. Throughout this work, when referring to a ``larger'' or ``bigger'' footprint or shape, we mean that the footprint or shape has been expanded such that the original is a subset of the new one.} The \F{class dynamics}, including minimum turning radius, maximum acceleration, and deceleration, are specified similarly to the robot's dynamics. Again, the minimum turning radius is treated the opposite of a poset. The \F{class footprint} is the planar shape of the class in 2D, generated by the map $\shape{}$. \edited{The resources include \R{colliding class trajectories} $\trajectorymapspace{}$}.~\cref{lem:pcp_mono} illustrates the monotonic relationship between \R{colliding class trajectories} and \F{occupancy queries}, indicating that an increase in \F{occupancy queries} leads to an equal or greater number of \R{colliding class trajectories}. This relationship also applies to \F{class dynamics}, altering \F{class dynamics} results in new \R{colliding class trajectories}. Specifically, higher acceleration and deceleration and a smaller minimum turning radius produce a broader range of \R{colliding class trajectories}. \edited{A larger \F{robot's shape} or larger \F{class footprint} results in more class configurations being in collision by keeping the \F{class dynamics} constant, with ``more'' understood in the sense of set-wise inclusion, leading to more \R{colliding class trajectories}}.

\begin{lemma}
\label{lem:pcp_mono}
The \R{colliding class trajectories} from $\pcp$ are monotone with respect to the \F{occupancy queries} as shown in~\cref{fig:mdpi_wtl}.
\end{lemma}
\begin{proof}
Consider two query sets $\queryspace_1\subseteq \queryspace_2$. We have 
\begin{equation*}
\begin{aligned}
\pcp_i(\queryspace_1)&\subseteq 
\big (\pcp_i(\queryspace_1) \cup \pcp_i(\queryspace_2 \setminus \queryspace_1)\big )\\
&=\pcp_i(\queryspace_2).
\end{aligned}
\end{equation*}
\end{proof}

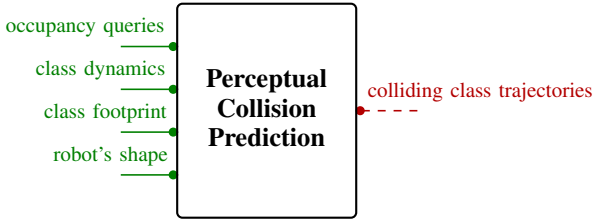
\begin{figure}[tb]
    \centering
    \scalebox{1.0}{\begin{tikzpicture}[DP, dp port sep=3pt]
    \node[dp={4}{1}] (ch) {\begin{tabular}{c}Perceptual \\ Collision \\ Prediction\end{tabular}};
    \draw[runconn, runame={colliding class trajectories},relres=above, posres=2.1] (ch_res1){};
    \draw[funconn, funame={occupancy queries},relfun=above, posfun=1.7] (ch_fun1){};
    \draw[funconn, funame={class dynamics},relfun=above, posfun=1.4] (ch_fun2){};
    % \draw[funconn, funame={class footprint$^{\mathrm{op}}$},relfun=above, posfun=1.5] (ch_fun3){};
    \draw[funconn, funame={class footprint},relfun=above, posfun=1.3] (ch_fun3){};
    \draw[funconn, funame={robot's shape},relfun=above, posfun=1.2] (ch_fun4){};
\end{tikzpicture}}
    \caption{The Perceptual Collision Prediction \gls{abk:mdpi} which implements the function $\pcp{}$. 
    \edited{The functionalities are the \F{occupancy queries} for the planner, the \F{class dynamics}, the \F{class footprint} and the \F{robot's shape}. The required resources are the \R{colliding class trajectories}.}}
    \label{fig:mdpi_wtl}
\end{figure}

The \textbf{Prior Check} \gls{abk:mdpi}, illustrated in~\cref{fig:mdpi_prior}, describes the~$\priorcheck$ function as outlined in~\cref{def:priorcheck}. 
The function~$\priorcheck$ takes all start configurations from the colliding class trajectories, which trajectory configurations are all in the prior $\prior{}$ of the class. \edited{Additionally, this \gls{abk:mdpi} removes infeasible colliding class trajectories that cause overlapping (not just touching) between robot and class footprints or those starting from a collision.} Thus, the functionalities are the \F{class configurations prior}, where classes can be in the scenario instance, and the \F{colliding class trajectories}~\edited{$\trajectorymapspace{}$} generated by $\pcp$. \edited{The resources are the final \R{perception requirements} $\sensreq$, the \R{robot's footprint}~$\shape{0}$ and the \R{class footprint}~$\shape{i}$ where the poset for the \R{robot's footprint} and the \R{class footprint} is defined by set-wise inclusion on~$\powerset{\reals^2}$.} According to~\cref{lem:priorcheck_mono}, priors that encompass more class configurations tend to filter out fewer configurations during~$\priorcheck$, resulting in more perception requirements. Given the relations established in~\cref{lem:queries_monotone} and~\cref{lem:pcp_mono}, where more complex tasks generate more \F{colliding class trajectories}, it follows, as demonstrated in~\cref{lem:sr_monotone}, that increased task complexity (more \F{colliding class trajectories}) also amplifies the \R{perception requirements}. \edited{If more \F{colliding class trajectories} or a prior with additional class configurations is required while keeping the \R{perception requirements} constant, the \R{robot's footprint} or the \R{class footprint} must be increased such that the additional start configurations from the \F{colliding class trajectories} are already in collision with the robot. In this way, the \R{robot's footprint} effectively acts as a perception pipeline. For example, if a \R{robot's footprint} encompasses $\reals^2$, no class trajectory can collide with it, as the robot already occupies all available space.}

\begin{lemma}
\label{lem:priorcheck_mono}
The class configurations in the \R{colliding class trajectories} are monotone with respect to the \F{class configurations prior} as shown in~\cref{fig:mdpi_prior}.
\end{lemma}
\begin{proof}
Consider two priors~ $\prior{i,1} \subseteq \prior{i,2}$ and a class configuration set $\edited{\Theta_i^{\workspace}}$. If $\edited{\Theta_i^{\workspace}} \subseteq \prior{i,1}$ then it holds also~$\edited{\Theta_i^{\workspace}} \subseteq \prior{i,2}$. If~$\edited{\Theta_i^{\workspace}} \subseteq \prior{i,2} \setminus \prior{i,1}$, then $\edited{\Theta_i^{\workspace}} \subseteq \prior{i,2}$ but $\edited{\Theta_i^{\workspace}} \cap \prior{i,1} = \emptyset$.
\end{proof}

\begin{lemma}
\label{lem:sr_monotone}
The \R{perception requirements} $\sensreq$ are monotone in the task, respectively in the \F{colliding class trajectories} (\cref{fig:mdpi_prior}).
\end{lemma}
\begin{proof}
Consider two tasks $\task_1\subseteq \task_2$. From~\cref{lem:queries_monotone} we know that occupancy queries are monotone in the task 
and from~\cref{lem:pcp_mono} we know that \F{colliding class trajectories} are monotone with the queries. We have 
\begin{equation*}
\begin{aligned}
\sensreq(\agent,\task_1)&\subseteq 
\big ( \sensreq(\agent,\task_1) \cup \sensreq(\agent,\task_2 \setminus \task_1) \big )\\
&=\sensreq(\agent,\task_2).
\end{aligned}
\end{equation*}
\end{proof}

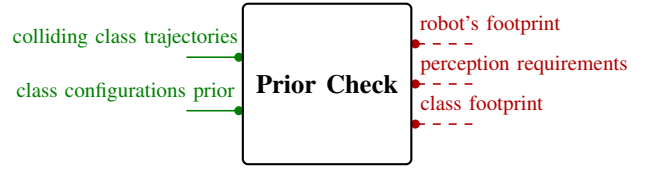
\begin{figure}[tb]
    \centering
    \scalebox{1}{\begin{tikzpicture}[DP, dp port sep=3pt]
    \node[dp={2}{3}] (ch) {Prior Check};
    \draw[runconn, runame={robot's footprint},relres=above, posres=1.3] (ch_res1){};
    \draw[runconn, runame={perception requirements},relres=above, posres=1.9] (ch_res2){};
    \draw[runconn, runame={class footprint},relres=above, posres=1.15] (ch_res3){};
    \draw[funconn, funame={colliding class trajectories},relfun=above, posfun=2.2] (ch_fun1){};
    \draw[funconn, funame={class configurations prior},relfun=above, posfun=2.2] (ch_fun2){};
\end{tikzpicture}}
    \caption{The Prior Check \gls{abk:mdpi}, which implements $\priorcheck{}$, 
    provides \F{class configurations prior} and \F{colliding class trajectories} functionalities and \edited{requires \R{robot's footprint}, \R{perception requirements} and \R{class footprint}}.}
    \label{fig:mdpi_prior}
\end{figure}

The \textbf{Robot body} \gls{abk:mdpi} in~\cref{fig:mdpi_robotbody} encompasses the characteristics of the robot body $\body{}$, such as the \F{robot's dynamics}, sensor \F{mounting configurations}, 
the \F{robot's footprint}, the maximum \F{payload mass} capacity, the \F{auxiliary power} capability, and \F{driving} range. 
This \gls{abk:mdpi} provides the \F{robot's dynamics} functionality, parameterized as minimum turning radius (considered opposite of a poset), maximum acceleration, and deceleration. Additionally, 
it outlines \F{mounting configurations} for sensors within $\sethree$, the \F{robot's footprint} $\shape{0}$ in $\reals^2$, the maximum \F{payload mass} in \unit[]{kg} the robot can carry, its \F{auxiliary power} capacity in \unit{}{W} for powering hardware such as sensors and computers,
and the \F{driving range} in \unit[]{m} representing the robot's driving range without recharge.
Requirements for this \gls{abk:mdpi} include the \R{robot's shape} $\threedshape$ in $\reals^3$, associated with \R{fixed costs} in \unit[]{CHF} and \R{operational costs} in \unit[]{CHF/m}. 
Enhanced \F{robot's dynamics}, such as greater acceleration/deceleration and a reduced turning radius, typically necessitate higher  \R{fixed costs} and \R{operational costs}. 
Similarly, increasing the \F{payload mass} and \F{auxiliary power} capacity implies a need for a more costly or larger \R{robot's shape}. Boosting the \F{driving range} involves augmenting the battery size, impacting both \R{fixed} and \R{operational costs}. 
Additional sensor \F{mounting configurations} may necessitate a larger \R{robot's shape} to accommodate the setup.
As aforementioned, a larger \F{robot's footprint} can potentially reduce perception requirements by obstructing more colliding class trajectories. Achieving a larger \F{robot's footprint} requires a correspondingly larger \R{robot's shape}. 
\begin{figure}[tb]
    \centering
    \scalebox{1}{\begin{tikzpicture}[DP, dp port sep=3pt]
    \node[dp={6}{3}] (ch) {Robot Body};
    \draw[runconn, runame={fixed cost},relres=above, posres=0.8] (ch_res1){};
    \draw[runconn, runame={operational cost},relres=above, posres=1.2] (ch_res2){};
    \draw[runconn, runame={robot's shape},relres=above, posres=1.0] (ch_res3){};
    \draw[funconn, funame={mounting configurations},relfun=above, posfun=2.1] (ch_fun1){};
    \draw[funconn, funame={payload mass},relfun=above, posfun=1.2] (ch_fun2){};
    \draw[funconn, funame={auxiliary power},relfun=above, posfun=1.4] (ch_fun3){};
    \draw[funconn, funame={driving range},relfun=above, posfun=1.2] (ch_fun4){};
    \draw[funconn, funame={robot's dynamics},relfun=above, posfun=1.5] (ch_fun5){};
    \draw[funconn, funame={robot's footprint},relfun=above, posfun=1.4] (ch_fun6){};
\end{tikzpicture}}
    \caption{The Robot body \gls{abk:mdpi} which provides the \F{dynamics} $\dynamics{}$, the \F{mounting configurations} for sensors each in $\sethree$, 
    the \F{body footprint} $\shape{0}$, 
    the \F{payload mass} in \unit[]{kg}, the \F{auxilary power} in \unit[]{W} and the \F{driving range} in \unit[]{m}, while requiring \R{robot's shape} $\threedshape$, 
    \R{fixed cost} in \unit[]{CHF} and \R{operational costs} in \unit[]{CHF/m}.}
    \label{fig:mdpi_robotbody}
\end{figure}
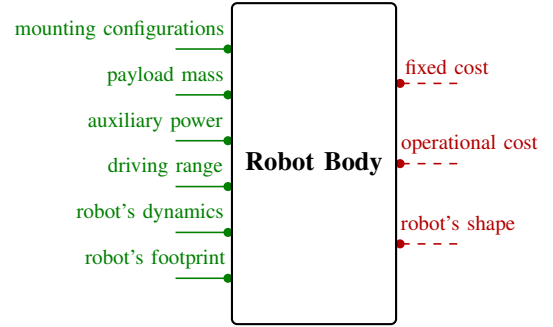

The \textbf{Computing} \gls{abk:mdpi}, visualized in~\cref{fig:mdpi_computing}, implements the computing units necessary for the robot's software operations, 
including both motion planning and perception. It provides computational capabilities as a functionality in terms of CPU and GPU performance, measured in memory capacity and operations per second. These computational capabilities are encapsulated as \F{compute}. The provision of \F{compute} is directly linked to associated \R{cost} in \unit[]{CHF}, \R{mass} in \unit[]{kg} and \R{power} consumption in \unit[]{W}.
As the demand for \F{compute} increases to accommodate more sophisticated software algorithms or larger data volumes, 
the specifications of the computing units must be scaled up accordingly. This, in turn, impacts the overall \R{cost} of the computing hardware, its \R{mass}, and its \R{power} consumption.

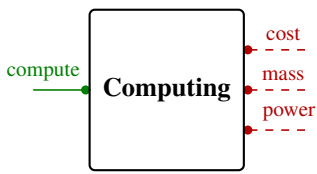
\begin{figure}[tb]
    \centering
    \scalebox{1}{
\begin{tikzpicture}[DP, dp port sep=3pt]
    \node[dp={1}{3}] (cpu) {Computing};
    \draw[runconn, runame={cost},relres=above, posres=0.6] (cpu_res1){};
    \draw[runconn, runame={mass}, relres=above, posres=0.6] (cpu_res2){};
    \draw[runconn, runame={power}, relres=above, posres=0.7] (cpu_res3){};
    \draw[funconn, funame={compute},relfun=above, posfun=0.8] (cpu_fun1){};
\end{tikzpicture}}
    \caption{The Computing \gls{abk:mdpi} which implements the computing units. 
    It provides \F{compute} and requires \R{cost} in \unit[]{CHF}, the \R{mass} in \unit[]{kg} and \R{power} consumption in \unit[]{W}.}
    \label{fig:mdpi_computing}
\end{figure}

\subsection{Sensor selection and placement problem}\label{sec:sspp}
This section introduces a methodology to obtain the relationship between perception pipelines and perception requirements for a particular task, while accounting for resource consumption (see~\cref{fig:alg_ssp_flow}).
\begin{figure}[tb]
    \centering
    \includegraphics[width=0.5\textwidth]{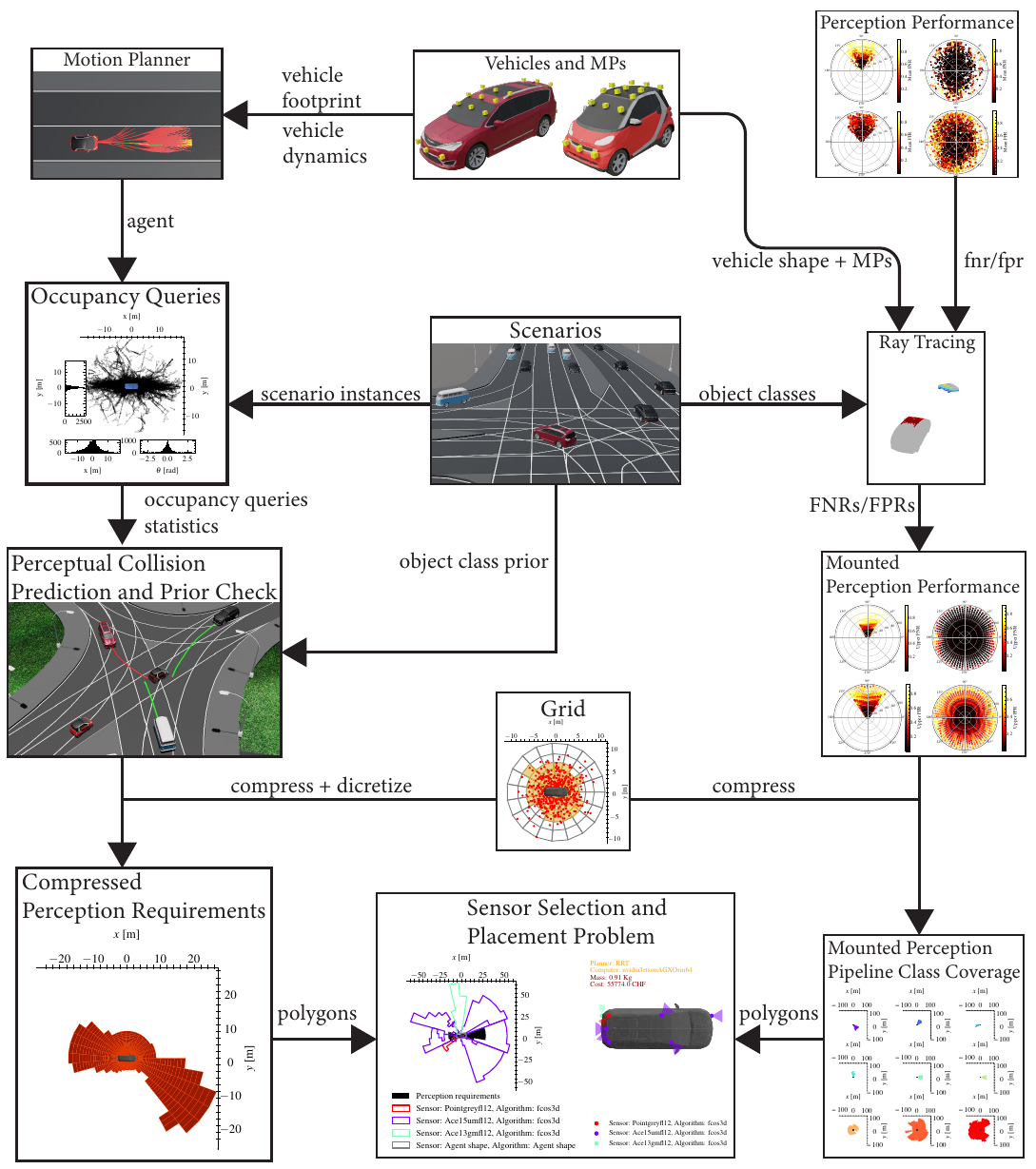}
    \caption{Overview of the sensor selection and placement process: starting with a catalog of robot bodies, sensor positions, orientations, perception pipelines, and motion planners, alongside with scenarios. 
    The workflow splits into agent activities (left) that transform task queries into perception requirements, and perception activities (right) that determine class configurations detectable by mounted perception pipelines. 
    The process concludes with the selection of optimal pipelines to minimize costs while satisfying perception requirements.}
    \label{fig:alg_ssp_flow}
\end{figure}
Employing a worst-case approach, this study assumes the absence of filters that account for historical detection data. This premise necessitates that for a perception pipeline to accurately respond to occupancy queries, its \gls{fnr} and \gls{fpr} must not exceed a predefined threshold $\epsilon$. Accordingly, this assumption ensures that the identification of class configurations from perception requirements is not influenced by temporal factors. 
\edited{Thus, all class configurations for which the both upper bounds from the $\ppp$ output intervals are smaller than the threshold~$\epsilon$ are considered covered or detectable by the perception pipeline~$\sensorp_j$.} This set of class configurations which can be seen by a perception pipeline depend on the mounting configuration on the robot body as well as the robot body shape itself.  The reason is that different mounting configurations will have different relative class configurations to the perception pipeline. Moreover, depending on the mounting configuration on the robot, the shape of the robot could block the sensor \gls{fov}. \edited{For instance, consider} a lidar sensor positioned on the roof of a vehicle. Due to its placement, some lidar beams are blocked by the vehicle's roof, preventing the lidar from measuring objects in close proximity to the vehicle. We call a perception pipeline with a mounting position on a robot body and some yaw and pitch mounting orientation as \emph{mounted perception pipeline}.

\begin{definition}[Mounted Perception Pipeline]
    Given a perception pipeline $\sensorp$, a robot body $\body$, a mounting position of a sensor $\mounts$ on the body, and the yaw and pitch angle of the sensor mounted on the robot $\mountso$, a mounted perception pipeline is a tuple containing the perception pipeline, the robot body, the mounting position and the mounting orientation: $\virtsensorp = \tup{\sensorp, \body, \mounts, \mountso}$.
\end{definition}

The following map is defined, which yields all the class configurations visible to a mounted perception pipeline, considering a specified threshold.

\begin{definition}[Mounted Perception Pipeline Class Coverage]
Consider a mounted perception pipeline $\virtsensorp$ characterized by its perception performance \edited{$\ppp$}, a target class instance $\classinstance{}$, an environment $\environment$ and a threshold $\epsilon$. The set of class configuration which can be detected by the mounted perception pipeline are defined as
\edited{\begin{equation}
\label{eq:mpcc}
\vspccoverage \colon \classinstanceset{} \times \mathbb{MPP} \times \environmentset \times \reals_{[0,1]} \to \powerset{\confspace{}^{\robot}},
\end{equation}}
where $\mathbb{MPP}$ is the \edited{finite} set of all mounted perception pipelines. %
\end{definition}

\edited{\begin{remark}\label{rem:simselfocclusion}
    The $\vspccoverage$ map is implemented by taking a set of class configurations from a grid, as detailed in~\cref{sec:solve:ssp} and illustrated in~\cref{fig:polar:grid}. For each mounted perception pipeline, robot body, class body and grid-based class configuration, we evaluate how many pixels or lidar points are projected onto the class body and how much is occluded by the robot body using a 3D ray-casting simulation, as described in~\cref{sec:approach}. Using this sensor-specific information, along with the class configuration, environment, and class appearance, we apply the $\ppp$ function to compute \gls{fnr} and \gls{fpr} intervals. By considering the upper bounds of these intervals and comparing them to a threshold~$\epsilon$, we determine whether the class configuration is detectable by the mounted perception pipeline. Detection is confirmed if both upper bounds are smaller than the threshold. The detected class configuration is then transformed from the sensor to the ego coordinate frame using the mounting configuration. This process is illustrated in~\cref{fig:alg_ssp_flow}.
\end{remark}}

Collections of mounted perception pipelines class coverage for some given $\kclass$ object class instances, $\menv$ environments and $\lvsp$ mounted perception pipelines are given as
\edited{\begin{equation*}
    \vspccoveragecat = \prod_{k=1}^{\kclass}\prod_{l=1}^{\lvsp}\prod_{m=1}^{\menv}\vspccoverage(\classinstance{k}, \virtsensorp_{l}, \environment_{m}, \epsilon).
\end{equation*}}
\begin{definition}[Sensor selection and placement problem]\label{pr:sensor_selection}
Consider a task $\task$, an agent $\agent$, a body $\body$ with mounting positions $\mountsset$, perception pipelines $\percpset$, 
mounting orientations $\mountsoset$ and a detection threshold $\epsilon$. The task involves $\kclass$ unique number of object class instances and $\menv$ number of environments. From the body, perception pipelines and mounting orientation, $\lvsp$ number of mounted perception pipelines $\virtsensorp$ can be generated. This leads to the task perception requirement $\sensreq(\agent,\task)$ and a set $\vspccoveragecat$ of collections of mounted perception pipelines class coverage with $\vspccoverage(\classinstance{k}, \virtsensorp_{l}, \environment_{m}, \epsilon) \subseteq \edited{\confspacen{k}^{\robot}}$, and $W$ cost functions $c_w:\virtsensorp_{l} \to \reals_{> 0}$. The problem is to identify $\mountedpercpset \subseteq \{\virtsensorp_i\}_{i\in \{1,\ldots, \lvsp\}}$ with the minimum total cost over all cost functions. The subset $\mountedpercpset$ must cover each element in $\sensreq(\agent,\task)$ with a matching $\vspccoverage(\classinstance{k}, \virtsensorp, \environment_{m}, \epsilon)$, specific to the same $\classinstance{k}$ and $\environment_m$ within $\sensreq(\agent,\task, \classinstance{k}, \environment_m) \subseteq \edited{\confspacen{k}^{\robot}}$. Furthermore, each $\virtsensorp \in \mountedpercpset$ must occupy a unique mounting position $\mounts$. The problem is outlined in~\cref{eq:ss_general}, employing a binary vector $x$ composed of elements $x_i \in \{0,1\}$, each denoting a decision variable. Here, $x_i = 1$ signifies the selection of the mounted perception pipeline $\virtsensorp_i$. An indicator function $\emtpyfunc$ is introduced to map a class configuration set~\edited{$\Theta^{\robot} \subseteq \confspacen{}^{\robot}$} to an empty set whenever the associated binary variable $x_i=0$: 
\begin{equation*}
    \emtpyfunc(\edited{\Theta^{\robot}}, x) = 
\begin{cases} 
   \edited{\Theta^{\robot}} & \text{if } x = 1 \\
   \emptyset & \text{if } x = 0 
\end{cases}.
\end{equation*}
Matrix $F$ indicates which mounted perception pipelines share the same mounting positions. In a given row of $F$, all entries set to 1 signify mounted perception pipelines with identical mounting positions.

\begin{equation}
\label{eq:ss_general}
\begin{aligned}
\min \quad & \sum_{i = 1}^{\lvsp} \sum_{j=1}^{W} w_jc_j(\virtsensorp_{i})\cdot x_i\\
\textrm{s.t.} \quad
&   \sensreq(\agent,\task, \classinstance{k}, \environment_m) \subseteq \\
& \bigcup_{l=1}^{\lvsp} \emtpyfunc(\vspccoverage(\classinstance{k}, \virtsensorp_l, \environment_{m}, \epsilon), x_l) \\ & \forall k \in \{1,\ldots,\kclass\}, m \in \{1,\ldots,\menv\},\\
& F \cdot x \leq \begin{bmatrix}
1 & \ldots & 1
\end{bmatrix}^{\text{T}},    \\
  & x_i \leq 1 \quad \forall i \in \{1,\ldots,\lvsp\},    \\
  & x_i \in \mathbb{N}_{0} \quad \forall i \in \{1,\ldots,\lvsp\},\\
  & \sum_{j=1}^W w_j=1, \ w_j\geq 0,\ j=1,\ldots, W.
\end{aligned}
\end{equation}
\end{definition}

The union of all class configurations detectable by the selected mounted perception pipelines, 
represented as $\mountedpercpset$, across all classes and environmental conditions is denoted as \emph{perception coverage}:
\edited{\begin{equation*}
    \sensreq(\agent,\task) \subseteq \prod_{k=1}^{\kclass}\prod_{\virtsensorp \in \mountedpercpset}\prod_{m=1}^{\menv}\vspccoverage(\classinstance{k}, \virtsensorp, \environment_{m}, \epsilon).
\end{equation*}}

The \textbf{Coverage} \gls{abk:mdpi}, illustrated in~\cref{fig:mdpi_coverage}, focuses on meeting the robot's \F{perception requirements} as a functionality by ensuring sufficient \R{perception coverage} as a resource, which includes the ability to detect necessary class configurations to accomplish the task. An increase in \F{perception requirements} directly necessitates an enhancement in \R{perception coverage}, which is demonstrated in~\cref{lem:pr_monotone}. 

\begin{lemma}
    \label{lem:pr_monotone}
    The \F{perception requirements} $\sensreq$ are monotone with \R{perception coverage}, 
    as shown in~\cref{fig:mdpi_coverage}.
    \end{lemma}
    \begin{proof}
    Consider the first constraint in~\cref{eq:ss_general}, representing the sensor selection and placement optimization problem. Clearly, if one increases the $\sensreq$ set, one needs to increase the union of selected perception pipeline class coverage $\vspccoverage$, representing the perception coverage.
    \end{proof}

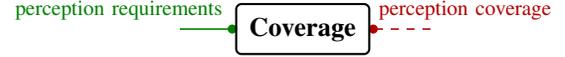
\begin{figure}[tb]
    \centering
    \scalebox{1.0}{\begin{tikzpicture}[DP]
    \node[dp={1}{1}] (ch) {Coverage};
    \draw[runconn, runame={perception coverage},relres=above, posres=1.6] (ch_res1){};
    \draw[funconn, funame={perception requirements},relfun=above, posfun=2.2] (ch_fun1){};
\end{tikzpicture}}
    \caption{The Coverage \gls{abk:mdpi} which provides \F{perception requirements} $\sensreq$ and requires \R{perception coverage} as a set of $\vspccoverage$.}
    \label{fig:mdpi_coverage}
\end{figure}

The \textbf{Mounted Perception Pipelines} \gls{abk:mdpi} in~\cref{fig:mdpi_mpp} implements the selection and positioning of perception pipelines on the robot to cover all perception requirements, thus ensuring \F{perception coverage}, considering all \F{class appearances} $\appearance$ within the task, 
and accommodating the \F{robot's shape} $\threedshape$. 
\edited{This \gls{abk:mdpi} requires a set of \R{mounting configurations} in $\sethree$ and a set of \R{perception performance} quantified by the upper limits of the $\ppp$ function.} The \R{perception performance} considers the opposite order of $\ppp$ upper limits, where a pipeline $\sensorp_a$ dominates $\sensorp_b$ if it has lower upper bounds for \glspl{fnr} and \glspl{fpr} across all class configurations~\edited{$\config{i}^{\robot}$}, class appearances $\appearance_{i}$ and environments $\environment$. Adding a class configuration to the \F{perception coverage} or new \F{class appearances} may necessitate a change to a more capable perception pipeline with improved \R{perception performance} to ensure coverage under the defined threshold $\epsilon$. Similarly, enhancing \F{perception coverage} with new class configurations or \F{class appearances} might necessitate additional \R{mounting configurations}. A larger \F{robot's shape} may introduce self-occlusion, impacting the \gls{fov} and necessitating additional sensor placements for coverage.

\edited{\begin{remark}
    A larger \R{robot's footprint} can reduce colliding class trajectories, as shown in~\cref{fig:mdpi_prior}, by covering them. However, a bigger \F{robot's shape} can increase sensor self-occlusion and create more colliding class configurations, resulting in additional colliding class trajectories. Therefore, balancing the \R{robot's footprint} and the \F{robot's shape} becomes crucial in the co-design process.

\end{remark}}

\begin{figure}[tb]
    \centering
    \scalebox{0.9}{\begin{tikzpicture}[DP, dp port sep=3pt]
    \node[dp={3}{2}] (ch) {\begin{tabular}{c}Mounted Perception \\ Pipelines\end{tabular}};
    \draw[runconn, runame={mounting configurations},relres=above, posres=2.1] (ch_res1){};
    \draw[runconn, runame={perception performance},relres=above, posres=2.0] (ch_res2){};
    % \draw[runconn, runame={robot's footprint},relres=above, posres=1.4] (ch_res3){};
    \draw[funconn, funame={perception coverage},relfun=above, posfun=2.0] (ch_fun1){};
    \draw[funconn, funame={class appearances},relfun=above, posfun=1.8] (ch_fun2){};
    \draw[funconn, funame={robot's shape},relfun=above, posfun=1.4] (ch_fun3){};
\end{tikzpicture}}
    \caption{The Mounted Perception Pipelines \gls{abk:mdpi} which provides the \F{perception coverage}, the set of all \F{class appearances} $\appearance$ in the task and the \F{robot's shape} $\threedshape$ as functionalities.
    \edited{The required resources are the set of \R{mounting configurations} in $\sethree$ and the \R{perception performance}}.}
    \label{fig:mdpi_mpp}
\end{figure}
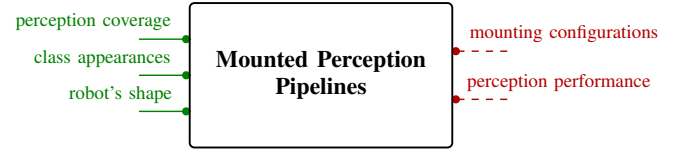

In~\cref{fig:mdpi_pp}, the \textbf{Perception Pipelines} \gls{abk:mdpi} outlines the implementation of available perception pipelines, encompassing both sensors and perception algorithms. It delivers \F{perception performance} as its functionality, demanding \R{cost} in \unit[]{CHF}, \R{mass} in \unit[]{kg}, \R{power} in \unit[]{W}, and \R{compute} as resources. The monotonic relationship indicates that enhancing \F{perception performance}, aiming for lower \gls{fnr} and \gls{fpr}, requires the employment of pricier, high-resolution sensors which generally consume more power and are heavier. 
Alternatively, it might involve leveraging more complex perception algorithms that demand substantial computational power.

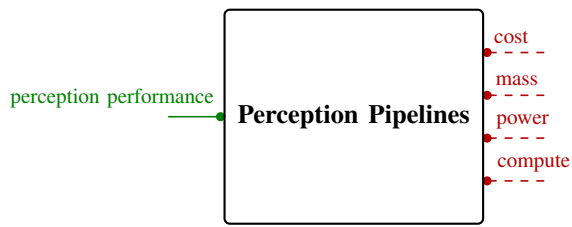
\begin{figure}[tb]
    \centering
    \scalebox{1.0}{\begin{tikzpicture}[DP, dp port sep=3pt]
    \node[dp={1}{4}] (ch) {Perception Pipelines};
    \draw[funconn, funame={perception performance},relfun=above, posfun=2.1] (ch_fun1){};
    \draw[runconn, runame={cost},relres=above, posres=0.4] (ch_res1){};
    \draw[runconn, runame={mass},relres=above, posres=0.5] (ch_res2){};
    \draw[runconn, runame={power},relres=above, posres=0.6] (ch_res3){};
    \draw[runconn, runame={compute},relres=above, posres=0.8] (ch_res4){};
\end{tikzpicture}}
    \caption{The Perception Pipelines \gls{abk:mdpi} which provides the \F{perception performance} and requires \R{cost} in \unit[]{CHF}, \R{mass} in \unit[]{kg}, \R{power} in \unit[]{W} and \R{compute}.
    }
    \label{fig:mdpi_pp}
\end{figure}

\subsection{Solving the Sensor Selection and Placement Set Cover Problem}\label{sec:solve:ssp}
The nature of \cref{pr:sensor_selection} closely resembles the weighted set cover problem~\cite{vazirani2001approximation}, since it also tries to cover a given set by a collection of subsets while minimizing a cost function. \edited{The weighted set cover problem is NP-complete. There exist approximations, such as greedy algorithms or \gls{ilp}. In addressing \cref{pr:sensor_selection}, we choose the \gls{ilp} relaxation of the set cover problem~\cite{vazirani2001approximation}. To formulate the \cref{pr:sensor_selection} as a weighted set cover problem, we need to make certain approximations.} This is necessary because both the task perception requirements, denoted as $\sensreq(\agent,\task)$, and the coverage of mounted perception pipelines for different classes, denoted as $\vspccoveragecat$, are infinite sets. In the next paragraphs we show how we formulate the sensor selection and placement problem as a weighted set cover problem.

\noindent \paragraph*{Class configurations in $\setwo$}
The first approximation involves constraining all class configurations in both $\sensreq(\agent,\task)$ and $\vspccoveragecat$ to exist within $\setwo$. Specifically, each class configuration is now defined as a tuple consisting of position in Cartesian coordinates and the relative orientation $\theta$ with respect to the robot frame, denoted as $\config{i}^{\robot} = \tup{\xpos{}, \ypos{}, \theta}$. As these class configurations are now geometric in nature and reside in $\setwo$, the problem closely resembles the \emph{polygon covering} problem~\cite{culberson1994covering}, which is a specific case of the set cover problem. In the weighted polygon covering problem, the objective is to cover a target polygon using a set of provided polygons, each associated with a specific cost. This problem permits overlapping among the polygons. However, the class configurations are represented in three-dimensional space ($\setwo$) and are essentially volumes rather than polygons. Therefore, we need a method to reduce the dimensionality of these configurations.

\noindent \paragraph*{From class configurations to polygons}
Given the orientation constraint $-\pi \leq \theta \leq \pi$, the class configurations are sorted into $\theta$-intervals, such as  $\{[-\pi, -\pi+\Delta\theta), [-\pi+\Delta\theta, -\pi+2\cdot\Delta\theta) \ldots [\pi-\Delta\theta, \pi)\}$. The subsequent step involves transforming the position coordinates of the class configuration within each $\theta$-interval into a set of polygons. Here, polygons represent surfaces in $\reals^2$ with location considerations. This set of polygons is termed a \emph{multi-polygon}, where the polygons in the set are not necessarily contiguous. As a result, a set of multi-polygons is generated, with each element corresponding to a distinct $\theta$-interval. Although various methods can be devised for this transformation, we stick to a worst-case analysis approach for consistency. The detailed description of this process is beyond the scope of this paper. The resulting set of multi-polygons is denoted as \emph{compressed class configurations}.

\begin{definition}[Compress]
$\compress$ is a mapping that generates a set of multi-polygons $\multipolygon$ from a set of class configurations $\edited{\confspacen{i}^{\robot}}$
and $T$ number of class configurations $\theta$-intervals.
\begin{equation*}
\begin{aligned}
    \compress &\colon \powerset{\edited{\confspacen{i}^{\robot}}} \to \prod_{j \in \{1, \ldots, T\}} \powerset{\reals^2}, \\
\end{aligned}
\end{equation*}
where $T \in \naturaln^{+}$.
\end{definition}
Applying $\compress$ to $\sensreq(\agent,\task)$ and $\vspccoveragecat$ results in $\overline{\sensreq}(\agent,\task)$ and $\overline{\vspccoveragecat}$, where all sets of class configurations are now expressed as compressed class configurations. Specifically, when $\compress$ is applied for each environment in $\sensreq$, nested sets are obtained for each environment, object class, and $\theta$-interval.

\noindent \paragraph*{Discretization}
To formulate the weighted set cover problem with the obtained polygons, we need to discretize $\overline{\sensreq}(\agent,\task)$.
A straightforward approach is to create a grid with cells, which can be made uniform as shown in~\cref{fig:unit:grid}, e.g., 1 by 1 meters in size. We use a polar grid with logarithmic scaling for radial distance as illustrated in~\cref{fig:polar:grid}, providing higher granularity for smaller distances and aligning more with sensor perception dynamics which scan the environment radially. 
\begin{figure}[tb] 
    \centering
    \begin{subfigure}[b]{0.24\textwidth}
        \centering
        \includegraphics[width=\textwidth]{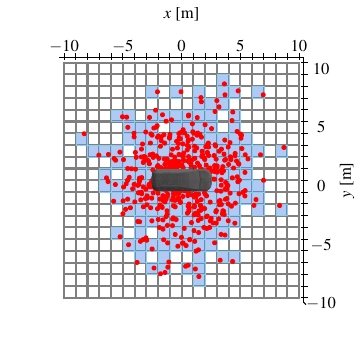}
        \caption{Unit grid cells.}
        \label{fig:unit:grid}
    \end{subfigure}
    \begin{subfigure}[b]{0.24\textwidth}
        \centering
        \includegraphics[width=\textwidth]{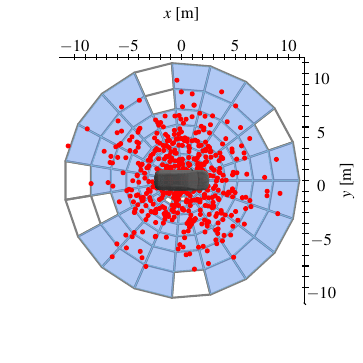}
        \caption{Polar grid with logarithmic scale.}
        \label{fig:polar:grid}
    \end{subfigure}
    \caption{The left image shows a uniform grid, while the right reports a polar grid with logarithmically scaled radial distances. Red dots, representing Gaussian synthetic class configurations, intersect with blue shaded cells.}
    \label{fig:polygon:cell:grid}
\end{figure}
This means for each multi-polygon in $\overline{\sensreq}(\agent,\task)$, which corresponds to a certain environment, a certain class and a certain $\theta$-interval, we obtain a discretized multi-polygon which is again a multi-polygon. These discretized perception requirements are represented as $\widehat{\sensreq}(\agent,\task)$. An example of discretized perception requirements of an \gls{av} driving in an urban environment, for a car class object for two different orientations is shown in~\cref{fig:comp:disc:sensreq}. 

\begin{figure}[t] 
    \centering
    \begin{subfigure}[b]{0.24\textwidth}
        \centering
        \includegraphics[width=\textwidth]{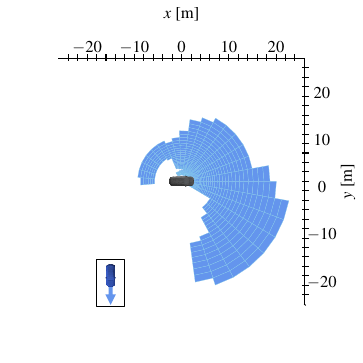}
        \caption{Car class for $\theta$-interval $[-100^{\circ}, -90^{\circ})$.}
    \end{subfigure}
    \begin{subfigure}[b]{0.24\textwidth}
        \centering
        \includegraphics[width=\textwidth]{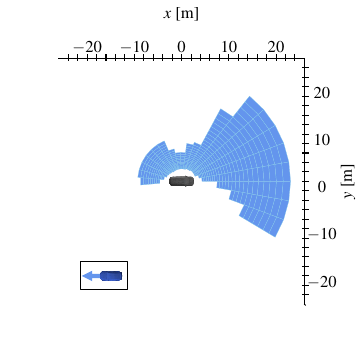}
        \caption{Car class for $\theta$-interval $[-180^{\circ}, -170^{\circ})$.}
    \end{subfigure}
    \caption{Example of discretized and compressed perception requirements of a car class (blue) for different orientations 
    relative to the ego vehicle (grey car).}
    \label{fig:comp:disc:sensreq}
\end{figure}

\begin{figure}[h]
    \centering
    \begin{subfigure}[h]{0.15\textwidth}
        \centering
        \includegraphics[width=\textwidth]{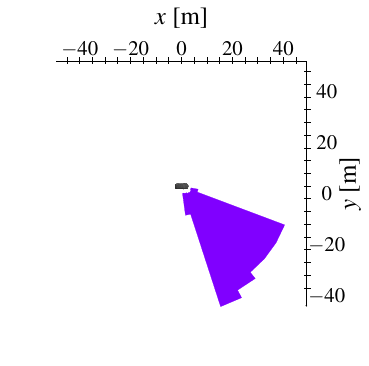}
    \end{subfigure}
    \begin{subfigure}[h]{0.15\textwidth}
        \centering
        \includegraphics[width=\textwidth]{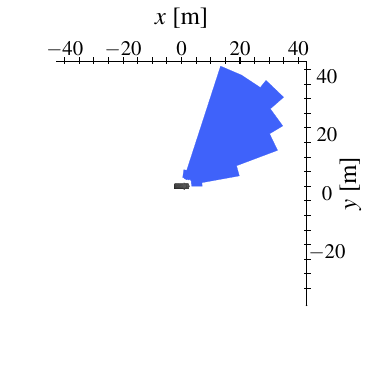}
    \end{subfigure}
   \begin{subfigure}[h]{0.15\textwidth}
        \centering
        \includegraphics[width=\textwidth]{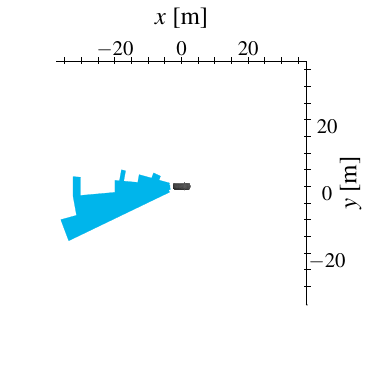}
    \end{subfigure}
    \begin{subfigure}[h]{0.15\textwidth}
        \centering
        \includegraphics[width=\textwidth]{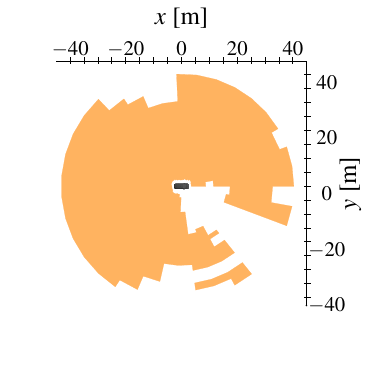}
    \end{subfigure}
    \begin{subfigure}[h]{0.15\textwidth}
        \centering
        \includegraphics[width=\textwidth]{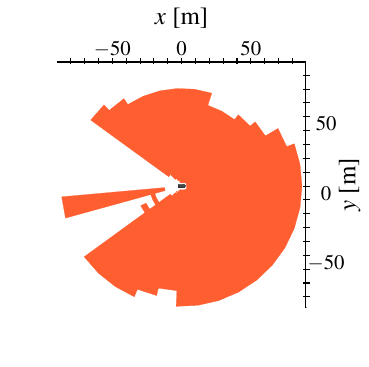}
    \end{subfigure}
    \begin{subfigure}[h]{0.15\textwidth}
        \centering
        \includegraphics[width=\textwidth]{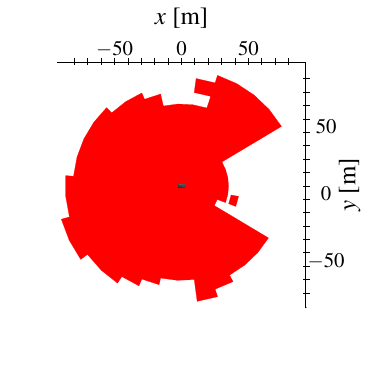}
    \end{subfigure}
    \caption{Examples of compressed mounted perception pipeline class coverage $\vspccoverage$ corresponding to the setting in \cref{fig:comp:disc:sensreq} with $\theta$-interval $[-100^{\circ}, -90^{\circ})$. Each plot corresponds to a unique mounted perception pipeline.}
    \label{fig:com_sens_perf}
\end{figure}
In \cref{fig:com_sens_perf}, examples of compressed $\vspccoverage$ are depicted for the class and robot specified in \cref{fig:comp:disc:sensreq} with $\theta$-interval $[-100^{\circ}, -90^{\circ})$. These polygons aim to cover the upper polygon shown in \cref{fig:comp:disc:sensreq}. Each polygon is associated with certain costs, and the objective is to minimize the total cost.

\edited{With all the components in place, we can formulate the problem in~\cref{pr:sensor_selection} as an \gls{ilp}.} Once again, we use the binary vector $x$, where each element $x_i \in \{0,1\}$ and represents a decision variable. The variable $x_i = 1$ if and only if the mounted perception pipeline $\virtsensorp_i$ is chosen.

\noindent \paragraph*{Cost Functions}
\edited{We extend the \gls{ilp} to a multi-weighted problem formulation by incorporating $W$ cost functions denoted as $c$.} Each cost function $c_j$ associates a mounted perception pipeline $\virtsensorp$ with normalized costs, where $0 \leq c_j(\virtsensorp) \leq 1$. These costs may represent various factors such as the price, mass, or power consumption of the sensor. Additionally, each cost $c_j$ is scaled by a cost weight $w_j$, ensuring that the sum of all weights equals one, i.e., $\sum_{j=1}^W w_j=1$. The cost function weights are generated by the Halton sequence~\cite{halton1960efficiency, owen2017randomized}, a generalized form of the one-dimensional Van der Corput sequence~\cite{van1935verteilungsfunktionen, lavalle2006},
where we only take sampled points which sum up to one.
This process involves generating a series of weights with low discrepancy and addressing the optimization problem for each weight set. Through this incremental search, we explore the Pareto front of the multi-objective optimization problem with a linear weighted sum~\cite{stanimirovic2011linear, marler2010weighted}.

\noindent \paragraph*{Constraints}
The initial constraint within the \gls{ilp} ensures the coverage of each element in $\widehat{\sensreq}(\agent,\task)$. This implies that for every polygon within $\widehat{\sensreq}(\agent,\task)$, we must ascertain which $\virtsensorp$ is providing coverage. To achieve this, we extract the corresponding multi-polygon from $\virtsensorp$ that shares the same object class, environment, and $\theta$-interval. By ``cover'' we mean that a multi-polygon $\multipolygon_i$ covers another polygon $\multipolygon_j$ if $\multipolygon_j \subseteq \multipolygon_i$. Consequently, a binary matrix $A$ is populated, possessing dimensions $N \times \lvsp$, where $N$ represents the number of polygons in $\widehat{\sensreq}(\agent,\task)$ and $\lvsp$ denotes the number of mounted perception pipelines. The entry in the $n$-th row and $l$-th column of matrix $A$ is denoted as $a_{nl}$, with $a_{nl} = 1$ indicating that polygon $n$ is covered by $\virtsensorp_l$, and $a_{nl} = 0$ otherwise. Subsequently, another binary matrix, denoted as $F$, is constructed with dimensions $D \times \lvsp$, where $D$ corresponds to the number of mounting positions. Matrix $F$ indicates which mounted perception pipelines share the same mounting positions. In a given row of $F$, all entries set to 1 signify mounted perception pipelines with identical mounting positions.
Finally we can find the mounted perception pipelines which cover $\widehat{\sensreq}(\agent,\task)$, while minimizing certain cost $c_j$ by solving the \gls{ilp} in \cref{eq:ilpmultiobj}.

\begin{equation}
\label{eq:ilpmultiobj}
\begin{aligned}
\min \quad & \sum_{i = 1}^{L} \sum_{j=1}^{W} w_jc_j(\virtsensorp_{i})\cdot x_i\\
\textrm{s.t.} \quad 
& A \cdot x \geq \begin{bmatrix}
1 & \ldots & 1
\end{bmatrix}^{\text{T}} \quad,    \\
& F \cdot x \leq \begin{bmatrix}
1 & \ldots & 1
\end{bmatrix}^{\text{T}} \quad,    \\
  & x_i \leq 1 \quad \forall i \in \{1,\ldots,L\},    \\
  & x_i \in \mathbb{N}_{0} \quad \forall i \in \{1,\ldots,L\},\\
  & \sum_{j=1}^W w_j=1, \ w_j\geq 0,\ j=1,\ldots, W.
\end{aligned}
\end{equation}

\section{Design of experiments and results}
\label{sec:results-gen}
In this section, we report a case study on designing an \gls{av} for an urban driving task. We outline the experimental design in \cref{sec:experiments}, present the results in \cref{sec:results}, and conclude with a discussion of the findings in \cref{sec:discussion}.

\begin{table}[tb]
\centering
\caption{Variables, options and sources for the AV co-design problem.}
\label{tab:catalogue}
\begin{tabular}{ m{2cm} m{4cm} c  } 
 $\mathbf{Variable}$ & $\mathbf{Option}$ & $\mathbf{Source}$ \\ 
 \hline
 Vehicle bodies & Smart Fortwo, Chrysler Pacifica, Mercedes-Benz C63 & \cite{cars} \\ 
 \hline
 Lidars & Velodyne: Alpha Prime, HDL 64, HDL 32; 
 OS2: 128, 64  & \cite{Velodyne, Ouster} \\ 
 \hline
 Cameras & Basler: acA1600-gm, acA1500-um, acA7-gm; FLIR: Point Grey & \cite{basler, flir}  \\ 
 \hline
 Object Detection Models & FCOS3D, Pointpillars  & \cite{mmdet3d2020, wang2021fcos3d, Lang2020PointPillars:Clouds} \\ 
 \hline
 Mounting Orientation Yaw & $-135.0^{\circ}$, $-90.0^{\circ}$, $-45.0^{\circ}$, $0.0^{\circ}$, $45.0^{\circ}$, $90.0^{\circ}$, $135.0^{\circ}$, $180.0^{\circ}$,   & [-]\\ 
 \hline
 Mounting Orientation Pitch & $0^{\circ}$  & [-] \\ 
 \hline
 Motion Planner & Lattice panner with A*, RRT, RRT*  & \cite{sucan2012the-open-motion-planning-library, althoff2017commonroad} \\ 
 \hline
  Computer & Jetson Nano, Orin Nano, Xavier NX, Orin NX, AGX Orin 64GB, AGX Orin 32GB, AGX Xavier 32GB & \cite{nvidia} \\ 
 \hline
\end{tabular}
\end{table}

\subsection{Design of experiments}\label{sec:experiments}
\noindent \paragraph*{Catalogs}
The components available for design are reported in the catalog in \cref{tab:catalogue}. 
The 3D meshes of the car bodies are sourced from TurboSquid~\cite{turbosquid}. 
Real sensor measurements from the nuScenes open-source dataset~\cite{caesar2020nuscenes}, along with state-of-the-art 3D object detection algorithms from the MMDetection3D library~\cite{mmdet3d2020}, are used to determine the FNRs and the FPRs for different object classes. 
The mounting position options are visualized in~\cref{fig:mp}. 
We utilize motion planners from the OMPL~\cite{sucan2012the-open-motion-planning-library} and CommonRoad~\cite{althoff2017commonroad} libraries, including RRT, RRT*, and a lattice planner enhanced with motion primitives and an A* search algorithm. Specifically, for the RRT* planner from the OMPL library, which is classified as a ``geometric'' planner, we employ Dubins paths~\cite{dubins1957curves, paden2016motionplanningsurvey} to connect sampled configurations considering the system's geometric and kinematic constraints. This approach enables the computation of paths that can be tracked by low level controllers as depicted in \cref{fig:rrtstar}. In contrast, the RRT planner corresponds to ``control-based'' implementations in the OMPL library which directly computes trajectories and control inputs, tailored for systems subject to differential constraints and incorporating a steering function. The three different motion planners operate with \unit[1]{s} and \unit[2]{s} planning horizons, which define the time into the future for which a planner calculates its trajectory.

\begin{remark}
    We acknowledge that the catalog may not represent the latest advances in motion planning and perception. The designer is free to create their own catalog.
\end{remark}

\begin{figure}[tb]
    \centering
    \begin{subfigure}[b]{0.24\textwidth}
        \centering
\includegraphics[width=\textwidth]{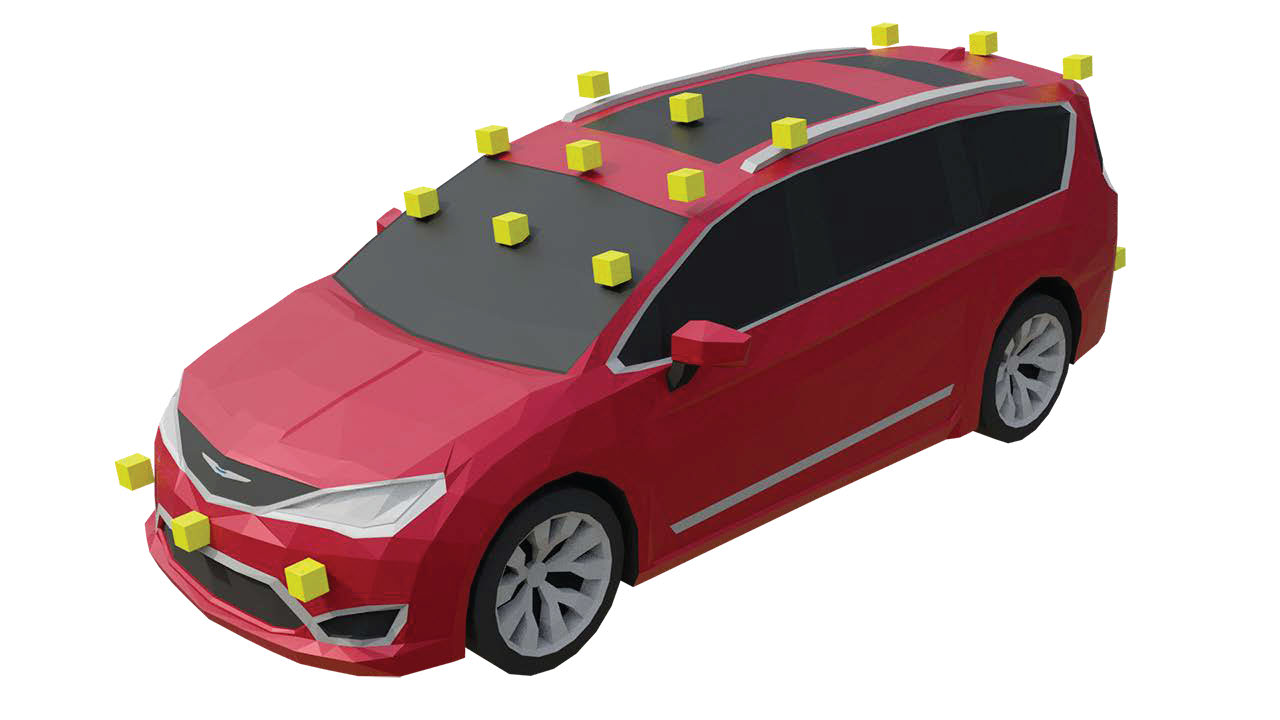}
    \end{subfigure}
    \begin{subfigure}[b]{0.24\textwidth}
        \centering
        \includegraphics[ width=\textwidth]{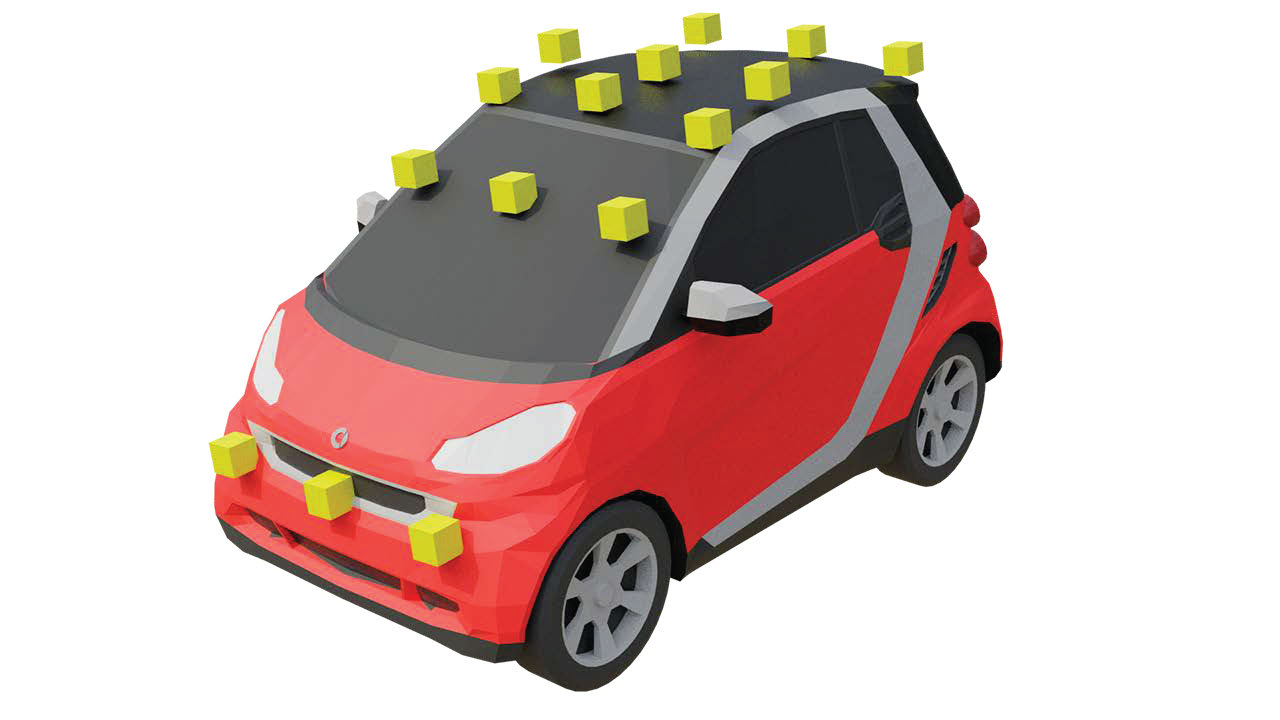}
    \end{subfigure}
    \caption{Exemplary mounting positions for two different vehicles.}
    \label{fig:mp}
\end{figure}

\noindent \paragraph*{Task}
The urban driving task contains 205 driving scenarios from the CommonRoad library~\cite{Maierhofer2021}, featuring five different vehicle classes. 
Each vehicle's configuration is defined \edited{by~$\config{}^{\workspace} \in \setwo$} and the vehicle dynamics are based on the bicycle model. The car prior configuration,~$\prior{\mathrm{car}}$, accounts for all possible car positions on the road, aligning with the driving direction. The objective is for the autonomous vehicle to reach a designated area. We analyze scenarios with two nominal speeds: $\unitfrac[30]{km}{h}$ and~$\unitfrac[50]{km}{h}$, under dry and rainy weather conditions during daylight and night.
We performed experiments with fewer scenarios to examine how task complexity affects the AV design. Additionally, the experiments varied the task prior assuming no cars can approach the AV from left and rear. 

\edited{The simulations, benchmarking, and optimizations described in~\cref{rem:percreq,,rem:bench,,rem:simselfocclusion} are highly parallelizable, enabling continuous computation of solutions. We report average computing times with \unit[90]{\%} confidence intervals from tasks run on the ETH Zürich Euler cluster, noting that further optimization of the implentation is possible with different resource allocations. Simulations for gathering occupancy queries took~\unit[$10 \pm 8$]{min} on a \unit[3.0]{GHz} CPU with \unit[1.4]{GB} of RAM. Generating perception requirements as described in~\cref{rem:percreq} averaged~\unit[$42_{-36}^{+60}$]{min} per simulation on a \unit[3.0]{GHz} CPU with \unit[6.5]{GB} of RAM. Simulations and inference in~\cref{rem:simselfocclusion} for the~$\vspccoverage$ map implementation took~\unit[$10 \pm 9$]{min}, considering a perception pipeline, mounting position, robot body, object class, environment, grid-based class configurations (\unit[$\sim 5000$]{}), and one $\theta$-interval, using a \unit[3.0]{GHz} CPU with \unit[5.5]{GB} of RAM. The sensor selection and placement \gls{ilp} in~\cref{eq:ilpmultiobj} was solved with the Gurobi solver~\cite{gurobi}, comprising 667 decision variables and around \unit[250,000]{} constraints. We optimized for four costs: sensor price, mass, power consumption, and object detection algorithm FLOPS. Solving took \unit[75]{s} on a \unit[2.3]{GHz} Intel Core i7 with \unit[16]{GB} of RAM. About 3000 weight sets were sampled to populate the Pareto front, yielding 3000 individual optimization problems per robot body, agent, and task combination. Outer optimization using the ZüperMind solver~\cite{censi2024} required \unit[25]{min} on the same hardware.}

\edited{
\begin{remark}
    Given the scope of this paper and its focus on overall AV design, we emphasize the results of the outer optimization and do not specifically compare the inner optimization results, despite these being a byproduct of the outer optimization. Comparison for the outer optimization is challenging due to the absence of automated design processes for complete AVs in the state-of-the-art. For the inner optimization, which includes perception pipeline selection and sensor placement, comparable methods exist, as noted in~\cref{sec:litrev}. However, none are both task-specific, capable of integrating with other design components, and bridging perception with decision-making as we achieve through the perception requirements. Additionally, learning-based techniques~\cite{dey2023machine} are difficult to compare due to variations in datasets and learning pipeline configurations.
\end{remark}
}

\subsection{Results}\label{sec:results}
We solve the presented co-design problem by fixing selected scenarios, and showing the corresponding Pareto fronts of minimal resources, as illustrated in~\cref{fig:codesign:power:scenarios:pareto:imp,,fig:codesign:mass:scenarios:pareto:imp,,fig:codesign:computation:scenarios:pareto:imp}. The figures show that more resources are required for more complex tasks. Each task's complexity is represented by the number of scenarios, with simpler tasks as subsets of more complex ones. The upper figures compare price (CHF) on the $x$-axis against power consumption (W) in~\cref{fig:codesign:power:scenarios:pareto:imp}, mass (kg) in~\cref{fig:codesign:mass:scenarios:pareto:imp}, and computation (GFLOPS) in~\cref{fig:codesign:computation:scenarios:pareto:imp} on the $y$-axis. 
Red dots indicate optimal solutions within each task, with the surrounding red area highlighting the feasible resource range (i.e., the upper sets of resources). 
Annotations with capital letters point to the implementations, detailed in the lower sub-figures.
We show the top view of the selected vehicle, with cameras marked with dots and lidars with squares to illustrate their mounting positions. Camera orientations are further highlighted by small triangles indicating the initial \gls{fov} and yaw direction, providing an indication of their potential coverage area. Each perception pipeline is color-coded. In addition, the graphics show the selected motion planner and computing unit.
\begin{figure}[tb] 
    \centering
    \begin{subfigure}[b]{0.33\textwidth}
        \centering
        \includegraphics[width=\textwidth]{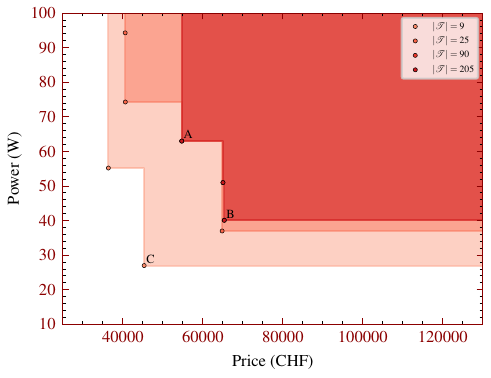}
    \end{subfigure}
    \\
    \begin{subfigure}[b]{0.15\textwidth}
        \centering
        \includegraphics[width=\textwidth]{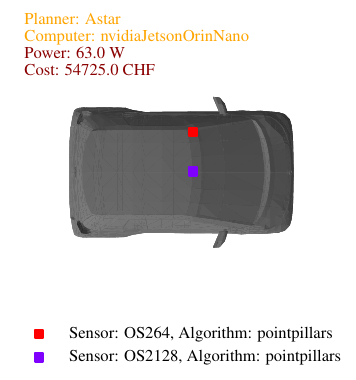}
        \captionsetup{labelformat=empty}
        \caption{A}
    \end{subfigure}
    \begin{subfigure}[b]{0.15\textwidth}
        \centering
        \includegraphics[width=\textwidth]{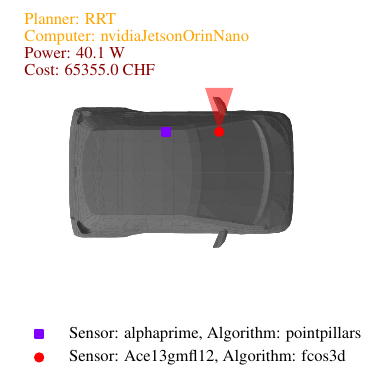}
        \captionsetup{labelformat=empty}
        \caption{B}
    \end{subfigure}
    \begin{subfigure}[b]{0.15\textwidth}
        \centering
        \includegraphics[width=\textwidth]{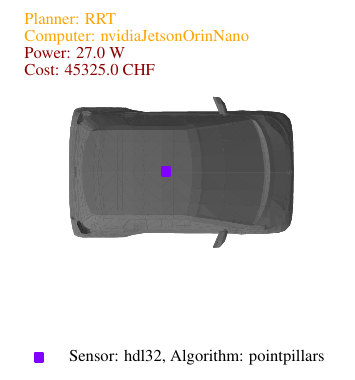}
        \captionsetup{labelformat=empty}
        \caption{C}
    \end{subfigure}
    \caption{Pareto front of price and power across tasks, 
    where tasks with more scenarios demand more resources and encompass those with fewer scenarios. 
    Implementations for point A, B, and C are visualized vertically. 
    B and C indicate the least power usage for the most and least complex tasks, respectively, while A shows the minimum price for the most complex task.} \label{fig:codesign:power:scenarios:pareto:imp}
\end{figure}
\begin{figure}[h!] 
    \centering
    \begin{subfigure}[b]{0.33\textwidth}
        \centering
        \includegraphics[width=\textwidth]{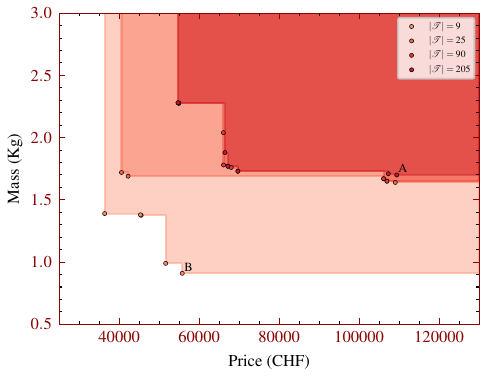}
    \end{subfigure}
    \\
    \begin{subfigure}[b]{0.15\textwidth}
        \centering
        \includegraphics[width=\textwidth]{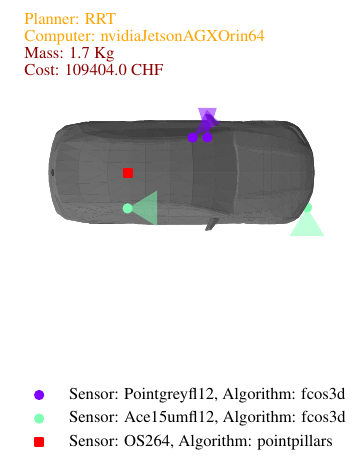}
        \captionsetup{labelformat=empty}
        \caption{A}
    \end{subfigure}
    \begin{subfigure}[b]{0.15\textwidth}
        \centering
        \includegraphics[width=\textwidth]{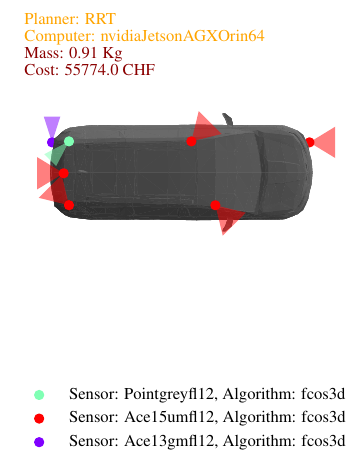}
        \captionsetup{labelformat=empty}
        \caption{B}
    \end{subfigure}
    \caption{Pareto front of price and mass across tasks, 
    where tasks with more scenarios demand more resources and encompass those with fewer scenarios. 
    Implementations for points A and B are visualized vertically. 
    A and B indicate the lowest mass for the most and least complex tasks, respectively.}
    \label{fig:codesign:mass:scenarios:pareto:imp}
\end{figure}

\begin{figure}[h!] 
    \centering
    \begin{subfigure}[b]{0.33\textwidth}
        \centering
        \includegraphics[width=\textwidth]{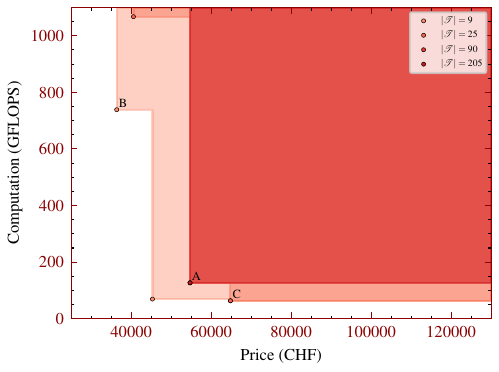}
    \end{subfigure}
    \\
    \begin{subfigure}[b]{0.15\textwidth}
        \centering
        \includegraphics[width=\textwidth]{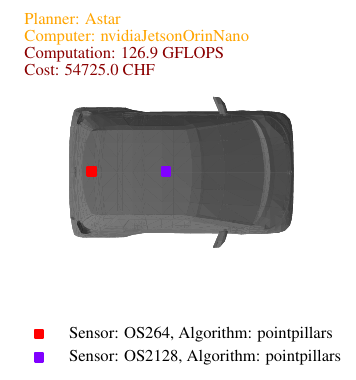}
        \captionsetup{labelformat=empty}
        \caption{A}
    \end{subfigure}
    \begin{subfigure}[b]{0.15\textwidth}
        \centering
        \includegraphics[width=\textwidth]{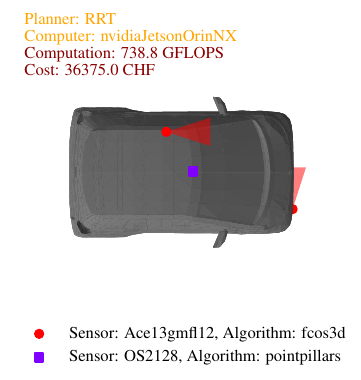}
        \captionsetup{labelformat=empty}
        \caption{B}
    \end{subfigure}
    \begin{subfigure}[b]{0.15\textwidth}
        \centering
        \includegraphics[width=\textwidth]{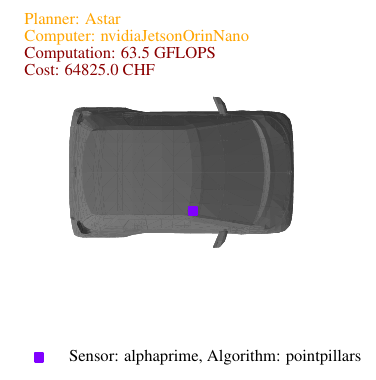}
        \captionsetup{labelformat=empty}
        \caption{C}
    \end{subfigure}
    \caption{Pareto front of price and computation across tasks, where more scenarios demand more resources . 
    Implementations plots for point A, B, and C are visualized vertically. 
    A and C indicate the least computation usage for the most and least complex tasks, respectively, while B shows the minimum price for the least complex task.}
\label{fig:codesign:computation:scenarios:pareto:imp}
\end{figure}
The impact of more resource requirements for the AV design by increasing the nominal speed from $\unitfrac[30]{km}{h}$ to $\unitfrac[50]{km}{h}$ within identical task scenarios is visualized in~\cref{fig:codesign:power:velocity:pareto:imp,,fig:codesign:mass:velocity:pareto:imp,,fig:codesign:computation:velocity:pareto:imp}, where we again show the Pareto fronts as well as the corresponding implementations for the different resources.
\begin{figure}[h!] 
    \centering
    \begin{subfigure}[b]{0.33\textwidth}
        \centering
        \includegraphics[width=\textwidth]{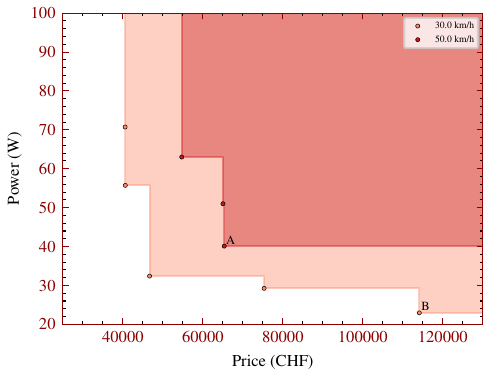}
    \end{subfigure}
    \\
    \begin{subfigure}[b]{0.15\textwidth}
        \centering
        \includegraphics[width=\textwidth]{figures/co_design/power/implementation_task_urban_car_cr_50.0_dry_day_power_40.1_65355.0_289.pdf}
        \captionsetup{labelformat=empty}
        \caption{A}
    \end{subfigure}
    \begin{subfigure}[b]{0.15\textwidth}
        \centering
        \includegraphics[width=\textwidth]{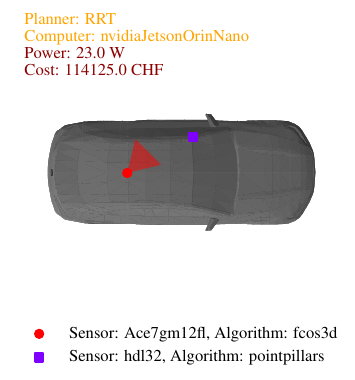}
        \captionsetup{labelformat=empty}
        \caption{B}
    \end{subfigure}
    \caption{Pareto front of price and power usage across task velocities, 
    where higher nominal velocities for the same set of scenarios require more resources. 
    Implementations for points A and B are visualized vertically. 
    A and B indicate lowest power usage for $\unitfrac[50]{km}{h}$ and $\unitfrac[30]{km}{h}$ nominal velocities, respectively.}
    \label{fig:codesign:power:velocity:pareto:imp}
\end{figure}
\begin{figure}[h!] 
    \centering
    \begin{subfigure}[b]{0.33\textwidth}
        \centering
        \includegraphics[width=\textwidth]{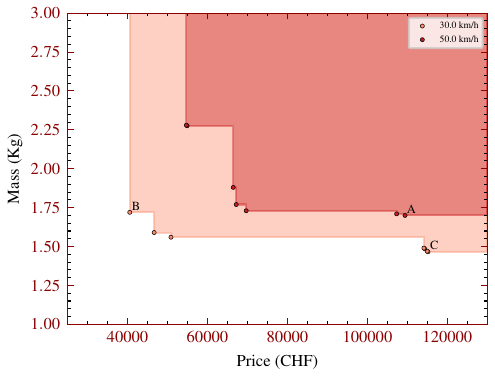}
    \end{subfigure}
    \\
    \begin{subfigure}[b]{0.15\textwidth}
        \centering
        \includegraphics[width=\textwidth]{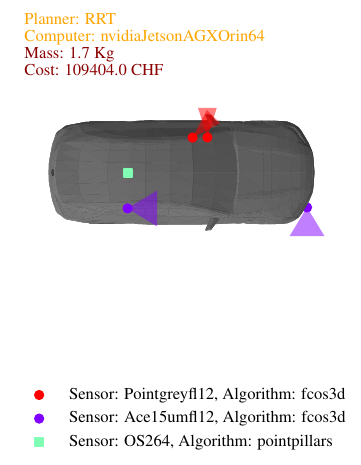}
        \captionsetup{labelformat=empty}
        \caption{A}
    \end{subfigure}
    \begin{subfigure}[b]{0.15\textwidth}
        \centering
        \includegraphics[width=\textwidth]{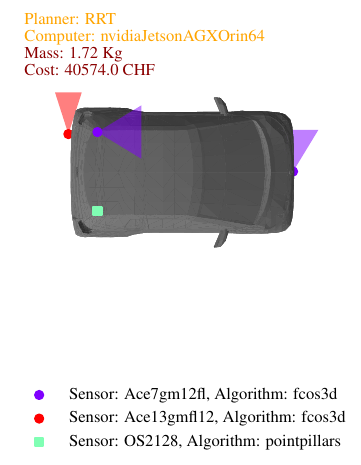}
        \captionsetup{labelformat=empty}
        \caption{B}
    \end{subfigure}
    \begin{subfigure}[b]{0.15\textwidth}
        \centering
        \includegraphics[width=\textwidth]{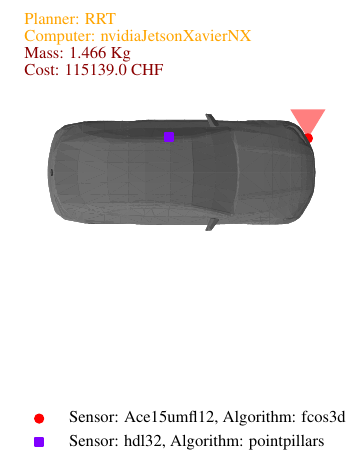}
        \captionsetup{labelformat=empty}
        \caption{C}
    \end{subfigure}
    \caption{Pareto front of price and mass across task velocities, 
    where higher nominal velocities for the same set of scenarios require more resourses. 
    Implementations for points A, B and C are visualized vertically. 
    A and C indicate lowest mass for $\unitfrac[50]{km}{h}$ and \unit[30]{kmh} nominal velocities, respectively. B indicates lowest price for $\unitfrac[30]{km}{h}$ nominal speed.}
    \label{fig:codesign:mass:velocity:pareto:imp}
\end{figure}
\begin{figure}[h!] 
    \centering
    \begin{subfigure}[b]{0.33\textwidth}
        \centering
        \includegraphics[width=\textwidth]{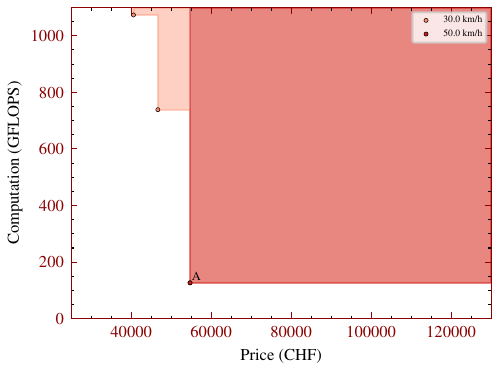}
    \end{subfigure}
    \\
    \begin{subfigure}[b]{0.15\textwidth}
        \centering
        \includegraphics[width=\textwidth]{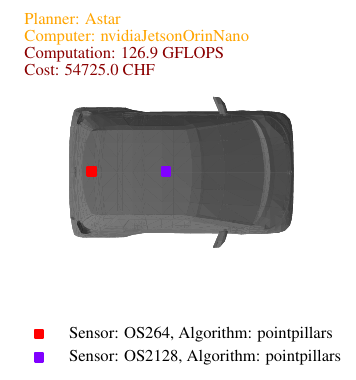}
        \captionsetup{labelformat=empty}
        \caption{A}
    \end{subfigure}
    \caption{Pareto front of price and computation across task velocities, 
    where higher velocities for the same set of scenarios require more resources. 
    Implementations for marked point A are visualized vertically. 
    A indicate lowest computation for $\unitfrac[50]{km}{h}$ and \unit[30]{kmh}.}
    \label{fig:codesign:computation:velocity:pareto:imp}
\end{figure}
\cref{fig:codesign:power:prior:pareto:imp,,fig:codesign:mass:prior:pareto:imp,,fig:codesign:computation:prior:pareto:imp} demonstrate how restricting car 
configurations prior within identical task scenarios leads to lower resource requirements, where 
the Pareto fronts, along with implementations are illustrated.
\begin{figure}[h!] 
    \centering
    \begin{subfigure}[b]{0.33\textwidth}
        \centering
        \includegraphics[width=\textwidth]{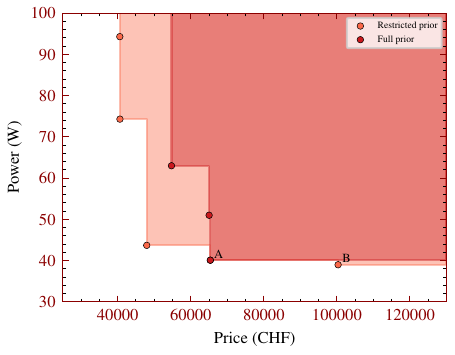}
    \end{subfigure}
    \\
    \begin{subfigure}[b]{0.15\textwidth}
        \centering
        \includegraphics[width=\textwidth]{figures/co_design/power/implementation_task_urban_car_cr_50.0_dry_day_power_40.1_65355.0_289.pdf}
        \captionsetup{labelformat=empty}
        \caption{A}
    \end{subfigure}
    \begin{subfigure}[b]{0.15\textwidth}
        \centering
        \includegraphics[width=\textwidth]{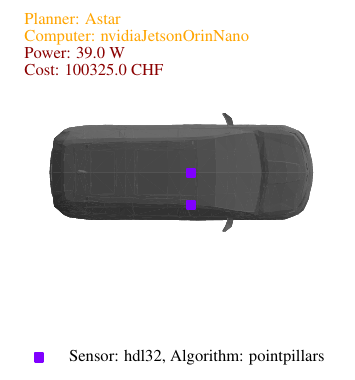}
        \captionsetup{labelformat=empty}
        \caption{B}
    \end{subfigure}
    \caption{Pareto front of price and power usage across priors, 
    where priors with more class configurations require more resources. 
    Implementations for points A and B are visualized vertically. 
    A and B indicate the lowest power usage for the least and most restricted prior, respectively.}
    \label{fig:codesign:power:prior:pareto:imp}
\end{figure}
\begin{figure}[h!] 
    \centering
    \begin{subfigure}[b]{0.33\textwidth}
        \centering
        \includegraphics[width=\textwidth]{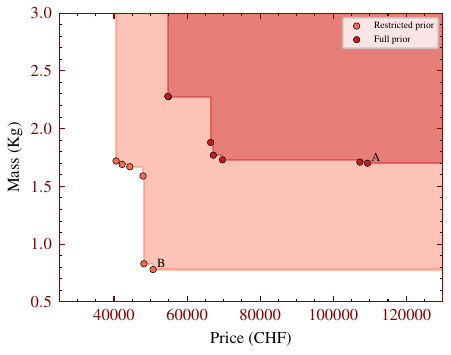}
    \end{subfigure}
    \\
    \begin{subfigure}[b]{0.15\textwidth}
        \centering
        \includegraphics[width=\textwidth]{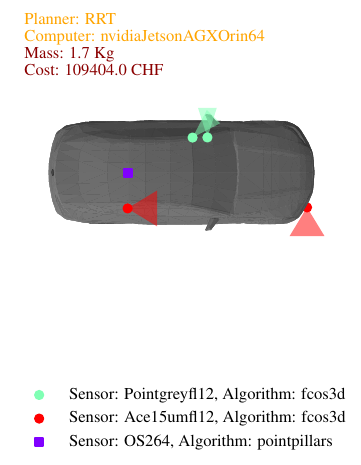}
        \captionsetup{labelformat=empty}
        \caption{A}
    \end{subfigure}
    \begin{subfigure}[b]{0.15\textwidth}
        \centering
        \includegraphics[width=\textwidth]{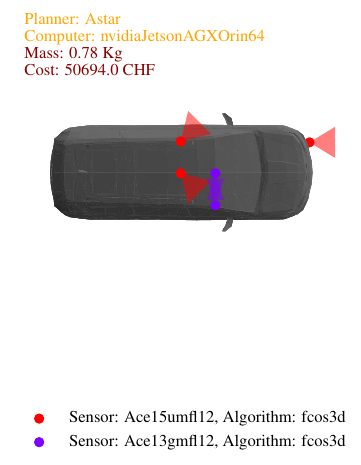}
        \captionsetup{labelformat=empty}
        \caption{B}
    \end{subfigure}
    \caption{Pareto front of price and mass across priors, 
    where priors with more class configurations require more resources. 
    Implementations for points A and B are visualized vertically. 
    A and B indicate the lowest mass for the least and most restricted prior, respectively.}
    \label{fig:codesign:mass:prior:pareto:imp}
\end{figure}
\begin{figure}[h!] 
    \centering
    \begin{subfigure}[b]{0.33\textwidth}
        \centering
        \includegraphics[width=\textwidth]{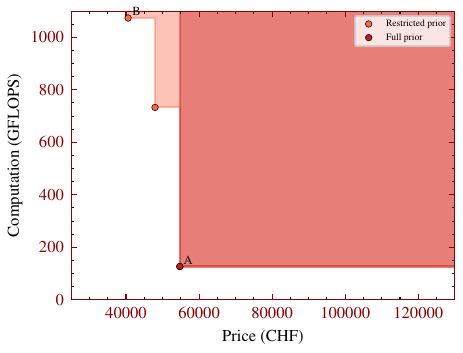}
    \end{subfigure}
    \\
    \begin{subfigure}[b]{0.15\textwidth}
        \centering
        \includegraphics[width=\textwidth]{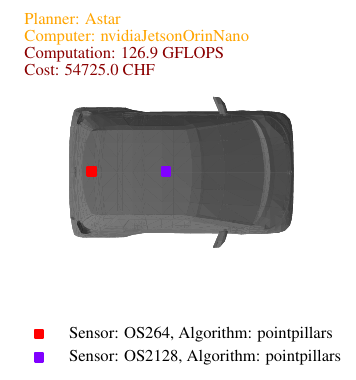}
        \captionsetup{labelformat=empty}
        \caption{A}
    \end{subfigure}
    \begin{subfigure}[b]{0.15\textwidth}
        \centering
        \includegraphics[width=\textwidth]{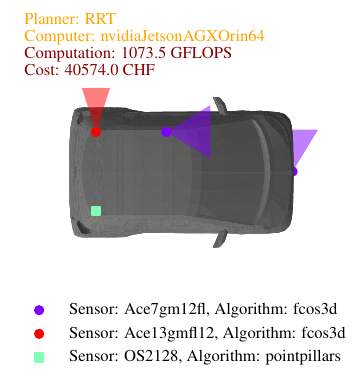}
        \captionsetup{labelformat=empty}
        \caption{B}
    \end{subfigure}
    \caption{Pareto front of price and computation across priors, 
    where priors with more class configurations require more resources. 
    Implementations for points A and B are visualized vertically. 
    A indicates the lowest computation for both priors (same implementation) and B indicates lowest price for the most restricted prior.}
    \label{fig:codesign:computation:prior:pareto:imp}
\end{figure}

In \cref{fig:phorizon_res_task} we show the influence of higher planning horizon leading to higher resource requirements on the selected sensors and perception algorithms by fixing the motion planner and the vehicle body. The figure compares the resources required - power, mass, price, and computation - for different tasks for planning horizons of one and two seconds. Each point represents the minimum resource solution for a given task and time horizon. In \cref{fig:planner_res_task}, we keep the vehicle body and planning horizon constant, but compare the resource trade-offs of using RRT* versus a lattice planner. This comparison aims to visualize the resource differences between motion planners, as expected from \cref{fig:planners}, and to highlight the impact of the planning strategy on the sensor selection and placement process.

\begin{figure}[h!] 
    \centering
    \begin{subfigure}[b]{0.22\textwidth}
        \centering
        \includegraphics[width=\textwidth]{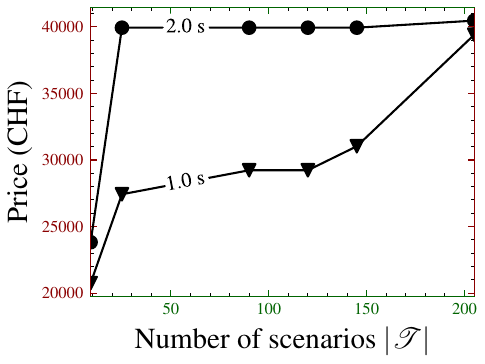}
        \caption{Price comparison.}
    \end{subfigure}
    \begin{subfigure}[b]{0.22\textwidth}
        \centering
        \includegraphics[width=\textwidth]{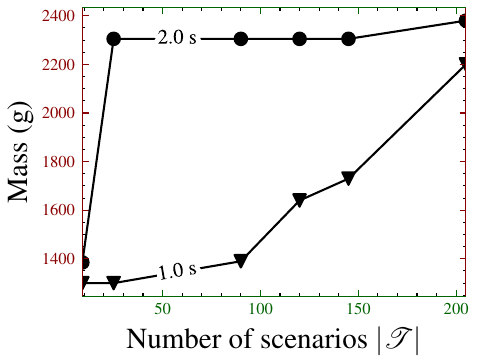}
        \caption{Mass comparison.}
    \end{subfigure}
    \begin{subfigure}[b]{0.22\textwidth}
        \centering
        \includegraphics[width=\textwidth]{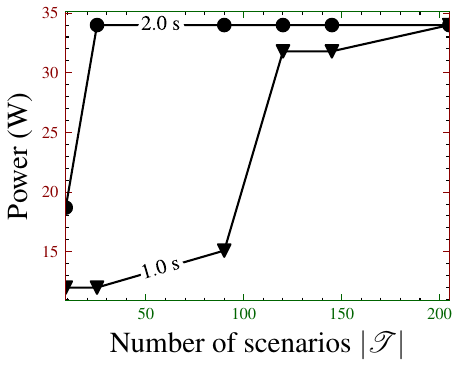}
        \caption{Power comparison.}
    \end{subfigure}
    \begin{subfigure}[b]{0.22\textwidth}
        \centering
        \includegraphics[width=\textwidth]{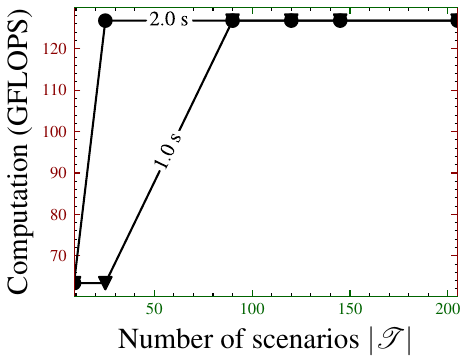}
        \caption{Computation comparison.}
    \end{subfigure}
    \caption{Higher planning horizons for the same planner and vehicle body require more resources for different tasks. Here we show the lattice planner with A* search and a hatchback vehicle body.}
    \label{fig:phorizon_res_task}
\end{figure}

\begin{figure}[h!] 
    \centering
    \begin{subfigure}[b]{0.22\textwidth}
        \centering
        \includegraphics[width=\textwidth]{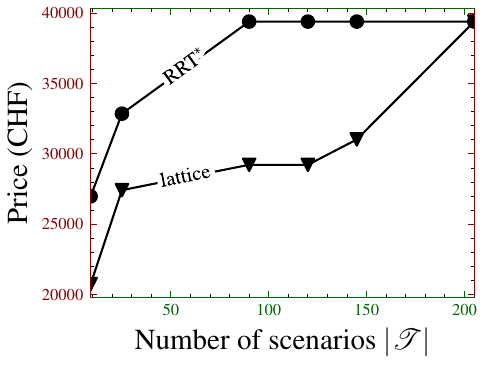}
        \caption{Price comparison.}
    \end{subfigure}
    \begin{subfigure}[b]{0.22\textwidth}
        \centering
        \includegraphics[width=\textwidth]{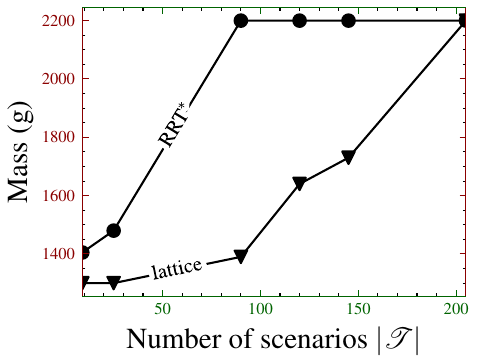}
        \caption{Mass comparison.}
    \end{subfigure}
    \begin{subfigure}[b]{0.22\textwidth}
        \centering
        \includegraphics[width=\textwidth]{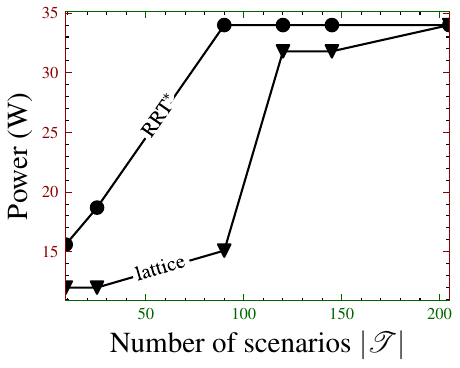}
        \caption{Power comparison.}
    \end{subfigure}
    \begin{subfigure}[b]{0.22\textwidth}
        \centering
        \includegraphics[width=\textwidth]{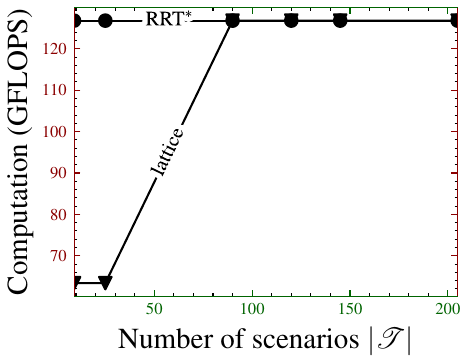}
        \caption{Computation comparison.}
    \end{subfigure}
    \caption{Resource comparison between RRT* planner and lattice planner with A* search for the same vehicle body (hatchback) and tasks.}
    \label{fig:planner_res_task}
\end{figure}

In~\cref{fig:codesign:computation:scenarios:pareto:imp,,fig:codesign:computation:prior:pareto:imp,,fig:codesign:computation:velocity:pareto:imp}, 
we display the implementations for the minimal computation solutions. The NVIDIA Jetson Orin Nano was chosen alongside the lattice motion planner using A* search for all cases. 
Notably, a camera sensor was never chosen for these solutions.
The implementations aiming for minimal mass are shown in \cref{fig:codesign:mass:scenarios:pareto:imp,fig:codesign:mass:prior:pareto:imp,,fig:codesign:mass:velocity:pareto:imp}, 
where there is a notable preference for cameras, predominantly coupled with the most powerful computing unit, the NVIDIA Jetson AGX Orin 64.
In \cref{fig:codesign:power:scenarios:pareto:imp,fig:codesign:power:prior:pareto:imp,,fig:codesign:power:velocity:pareto:imp} 
we present the implementations for the AV design with minimal power needs. Similarly as for the minimal computation,  
only one or two lidars are chosen.

Moreover, we present implementations tailored for the most cost-effective AV design 
in~\cref{fig:codesign:power:scenarios:pareto:imp,,fig:codesign:computation:scenarios:pareto:imp,,fig:codesign:computation:prior:pareto:imp,,fig:codesign:mass:velocity:pareto:imp}. 
Every implementation features at least one lidar sensor. Except for the cases highlighted in~\cref{fig:codesign:power:scenarios:pareto:imp,,fig:codesign:computation:prior:pareto:imp}, 
corresponding to the most complex task and the task with restricted prior, all configurations additionally incorporate camera sensors.
For the most complex task containing the most scenarios, highest nominal speed and no prior restriction, each implementation includes at least one lidar sensor. 

Throughout the minimal resource solutions for various tasks, we queried for the least resources by setting the average speed functionality requirement to just above zero. 
Thereby, the RRT* motion planner was consistently not selected.
Conversely, when examining tasks by requiring higher average speeds (e.g., $\unitfrac[24]{km}{h}$),
as illustrated in~\cref{fig:codesign:power:avgspeed:pareto:imp,,fig:codesign:mass:avgspeed:pareto:imp,,fig:codesign:computation:avgspeed:pareto:imp} for power, mass, and computation impacts, 
it becomes evident that the resource demands increase for higher average speeds, such as $\unitfrac[24]{km}{h}$ (with nominal speed of $\unitfrac[30]{km}{h}$). 
In every solution where minimal power, mass, computation, and cost were evaluated, 
the RRT* planner, coupled with the sedan vehicle, emerged as the selected choice. 
This pattern underscores the RRT* planner's superior efficiency within this case study, 
further highlighted by the sedan vehicle's highest acceleration capabilities and highest price.

\begin{figure}[h!] 
    \centering
    \begin{subfigure}[b]{0.33\textwidth}
        \centering
        \includegraphics[width=\textwidth]{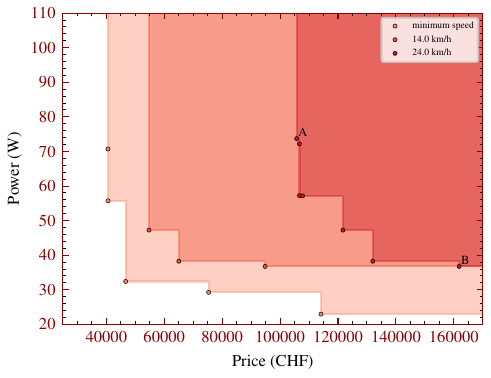}
    \end{subfigure}
    \\
    \begin{subfigure}[b]{0.15\textwidth}
        \centering
        \includegraphics[width=\textwidth]{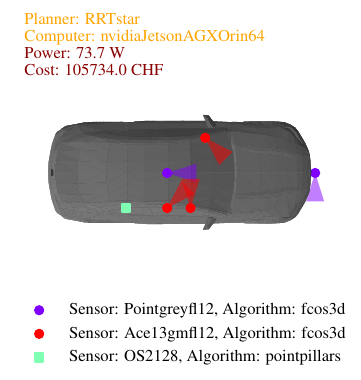}
        \captionsetup{labelformat=empty}
        \caption{A}
    \end{subfigure}
    \begin{subfigure}[b]{0.15\textwidth}
        \centering
        \includegraphics[width=\textwidth]{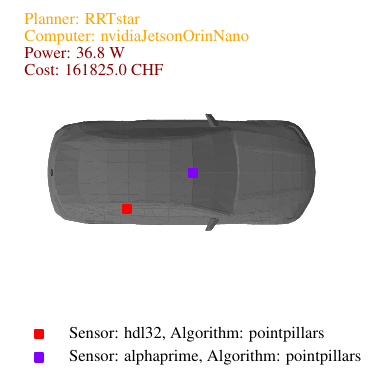}
        \captionsetup{labelformat=empty}
        \caption{B}
    \end{subfigure}
    \caption{Pareto front of price and power consumption across different average speeds, where planners providing higher average speed across all scenarios ($\unitfrac[30]{km}{h}$ nominal speed) demand more resources. 
    Implementations plots for points A and B are visualized vertically. 
    A and B indicate the lowest price and lowest power for the highest average speed, respectively.}
    \label{fig:codesign:power:avgspeed:pareto:imp}
\end{figure}

\begin{figure}[h!] 
    \centering
    \begin{subfigure}[b]{0.33\textwidth}
        \centering
        \includegraphics[width=\textwidth]{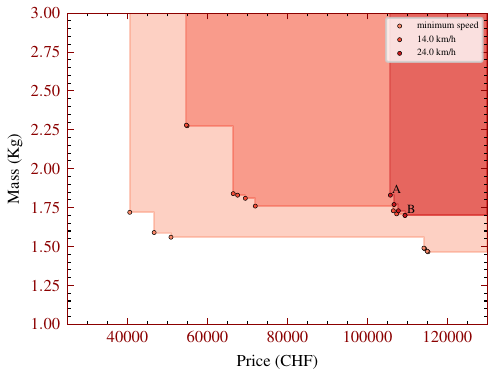}
    \end{subfigure}
    \\
    \begin{subfigure}[b]{0.15\textwidth}
        \centering
        \includegraphics[width=\textwidth]{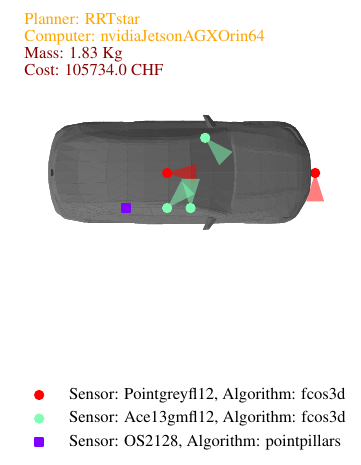}
        \captionsetup{labelformat=empty}
        \caption{A}
    \end{subfigure}
    \begin{subfigure}[b]{0.15\textwidth}
        \centering
        \includegraphics[width=\textwidth]{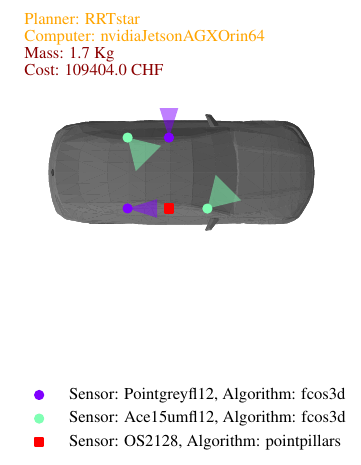}
        \captionsetup{labelformat=empty}
        \caption{B}
    \end{subfigure}
    \caption{Pareto front of price and mass across different average speeds, 
    where planners providing higher average speed across all scenarios ($\unitfrac[30]{km}{h}$ nominal speed) demand more resources. 
    Implementations for points A and B are visualized vertically. 
    A and B indicate the lowest price and mass for the highest average speed, respectively.}
    \label{fig:codesign:mass:avgspeed:pareto:imp}
\end{figure}

\begin{figure}[h!] 
    \centering
    \begin{subfigure}[b]{0.33\textwidth}
        \centering
        \includegraphics[width=\textwidth]{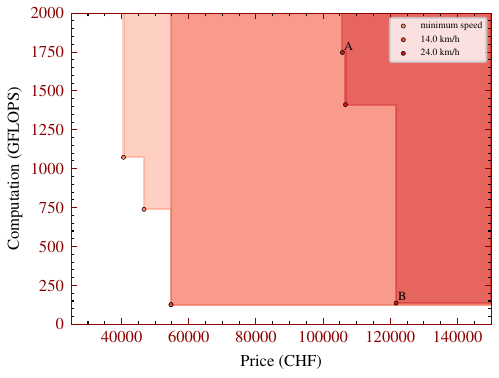}
    \end{subfigure}
    \\
    \begin{subfigure}[b]{0.15\textwidth}
        \centering
        \includegraphics[width=\textwidth]{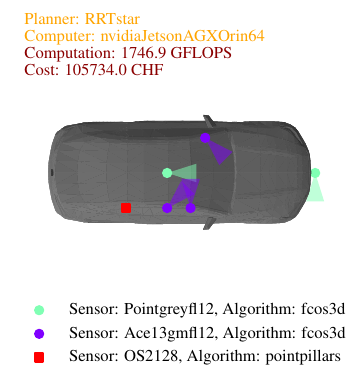}
        \captionsetup{labelformat=empty}
        \caption{A}
    \end{subfigure}
    \begin{subfigure}[b]{0.15\textwidth}
        \centering
        \includegraphics[width=\textwidth]{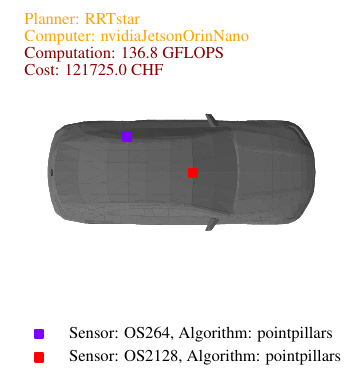}
        \captionsetup{labelformat=empty}
        \caption{B}
    \end{subfigure}
    \caption{Pareto front of price and computation across different average speeds, 
    where planners providing higher average speed across all scenarios ($\unitfrac[30]{km}{h}$ nominal speed) demand more resources. 
    Implementations for points A and B are visualized vertically. 
    A and B indicate the lowest price and lowest computation for the highest average speed, respectively.}
    \label{fig:codesign:computation:avgspeed:pareto:imp}
\end{figure}

\subsection{Discussion}\label{sec:discussion}
Our results show that increased task complexity, manifested by more scenarios, higher speeds, or broader prior knowledge, requires more resources for AV design. Each additional scenario may introduce new occupancy queries and prior knowledge, expanding the perception requirements. Higher speeds require sensor pipelines to detect objects at greater distances to account for the faster movement of the AV and the faster dynamics of the surrounding objects. In addition, a wider range of possible class configurations based on prior knowledge increases the perception requirements, calling for more advanced sensor pipelines that consume additional resources.

Motion planners that generate broader occupancy query distributions require enhanced sensing capabilities, thereby increasing the resource allocation to sensor pipelines to provide the required information. The broader occupancy query distributions result from either extended planning horizons, as illustrated in~\cref{fig:phorizon_res_task}, or the inherent strategy of the motion planner, as illustrated in~\cref{fig:planners} and \cref{fig:planner_res_task}. In the optimization process for minimal resource solutions at the lowest average speeds, the RRT* planner was consistently not selected. However, when the requirement shifted towards achieving the highest average speeds, 
the RRT* planner became the exclusive choice, paired with the vehicle body with the highest acceleration. This pattern suggests that while the RRT* planner demands more resources, 
it stands out as the most efficient option for optimizing average speed in the task.
Our analysis further confirms that to minimize computational requirements in AV design, lidar sensors emerge as the preferred choice due to their perception algorithms requiring fewer operations per second. Conversely, to reduce mass or cost, camera sensors are preferred due to their lighter weight and lower price compared to lidars. However, designs addressing the most complex task always include lidar sensors. This underscores the superior capability of lidar-equipped sensor pipelines due to their lower FNR and FPR across a wider range of class configurations.
\section{Conclusion}
\label{sec:conclusion}
This paper introduced a framework for designing mobile robots tailored to specific tasks by selecting hardware and software components. 
The choice comprises various elements including robot bodies, sensors, perception algorithms, sensor mounting configurations, motion planning algorithms, and computing units. 
We delved into the decision-making aspect of mobile robots by exploring what information a motion planner requires from the perception system. 
We introduced occupancy queries for sampling-based motion planners, allowing one to identify the necessary perception requirements based on prior knowledge of object classes, their dynamics, and shapes within the environment.
With the obtained perception requirements and the perception performance of a sensor combined with a detection algorithm, abstracted into FNRs and FPRs metrics, we formulated the sensor selection and placement problem and solved it as a weighted set cover problem using an \gls{ilp} approximation.
Our case study on designing an \gls{av} for urban driving scenarios revealed that enhanced task complexity, in terms of scenario variety or nominal speeds, necessitates more resources for the robot's design. We demonstrated how restricting prior knowledge of object configurations within scenarios can simplify designs and reduce resource requirements. Moreover, motion planners that generate broader distributions of occupancy queries or require longer planning horizons lead to increased task performance and perception requirements, necessitating more advanced and costly sensors and perception algorithms for the robot's design.
The findings highlight that the preference for specific sensors is influenced by the prioritization of resources. For designs prioritizing lower costs and weight, camera sensors are favored. Conversely, when minimizing power consumption 
and computing resources, lidar sensors are the preferred choice. Overall, lidar sensors exhibit superior perception performance and coverage, proving to be essential for handling complex tasks.
In future work, we aim to integrate additional agent architectures and motion planners beyond sampling-based. Additionally, rather than using upper bounds of FNRs and FPRs to determine object detection, we plan to implement filtering and sensor fusion techniques that incorporate considerations of time and uncertainty into the detection and sensor selection process. Moreover, we plan to conduct expanded case studies that include a variety of tasks and robots, not limited to AVs, and utilize state-of-the-art perception and decision-making software.
\bibliographystyle{IEEEtran}
\bibliography{references}
\begin{IEEEbiography}[{\includegraphics[width=1in,height=1.25in,clip,keepaspectratio]{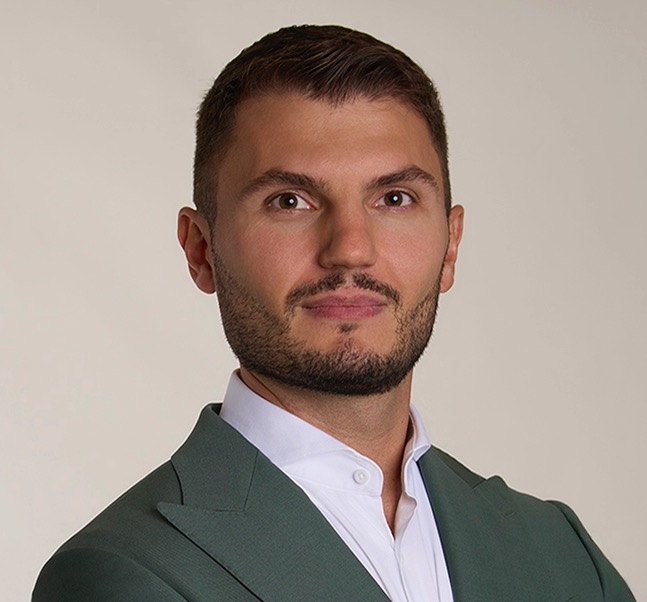}}]{Dejan Milojevic} 
is a postdoctoral researcher in Prof. Emilio Frazzoli's group at ETH Zürich. 
He received his Ph.D. in robotics, affiliated with both ETH Zürich and Empa. He holds a BSc. and MSc. in Mechanical Engineering from ETH Zürich. He has conducted research at Stanford University under Prof. Marco Pavone and has worked as a software engineer for Vay in Berlin, Germany. His research interests include co-design, sensor selection, perception, and decision-making in robotics.
\end{IEEEbiography}
\begin{IEEEbiography}[{\includegraphics[width=1in,height=1.25in,clip,keepaspectratio]{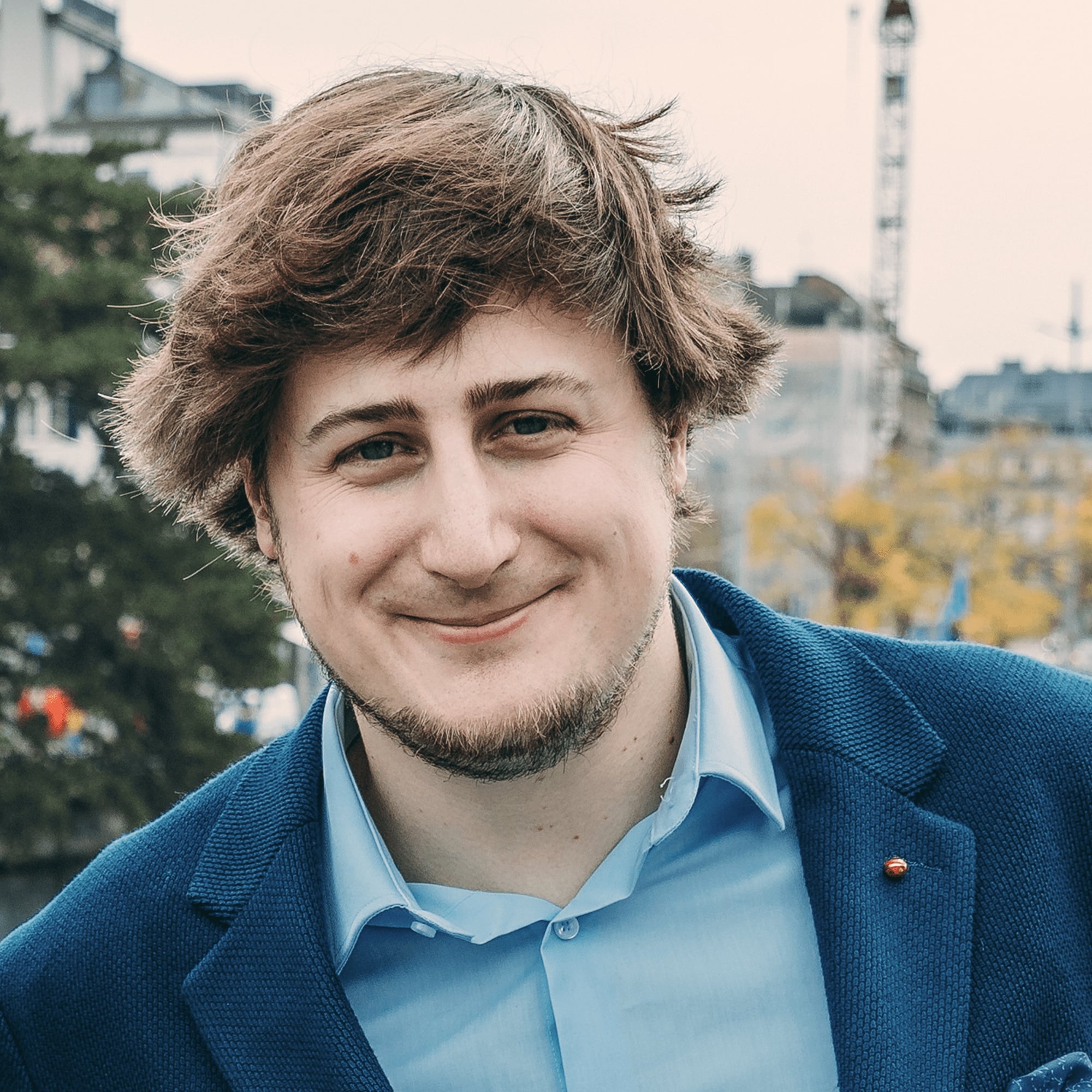}}]{Gioele Zardini} 
is the Rudge (1948) and Nancy Allen Assistant Professor at Massachusetts Institute of Technology. 
He is a PI in the Laboratory for Information and Decision Systems (LIDS), the Department of Civil and Environmental Engineering (CEE), and an affiliate faculty with the Institute for Data, Systems and Society (IDSS). 
From January to June 2024 he was a postdoctoral scholar at Stanford University, working with Prof. Marco Pavone, sponsored by NASA. He obtained his Ph.D. in Prof. Emilio Frazzoli's group at ETH Zürich. 
He received his BSc. and MSc. in Mechanical Engineering with focus in Robotics, Systems and Control from ETH Zürich in 2017 and 2019, respectively. Before joining MIT, he spent time in Singapore at nuTonomy (then Aptiv, now Motional), at Stanford University, in Marco Pavone's Autonomous Systems Lab.

\end{IEEEbiography}
\begin{IEEEbiography}[{\includegraphics[width=1in,height=1.25in,clip,keepaspectratio]{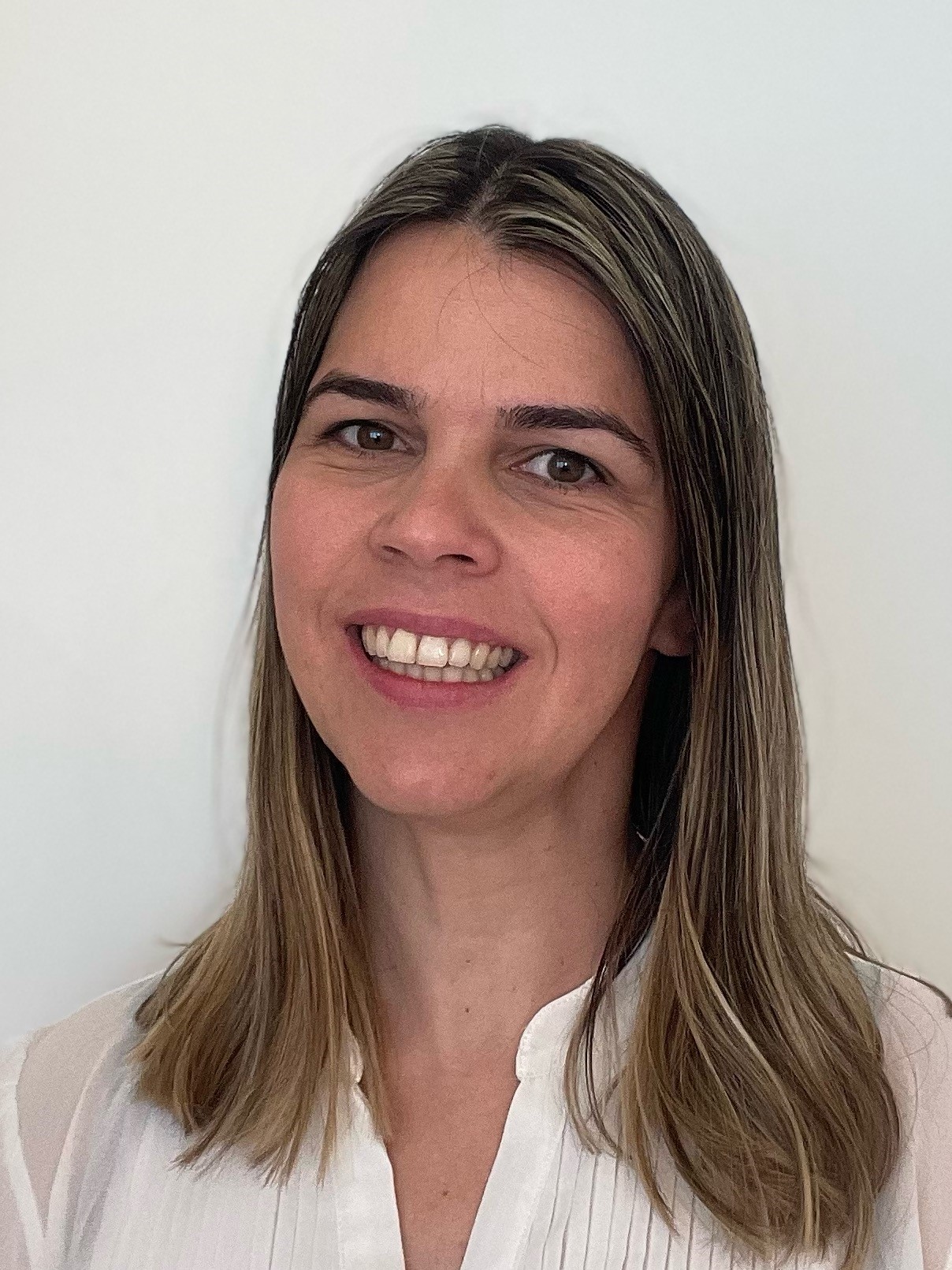}}]{Miriam Elser} 
leads the Vehicle Systems Group at the Chemical Energy Carriers and Vehicle Systems Laboratory at Empa. She holds a Master degree in Physics from the University of Milan and completed her Ph.D. in Atmospheric Environmental Science at ETH Zürich and the Paul Scherrer Institute in Switzerland. Miriam's current research focuses on future road mobility, with an emphasis on developing decarbonization strategies for road vehicles, validating and integrating new technologies such as autonomous driving.
\end{IEEEbiography}
\begin{IEEEbiography}[{\includegraphics[width=1in,height=1.25in,clip,keepaspectratio]{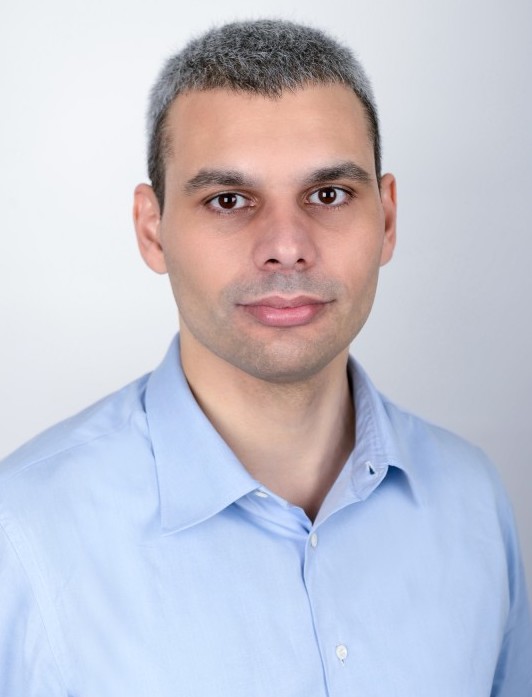}}]{Andrea Censi} 
is the deputy director of the Dynamic Systems and Control chair at ETH Zürich, director of the Duckietown Foundation, and  founder of Zupermind.
He obtained a M.Eng. degree in Control and Robotics from the University of Rome, ``Sapienza'', and a Ph.D. from California Institute of Technology. He has been a research scientist at the Massachusetts Institute of Technology, and the Director of Research at Aptiv Autonomous Mobility (now Motional).
He has been the recipient of NSF and AFRL awards.
\end{IEEEbiography}
\begin{IEEEbiography}[{\includegraphics[width=1in,height=1.25in,clip,keepaspectratio]{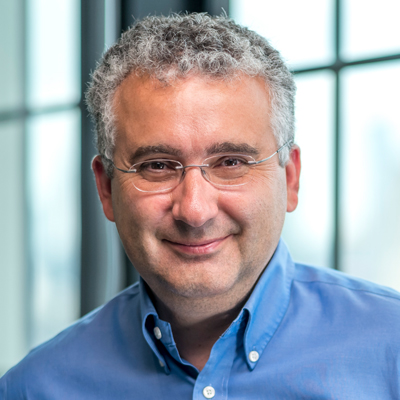}}]{Emilio Frazzoli} is a Professor of Dynamic Systems and Control at ETH Zürich. 
Until March 2021, he was Chief Scientist of Motional, the latest embodiment of nuTonomy, the startup he founded with Karl Iagnemma in 2013.
He received the Laurea degree in aerospace engineering from the University of Rome, ``Sapienza'', in 1994, and the Ph.D. degree in Aeronautics and Astronautics from MIT in 2001.
Before joining ETH Zürich in 2016, he held faculty positions at UIUC, UCLA, and MIT.
His current research interests focus primarily on autonomous vehicles, mobile robotics, and transportation systems.
He was the recipient of a NSF CAREER award in 2002, the IEEE George S. Axelby award in 2015, the IEEE Kiyo Tomiyasu award in 2017, the RSS Test of Time award in 2022, and is an IEEE Fellow since 2019.
\end{IEEEbiography}

\end{document}